%% file: mv_curves.tex
\documentclass[11pt,a4paper]{article}

\usepackage{amsthm,amsmath}
\usepackage{float} 
\usepackage[round]{natbib}
\usepackage[colorlinks,allcolors=blue]{hyperref}

\usepackage{amssymb}
\usepackage{amscd}
\usepackage{graphicx}
\usepackage{subfigure}

\usepackage{algpseudocode}
\usepackage{algorithm}
\makeatletter
\renewcommand{\ALG@beginalgorithmic}{\normalsize}
\makeatother

\usepackage{color}  
\usepackage{enumitem}  
\usepackage{authblk}
\usepackage{fullpage}

\newcommand{\RR}{\mathbb{R}}

\newcommand{\roc}{{\rm  ROC}}
\newcommand{\mv}{{\rm  MV}}

\def\argmin{\mathop{\rm arg min}}
\def\supp{\mathop{\rm supp}}

\def\S{{\cal S}}

\def\X{{\cal X}}
\def\T{{\cal T}}

\newtheorem{theorem}{{\bf Theorem}}
\newtheorem{example}{{\bf Example}}
\newtheorem{proposition}{{\bf Proposition}}

\newtheorem{definition}{{\bf Definition}}
\newtheorem{corollary}{{\bf Corollary}}
\newtheorem{lemma}{{\bf Lemma}}
\newtheorem{remark}{{\bf Remark}}

\newlist{assump}{enumerate}{10}
\setlist[assump]{label={\bf (A\arabic*)}}

\newlist{assump_boot}{enumerate}{10}
\setlist[assump_boot]{label={\bf (B\arabic*)}}

\newlist{assump_ml}{enumerate}{10}
\setlist[assump_ml]{label={\bf (C\arabic*)}}

\title{Mass Volume Curves and Anomaly Ranking}

\author[1]{Stephan Cl\'emen\c{c}on \thanks{stephan.clemencon@telecom-paristech.fr}}
\author[1,2]{Albert Thomas \thanks{albert.thomas@centraliens.net}}
\affil[1]{LTCI, T\'el\'ecom ParisTech, Universit\'e Paris Saclay, 75013, Paris, France}
\affil[2]{Airbus Group Innovations, 92150, Suresnes, France}

\date{}

\begin{document}

\maketitle

\begin{abstract}
This paper aims at formulating the issue of ranking multivariate unlabeled observations depending on their degree of abnormality as an unsupervised statistical learning task. In the 1-d situation, this problem is usually tackled by means of tail estimation techniques: univariate observations are viewed as all the more `abnormal' as they are located far in the tail(s) of the underlying probability distribution. It would be desirable as well to dispose of a scalar valued `scoring' function allowing for comparing the degree of abnormality of multivariate observations. Here we formulate the issue of scoring anomalies as a M-estimation problem by means of a novel functional performance criterion, referred to as the \textit{Mass Volume} curve (MV curve in short), whose optimal elements are strictly increasing transforms of the density almost everywhere on the support of the density. We first study the statistical estimation of the MV curve of a given scoring function and we provide a strategy to build confidence regions using a smoothed bootstrap approach. Optimization of this functional criterion over the set of piecewise constant scoring functions is next tackled. This boils down to estimating a sequence of empirical minimum volume sets whose levels are chosen adaptively from the data, so as to adjust to the variations of the optimal MV curve, while controling the bias of its approximation by a stepwise curve. Generalization bounds are then established for the difference in sup norm between the MV curve of the empirical scoring function thus obtained and the optimal MV curve.

\end{abstract}

\input{Sec1_Intro}
\input{Sec2_Background}
\input{Sec3_Ranking}
\input{Sec4_StatEstimation}
\input{Sec5_MEstimation}
\input{Sec6_AdaptiveARank}
\input{Sec7_Conclusion}

\section*{Acknowledgements}
We are grateful to Fran\c{c}ois Portier and Prof. Patrice Bertail for very helpful discussions. We also thank an anomymous reviewer for raising an important question which led to an improvement of Proposition \ref{prop:opt}.

\input{Sec_AppendixAdaptive}

\bibliographystyle{plainnat}
\bibliography{mvset}

\end{document}

%% file: Sec1_Intro.tex

\section{Introduction} 
 
In a wide variety of applications, ranging from the monitoring of aircraft engines in aeronautics to non destructive control quality in the industry through fraud detection, network intrusion surveillance or system management in data centers (see for instance \citep{VCTWMS}), \textit{anomaly detection} is of crucial importance. In most common situations, anomalies correspond to `rare' observations and must be automatically detected based on an unlabeled dataset.
In practice, the very purpose of anomaly detection techniques is to rank observations by degree of abnormality/novelty, rather than simply assigning them a binary label, `abnormal' \textit{vs} `normal'.  In the case of univariate observations, abnormal values are generally those which are extremes, \textit{i.e.} `too large' or `too small' in regard to central quantities such as the mean or the median, and anomaly detection may then derive from standard tail distribution analysis: the farther in the tail the observation lies, the more `abnormal' it is considered. In contrast, it is far from easy to formulate the issue in a multivariate context. In the present paper, motivated by applications such as those aforementioned, we place ourselves in the situation where (unlabeled) observations take their values in a possibly very high-dimensional space, $\X\subset\mathbb{R}^d$ with $d\geq 1$ say, making approaches based on nonparametric density estimation or multivariate (heavy-) tail modeling hardly feasible, if not unfeasible, due to the curse of dimensionality. In this framework,  a variety of statistical techniques, relying on the concept of \textit{minimum volume set} investigated in the seminal contributions of \citet{EinmahlMason92} and \citet{Polonik97} (see Section~\ref{sec:background}), have been developed in order to split the feature space $\X$ into two halves and decide whether observations should be considered as `normal' (namely when lying in the {\it minimum volume} set $\Omega_{\alpha}\subset \X$ estimated on the basis of the dataset available) or not (when lying in the complementary set $\X\setminus \Omega_{\alpha}$) with a given confidence level $1-\alpha\in (0,1)$. One may also refer to \citep{ScottNowak06} and to \citep{Kolt97} for closely related notions. The problem considered here is of different nature, the goal pursued is not to assign to all possible observations a label `normal' \textit{vs} `abnormal', but to rank them according to their level of `abnormality'. The most natural way to define a preorder on the feature space $\X$ is to transport the natural order on the real line through some (measurable) \textit{scoring function} $s:\X\rightarrow \mathbb{R}_+$: the `smaller' the score $s(x)$, the more likely the observation $x$ is viewed as `abnormal'. This problem shall be here referred to as \textit{anomaly scoring}. It can be somehow related to the literature dedicated to \textit{statistical depth functions} in nonparametric statistics and operations research, see \citep{ZuoSerfling00} and the references therein. Such parametric functions are generally proposed \textit{ad hoc} in order to define a `center' for the probability distribution of interest and a notion of distance to the latter, \textit{e.g.} the concept of `center-outward ordering' induced by {\it halfspace depth} in \citep{Tukey75} or \citep{DonohoGasko}. The angle embraced in this paper is quite different, its objective is indeed twofold: 1) propose a performance criterion for the anomaly scoring problem so as to formulate it in terms of $M$-estimation 2) investigate the accuracy of scoring rules which optimize empirical estimates of the criterion thus tailored.

Due to the global nature of the ranking problem, the criterion we promote is \textit{functional}, just like the \textit{Receiver Operating Characteristic} (ROC) and {\it Precision-Recall} curves in the supervised ranking setup (\textit{i.e.} when a binary label, \textit{e.g.} `normal' \textit{vs} `abnormal', is assigned to the sampling data), and shall be referred to as the \textit{Mass Volume curve} ({\sc MV} curve in abbreviated form). The latter induces a partial preorder on the set of scoring functions: the collection of optimal elements is defined as the set of scoring functions whose {\sc MV} curve is minimum everywhere. Such optimal scoring functions are proved to coincide to strictly increasing transforms of the underlying probability density almost everywhere on the support of this underlying probability density (see Section~\ref{sec:problem} for the exact definition). In the unsupervised setting, $\mv$ curve analysis is shown to play a role quite comparable to that of $\roc$ curve analysis for supervised anomaly detection. The issue of estimating the $\mv$ curve of a given scoring function based on sampling data is then tackled and a smooth bootstrap method for constructing confidence regions is analyzed. A statistical methodology to build a nearly optimal scoring function is next described, which works as follows: first, a piecewise constant approximant of the optimal $\mv$ curve is estimated by solving a few minimum volume set estimation problems where confidence levels are chosen adaptively from the data to adjust to the variations of the optimal $\mv$ curve; second, a piecewise constant scoring function is built based on the sequence of estimated minimum volume sets. The $\mv$ curve of the scoring rule thus produced can be related to a stepwise approximant of the (unknown) optimal $\mv$ curve, which permits to establish the generalization ability of the algorithm through rate bounds in terms of $\sup$ norm in the $\mv$ space.
\par The rest of the article is structured as follows. Section~\ref{sec:background} describes the mathematical framework, sets out the main notations and recalls the crucial notions related to anomaly/novelty detection on which the analysis carried out in the paper relies. Section~\ref{sec:problem} first provides an informal description of the anomaly scoring problem and then introduces the $\mv$ curve criterion, dedicated to evaluate the performance of any scoring function. The set of optimal elements is described and statistical results related to the estimation of the $\mv$ curve of a given scoring function are stated in Section~\ref{sec:stat}. Statistical learning of an anomaly scoring function is then formulated as a functional $M$-estimation problem in Section~\ref{sec:form}, while Section~\ref{sec:learning} is devoted to the study of a specific algorithm for the design of nearly optimal anomaly scoring functions. Technical proofs are deferred to the Appendix section.

We finally point out that a very preliminary version of this work has been presented in the conference AISTATS 2013 (see \citep{AISTATS13}). The present article gives a better characterization of the optimal scoring functions and investigates much deeper the statistical assessment of the performance of a given scoring function. In particular, it shows how to construct confidence regions for the $\mv$ curve of a scoring function in a computationally feasible fashion, using the (smooth) bootstrap methodology for which we state a consistency result. This consistency result gives a rate of convergence which promotes the use of a smooth bootstrap approach, rather than a naive bootstrap technique. Additionally, in Section~\ref{sec:learning}, the confidence levels of the minimum volume sets to be estimated are chosen in a data-driven way (instead of considering a regular subdivision of the interval $[0,1]$ fixed in advance), giving to the statistical estimation and learning procedures proposed in this paper crucial adaptivity properties, with respect to the (unknown) shape of the optimal $\mv$ curve. Eventually, we give an example showing that the nature of the problem tackled here is very different than that of density estimation and we also give a simpler formula of the derivative of the optimal $\mv$ curve (that of the underlying density) compared to the one originally given in \citep{AISTATS13}.

%% file: Sec2_Background.tex

\section{Background and Preliminaries} \label{sec:background}
As a first go, we start off with describing the mathematical setup and recalling key concepts in anomaly detection involved in the subsequent analysis.
 
\subsection{Framework and Notations} \label{sec:notations}
 
Here and throughout, we suppose that we observe independent and identically distributed realizations $X_1,\ldots,X_n$ of an unknown continuous probability distribution $F(dx)$, copies of a generic random variable $X$, taking their values in a (possibly very high dimensional) feature space $\mathcal{X}\subset\mathbb{R}^d$, with $d\geq 1$. The density of the random variable $X$ with respect to $\lambda(dx)$, the Lebesgue measure on $\mathbb{R}^d$, is denoted by $f$, its support $\{x \in \mathcal{X}, \, f(x) > 0 \}$ by $\supp(f)$ and the indicator function of any event $\mathcal{E}$ by $\mathbb{I}\{\mathcal{E}\}$. For any set $\mathcal{Z} \subset \mathcal{X}$, its complementary is denoted by $\overline{\mathcal{Z}} = \mathcal{X} \setminus \mathcal{Z}$. The sup norm of any real valued function $g: \mathbb{R}^d \rightarrow \mathbb{R}$ is denoted by $\Vert g \Vert_{\infty}$ and the Dirac mass at any point $a$ by $\delta_a$. The notation $O_{\mathbb{P}}(1)$ is used to mean boundedness in probability.
The quantile function $H^{\dag}$ of any real valued random variable $Z$ with cumulative distribution function $H(\cdot) = \mathbb{P}(Z \leq \cdot)$ is defined by $H^{\dag}(\alpha)=\inf\{t\in \RR: H(t)\geq \alpha\}$ for all $\alpha \in (0,1)$. For any real valued random variable $Z$, the generalized inverse $G^{-1}$ of the decreasing function $G(\cdot) = \mathbb{P}(Z \geq \cdot)$ is defined by $G^{-1}(\alpha)=\inf\{t\in \RR: G(t)\leq \alpha\}$ for all $\alpha \in (0,1)$.
For any function $s : \mathcal{X} \rightarrow \RR$, $F_s$ denotes the cumulative distribution function of the random variable $s(X)$. In addition, for any $\alpha\in (0,1)$, $Q(s,\alpha)= F_s^{\dag}(1-\alpha)$ denotes the quantile at level $1-\alpha$ of the distribution of $s(X)$. We also set $Q^*(\alpha)=Q(f,\alpha)$ for all $\alpha\in (0,1)$. Finally, constants in inequalities are denoted by either $C$ or $c$ and may vary at each occurrence.

A natural way of defining preorders on $\X$ is to map its elements onto $\RR_+$ and use the natural order on the real half-line.
\begin{definition}\label{def:scoring_function}{\sc (Scoring function)}
A scoring function is any measurable function $s : \mathcal{X} \rightarrow \mathbb{R}_+$ that is integrable with respect to the Lebesgue measure.
\end{definition}
The set of all scoring functions is denoted by $\mathcal{S}$ and we denote the level sets of any scoring function $s \in \mathcal{S}$ by:
\begin{equation*}
\Omega_{s,t}=\{x\in \X:\;\; s(x)\geq t\},\;\; t\in [-\infty,\;+\infty] \, .
\end{equation*}
Observe that the family is decreasing (for the inclusion, as $t$ increases from $-\infty$ to $+\infty$):
$$
\forall (t,t')\in \RR^2,\;\  t\geq t'\Rightarrow \Omega_{s,t}\subset \Omega_{s,t'}
$$
and that $\lim_{t\rightarrow +\infty}\Omega_{s,t}=\emptyset$ and $\lim_{t\rightarrow -\infty}\Omega_{s,t}=\X$\footnote{Recall that a sequence $(A_n)_{n\geq 1}$ of subsets of an ensemble $E$ is said to converge if and only if $\lim\sup A_n=\lim\inf A_n$. In such a case, one defines its limit, denoted by $\lim A_n$ as $\lim\sup A_n=\lim\inf A_n$.}.
 
The following quantities shall also be used in the sequel. For any scoring function $s$ and threshold level $t\geq 0$, define:
\begin{equation*}
\alpha_s(t)=\mathbb{P}\{s(X)\geq t\}\text{ and }
\lambda_s(t) =\lambda\left( \{x\in \X:\;\; s(x)\geq t\} \right) \, .
\end{equation*}
The quantity $ \alpha_s(t)$ is referred to as the \textit{mass} of the level set $\Omega_{s,t}$, while $\lambda_s(t)$ is generally termed the \textit{volume} (with respect to the Lebesgue measure).

\begin{remark} The integrability of a scoring function (see Definition \ref{def:scoring_function}) implies that the volumes are finite on $\RR_+^*$: for any $s \in \mathcal{S}$,
\begin{equation*}
\forall t> 0, \quad \lambda_s(t)=\int_{\X}\mathbb{I}\{x, s(x) \geq t\}dx \leq \int_{\X}\frac{s(x)}{t}dx<+\infty \, .
\end{equation*}
\end{remark}

Incidentally, we point out that, in the specific case $s=f$, the set $\Omega_{f,t} = \{x \in \mathcal{X}, f(x) \geq t \}$ is the \textit{density contour cluster} of the density function $f$ at level $t$, see \citep{Polonik95} for instance. Such a set is also referred to as a \textit{density level set}. Observe in addition that, using the terminology introduced in \citep{LPS99}, $\Omega_{s,t}$ is the \textit{region enclosed by the contour of depth} $t$ when $s$ is a depth function.



\subsection{Minimum Volume Sets}\label{subsec:mv_set}
The notion of minimum volume sets has been introduced and investigated in the seminal contributions of \citet{EinmahlMason92} and \citet{Polonik97} in order to describe regions where a multivariate random variable $X$ takes its values with highest/smallest probability.  Let $\alpha\in (0,1)$, a minimum volume set $\Omega^*_{\alpha}$ of mass at least $\alpha$ is any solution of the constrained minimization problem
\begin{equation}\label{eq:mvpb}
\min_{\Omega}\lambda(\Omega) \text{ subject to } \mathbb{P}\{X\in \Omega\}\geq \alpha \, ,
\end{equation}
where the minimum is taken over all measurable subsets $\Omega$ of $\X$.
Application of this concept includes in particular novelty/anomaly detection: for large values of $\alpha$, abnormal observations (outliers) are those which belong to the complementary set $\X\setminus\Omega^*_{\alpha}$ . In the continuous setting, it can be shown that there exists a threshold value equal to $Q^*(\alpha)$ such that the level set $\Omega_{f,Q^*(\alpha)}$ is a solution of the constrained optimization problem above.
The (generalized) quantile function is then defined by:
\begin{equation*}
\forall \alpha\in (0,1),\;\; \lambda^*(\alpha)\overset{def}{=}\lambda(\Omega^*_{\alpha}) \, .
\end{equation*}
The definition of $\lambda^*$ can be extended to the interval $[0,1]$ by setting $\lambda^*(0)=0$ and $\lambda^*(1) = \lambda(\supp(f)) \leq \infty$.
The following assumptions shall be used in the subsequent analysis.

\begin{assump}
\item \label{as:bounded_density} The density $f$ is bounded: $\Vert f \Vert_{\infty}<+\infty$.
\item \label{as:flat_parts} The density $f$ has no flat parts, \textit{i.e.} for any $t > 0$,
\begin{equation*}
\mathbb{P}\{f(X)=t\}=0 \, .
\end{equation*}
\end{assump}

Under the assumptions above, for any $\alpha\in (0,1)$, there exists a unique minimum volume set $\Omega^*_{\alpha}$ (equal to $\Omega_{f,Q^*(\alpha)}$ up to subsets of null $\lambda$-measure), whose mass is equal to $\alpha$ exactly. Additionally, the mapping $\lambda^*$ is continuous on $(0,1)$ and uniformly continuous on $[0,1-\varepsilon]$ for all $\varepsilon \in (0,1)$ (when the support of $F(dx)$ is compact, uniform continuity holds on the whole interval $[0,1]$).
 
From a statistical perspective, estimates $\widehat{\Omega}_{\alpha}$ of minimum volume sets are built by replacing the unknown probability distribution $F$ by its empirical version $\widehat F=(1/n)\sum_{i=1}^n\delta_{X_i}$ and restricting optimization to a collection $\mathcal{A}$ of borelian subsets of $\X$, supposed rich enough to include all density level sets (or reasonable approximations of the latter). In \citep{Polonik97}, functional limit results are derived for the generalized empirical quantile process $\{\lambda(\widehat{\Omega}^*_{\alpha})-\lambda^*(\alpha)\}$ under certain assumptions for the class $\mathcal{A}$ (stipulating in particular that $\mathcal{A}$ is a Glivenko-Cantelli class for $F(dx)$). In \citep{ScottNowak06}, it is proposed to replace the level $\alpha$ by $\alpha-\phi_n$ where $\phi_n$ plays the role of tolerance parameter. The latter should be chosen of the same order as the supremum $\sup_{\Omega\in \mathcal{A}}\vert \widehat F(\Omega)-F(\Omega) \vert$ roughly, complexity of the class $\mathcal{A}$ being controlled by the {\sc VC} dimension or by means of the concept of Rademacher averages, so as to establish rate bounds at $n<+\infty$ fixed. 
 
Alternatively, so-termed \textit{plug-in} techniques, consisting in computing first an estimate $\widehat{f}$ of the density $f$ and considering next level sets $\Omega_{\widehat{f}, t}$ of the resulting estimator have been investigated in several papers, among which \citep{Cavalier1997,Tsybakov97,Rigollet09,Cadre2006,Cadre2013}. Such an approach however yields significant computational issues even for moderate values of the dimension, inherent to the curse of dimensionality phenomenon.

%% file: Sec3_Ranking.tex

 \section{Ranking Anomalies}\label{sec:problem}
 In this section, the issue of scoring observations depending on their level of novelty/abnormality is first described in an informal manner and next formulated quantitatively, as a functional optimization problem, by means of a novel concept, termed the \textit{Mass Volume curve}.

 \subsection{Overall Objective}
 
The problem considered in this article is to learn a scoring function $s$, based on training data $X_1,\ldots,X_n$, so as to describe the extremal behavior of the (high-dimensional) random vector $X$ by that of the univariate variable $s(X)$, which can be summarized by its tail behavior near $0$: hopefully, the smaller the score $s(x)$, the more abnormal/rare the observation $x$ should be considered. Hence, an optimal scoring function should naturally permit to rank observations $x$ by increasing order of magnitude of $f(x)$ (see Section~\ref{sec:functional_criterion} for the precise definition of the set of optimal scoring functions). The preorder on $\X$ induced by such an optimal scoring function could then be used to rank observations by their degree of abnormality: for any optimal scoring function $s$, an observation $x'$ is more abnormal than an observation $x$ if and only if $s(x') < s(x)$.

\subsection{A Functional Criterion: the Mass Volume Curve}
\label{sec:functional_criterion}
We now introduce the concept of {\sc Mass Volume} curve and shows that it is a natural criterion to evaluate the accuracy of decision rules in regard to anomaly scoring.



\begin{definition}\label{def:mv}{\sc (True Mass Volume curve)} Let $s\in\S$. Its Mass Volume curve ($\mv$ curve in abbreviated form) with respect to the probability distribution of the random variable $X$ is the plot of the function
\begin{equation*}
\mv_s:\;\; \alpha\in (0,1)\mapsto \mv_s(\alpha)\overset{def}{=}\lambda_s \circ \alpha_s^{-1}(\alpha) \, .
\end{equation*}
If the scoring function $s$ is upper bounded, $\alpha_s^{-1}(0)$ exists and $\mv_s$ is defined on $[0,1)$.
\end{definition}

\begin{remark}{\sc (Parametric Mass Volume Curve)}
If $\alpha_s$ is invertible, and therefore $\alpha_s^{-1}$ is the ordinary inverse of $\alpha_s$, the Mass Volume curve of the scoring function $s$ can also be defined as the parametric curve:
\begin{equation*}
t\in \RR_+\mapsto \left( \alpha_s(t), \lambda_s(t) \right)\in [0,1)\times[0,+\infty) \, .
\end{equation*}
\end{remark}

\begin{figure}[htb]
\centering
\subfigure[Scoring functions]{%
\includegraphics[width=0.48\columnwidth]{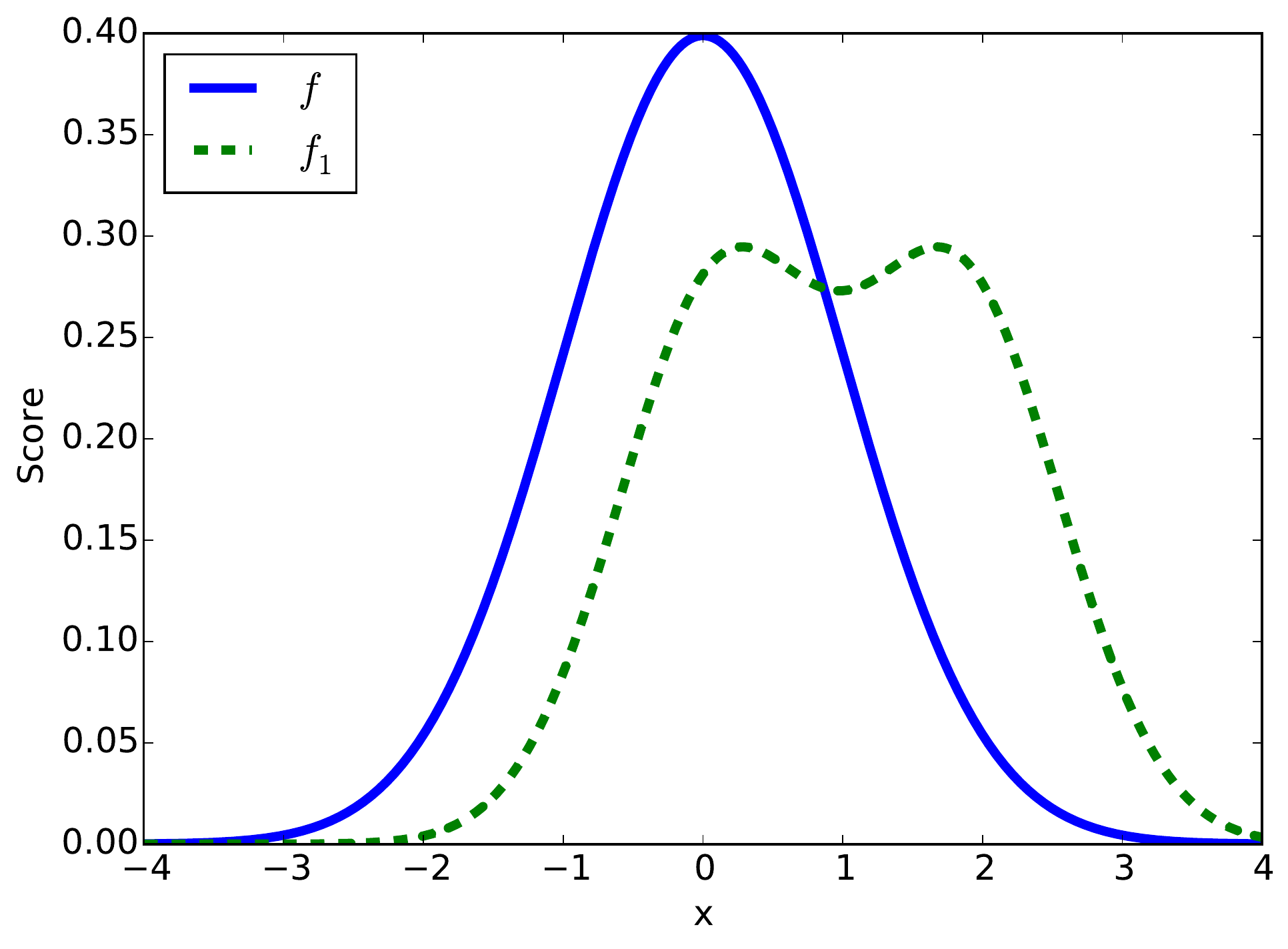}
\label{fig:examples_1}}
\subfigure[Mass Volume curves]{%
\includegraphics[width=0.48\columnwidth]{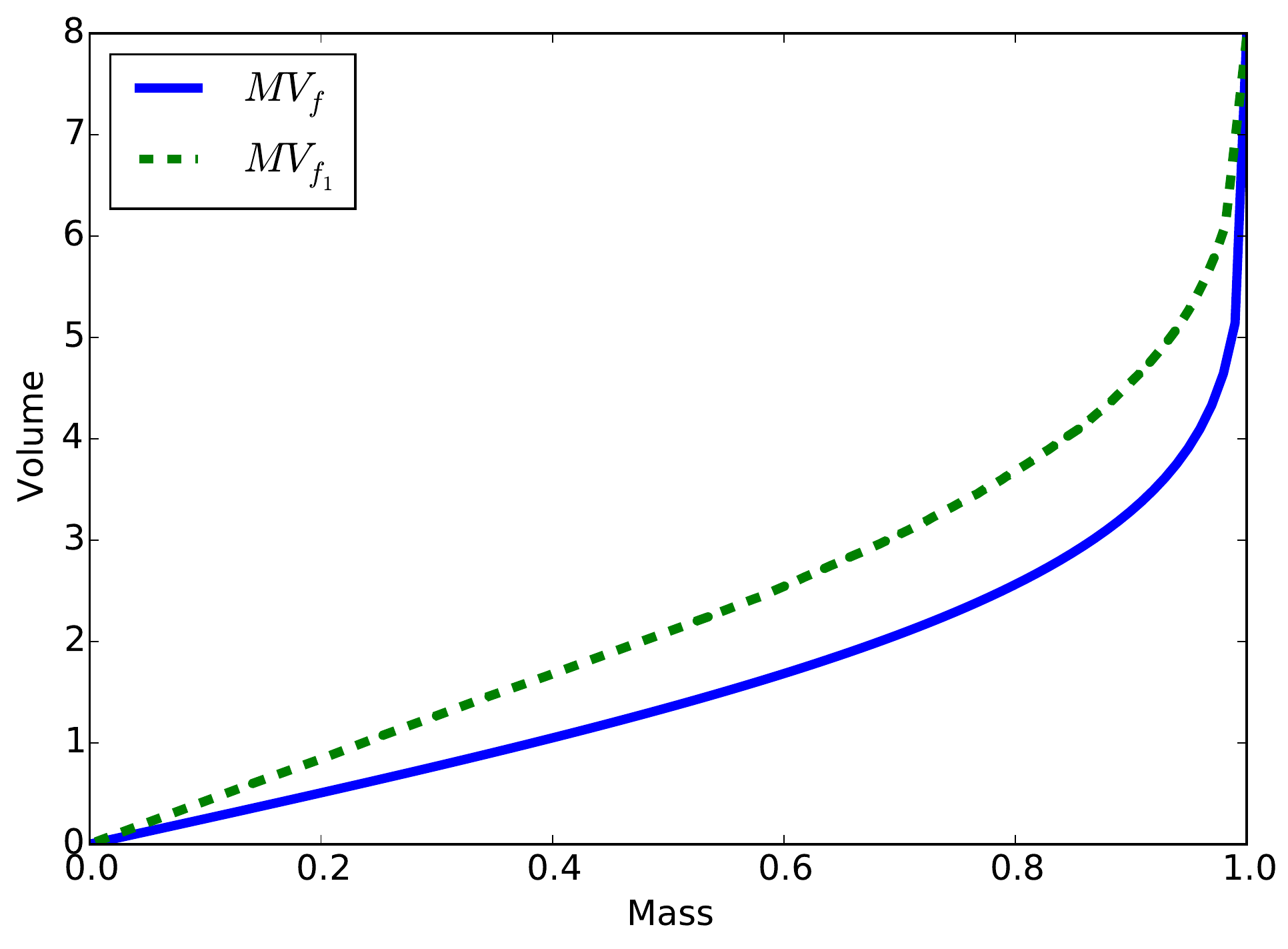}
\label{fig:examples_2}}
\caption{Scoring functions and their associated Mass Volume curves when considering a truncated Gaussian distribution with density $f$.}
\label{fig:mv_curves_examples}
\end{figure}


\begin{remark}{\sc (Connections to ROC/concentration analysis)}\label{rk:ROC}
We point out that the curve $\alpha \in (0,1) \mapsto 1 - \lambda_s\circ \alpha_s^{-1}(1-\alpha)$ resembles a receiver operating characteristic (ROC) curve of the test/diagnostic function $s$ (see \citep{Ega75}), except that the distribution under the alternative hypothesis is not a probability measure but the image of Lebesgue measure on $\X$ by the function $s$, while the distribution under the null hypothesis is the probability distribution of the random variable $s(X)$. Hence, the curvature of the $\mv$ graph somehow measures the extent to which the spread of the distribution of $s(X)$ differs from that of a uniform distribution. Observe also that, in the case where the support of $F(dx)$ coincides with the unit square $[0,1]^d$, the $\mv$ curve of any scoring function $s\in \mathcal{S}$ corresponds to the concentration function $(GF^{-1})_{\mathcal{C}_{\mathcal{S}}}$ introduced in \citep{Polonik99}, where $G(dx)$ is the uniform probability distribution on $[0,1]^d$ and $\mathcal{C}_{\mathcal{S}}=\{\Omega_{s,t}:\; t\geq 0  \}$, to compare the concentration of $F(dx)$ with that of $G(dx)$ with respect to the class of sublevel sets of the function $s$. Notice finally that, when $s$ is a depth function, $\mv_s$ is a scale curve, as defined in \citep{LPS99}.
\end{remark}

This functional criterion induces a partial order over the set of all scoring functions. Let $s_1$ and $s_2$ be two scoring functions on $\X$, the ordering provided by $s_1$ is better than that induced by $s_2$ when
$$
\forall \alpha \in (0,1),\;\; \mv_{s_1}(\alpha)\leq \mv_{s_2}(\alpha) \, .
$$
Typical $\mv$ curves are illustrated in Fig.~\ref{fig:mv_curves_examples}. A desirable $\mv$ curve increases slowly and rises near $1$, just like the lowest curve of Fig.~\ref{fig:mv_curves_examples}. This corresponds to the situation where the distribution of the random variable $X$ is much concentrated around its modes and the highest values (respectively, the lowest values) taken by $s$ are located near the modes of $F(dx)$ (respectively, in the tail region of $F(dx)$). The $\mv$ curve of the scoring function $s$ is then close to the right lower corner of the Mass Volume space.
We point out that, in certain situations, some parts of the {\sc MV} curve may be of interest solely, corresponding to large values of $\alpha$ when focus is on extremal observations (the tail region of the random variable $X$) and to small values of $\alpha$ when modes of the underlying distribution are investigated.

We define the set of optimal scoring functions $\mathcal{S}^*$ as follows. We recall here that for any set $\mathcal{Z} \subset \mathcal{X}$, its complementary is denoted by $\overline{\mathcal{Z}} = \mathcal{X} \setminus \mathcal{Z}$.
\begin{definition}
The set of optimal scoring functions $\mathcal{S}^*$ is the set of scoring functions $s \in \mathcal{S}$ such that there exist $\mathcal{Z} \subset \mathcal{X}$ and a function $T: Im f \rightarrow \mathbb{R}^+$ such that
\begin{enumerate}
\item[(i)] $\lambda(\overline{\mathcal{Z}} \cap \supp(f)) = 0$,
\item[(ii)] For all $x \in \mathcal{Z} \cap \supp(f)$, $s(x) = T \circ f(x)$,
\item[(iii)] $T_{\vert f(\mathcal{Z} \cap \supp(f))}$ is strictly increasing, 
\item[(iv)] For $\lambda$-almost all $x \in \overline{\supp(f)}$, $F_s(s(x)) = 0$.
\end{enumerate}
\end{definition}

One can note that, as expected, the density $f$ and strictly increasing transforms of the density belong to $\mathcal{S}^*$. More generally and roughly speaking, an optimal scoring function $s \in \mathcal{S^*}$ is a strictly increasing transform of the density $\lambda$-almost everywhere on the support of the density. Note also that on the one hand, if $a_s = \sup\{t, F_s(t) = 0 \}$, thanks to $(i)$, $(ii)$ and $(iii)$, we have $F_s(s(x)) = F_f(f(x)) > 0$ for almost all $x \in \supp(f)$ (because $\supp(f) = \{x, F_f(f(x)) > 0\}$ up to subsets of null $\lambda$ measure as shown in Lemma \ref{lem:support_z_f}). Therefore $s \geq a_s$ $\lambda$-almost everywhere on $\supp(f)$. On the other hand, thanks to $(iv)$, $F_s(s(x)) = 0$ for almost all $x \in \overline{\supp(f)}$ and therefore $s \leq a_s$ $\lambda$-almost everywhere on $\overline{\supp(f)}$. Thus the values of $s$ on $\overline{\supp(f)}$ are $\lambda$-almost everywhere lower than the values of $s$ on $\supp(f)$, where $s$ is almost everywhere a strictly increasing transform of the density $f$.

The result below shows that optimal scoring functions are those whose $\mv$ curves are minimum everywhere.

\begin{proposition} \label{prop:opt} {\sc (Optimal $\mv$ curve)} Let assumptions \ref{as:bounded_density} and \ref{as:flat_parts} be fulfilled.
The elements of the class $\S^*$ have the same $\mv$ curve, equal to $\mv_f$, and provide the best possible ordering of $\X$'s elements in regard to the $\mv$ curve criterion:
\begin{equation}\label{eq:opt}
s^* \in \mathcal{S}^* \iff \forall s \in \S \text{ and } \forall \alpha \in (0,1), \, \mv_{s^*}(\alpha)\leq \mv_s(\alpha) \, .
\end{equation}
In addition, we have: $\forall s \in \S$ and $\forall \alpha \in  (0,1)$, 
\begin{equation}\label{eq:bound1}
\mv_{s}(\alpha)-\mv^*(\alpha)\leq \lambda\left( \Omega^*_{\alpha} \Delta \Omega_{s,Q(s,\alpha)}\right) \, ,
\end{equation}
where $\mv^*(\alpha)=\mv_f(\alpha)$ for all $\alpha\in (0,1)$ and where $\Delta$ denotes the symmetric difference.
\end{proposition}

From now on we will thus use $\mv^*=\mv_f$ to denote the $\mv$ curve of elements of $\mathcal{S}^*$. The proof of Proposition \ref{prop:opt} can be found in Appendix~\ref{sec:appendix_proofs}. Incidentally, notice that, equipped with the notations introduced in \ref{subsec:mv_set}, we have $\lambda^*(\alpha)=\mv^*(\alpha)$ for all $\alpha\in (0,1)$. We also point out that bound \eqref{eq:bound1} reveals that the pointwise difference between the optimal $\mv$ curve and that of a scoring function candidate $s$ is controlled by the error made in recovering the specific minimum volume set $\Omega^*_{\alpha}$ through $\Omega_{s,Q(s,\alpha)}$.

\begin{remark}
In the framework we develop, anomaly scoring boils down to recovering the decreasing collection of all level sets of the density function $f$, $\{\Omega^*_{\alpha}:\; \alpha\in (0,1)\}$, without necessarily disposing of the corresponding levels. Indeed, one may check that any scoring function of the form
\begin{equation}\label{eq:form_opt}
s^*(x)=\int_0^1\mathbb{I}\{x\in \Omega^*_{\alpha}\}d\mu(\alpha) \, ,
\end{equation}
where $\mu(d\alpha)$ is an arbitrary finite positive measure dominating the Lebesgue measure on $(0,1)$, belongs to $\S^*$. Observe that $F_f\circ f$ corresponds to the case where $\mu$ is chosen to be the Lebesgue measure on $(0,1)$ in Eq. \eqref{eq:form_opt}. 
 The anomaly scoring problem can be thus cast as an overlaid collection of minimum volume set estimation problems. This observation shall turn out to be very useful when designing practical statistical learning strategies in Section~\ref{sec:learning}.
\end{remark}

We point out that the optimal $\mv$ curve  provides a measure of the mass concentration of the random variable $X$: the lower the curve $\mv^*$, the more concentrated the distribution $f(x)\lambda(dx)$. Indeed, the \textit{excess mass functional} introduced in \citep{MS91} can be expressed as $t\geq 0 \mapsto E(t)= \alpha_f(t)-t\lambda_f(t)$. Hence, we have $E(Q^*(\alpha))=\alpha-Q^*(\alpha)MV^*(\alpha)$ for all $\alpha\in (0,1)$. One may refer to \citep{Polonik95} for results related to the statistical estimation of density contour clusters and applications to multimodality testing using the excess mass functional.

\begin{example}{\sc (Univariate Gaussian distribution)} Let $X$ be random variable with Gaussian distribution $\mathcal{N}(0,1)$, we have $\mv^*(\alpha)=2\Phi^{-1}((1+\alpha)/2)$, where $\Phi$ is the cumulative distribution function of $X$.
\end{example}

\begin{example}{\sc (Multivariate Gaussian distribution)} Let $X \in \mathbb{R}^d$ be a multivariate random variable with Gaussian distribution $\mathcal{N}(0,\mathbf{\Sigma})$. We assume that $\Sigma$ is a diagonal matrix and we denote by $a_1,\dots,a_d \in \mathbb{R}^{+*}$ its diagonal coefficients. The density level set $\{x, f(x) \geq t \}$ is exactly the set $\{x, x^T\Sigma^{-1}x \leq c_t \}$ with $c_t > 0$. It is well known that $X^T\Sigma^{-1}X \sim \chi^2_d$, the $\chi^2$ distribution with $d$ degrees of freedom. Thus $\alpha_f^{-1}(\alpha) = c_t = \chi^2_d(\alpha)$, $\chi^2_d(\alpha)$ denoting the quantile of order $\alpha$ of the $\chi^2_d$ distribution. The set $\{x, x^T\Sigma^{-1}x \leq \chi^2_d(\alpha) \}$ is exactly the ellipsoid with semi-principal axes of length $a_1\sqrt{\chi^2_d(\alpha)}, \dots, a_d\sqrt{\chi^2_d(\alpha)}$. The formula of the volume of an ellipsoid gives $\mv^*(\alpha) = \pi^{d/2}/\Gamma(d/2 +1)(\chi^2_d(\alpha))^{d/2}\prod_{i=1}^d a_i$ where $\Gamma$ is the gamma function.
\end{example}

The following result reveals that the optimal $\mv$ curve is convex and provides a closed analytical form for its derivative. The proof is given in Appendix~\ref{sec:appendix_proofs}.

\begin{proposition}\label{prop:convex_derivative}{\sc (Convexity and derivative)}
Suppose that assumptions \ref{as:bounded_density} and \ref{as:flat_parts} are satisfied. Then, as $f$ is bounded, the optimal curve $\mv^*$ is defined on $[0,1)$. Furthermore,
\begin{enumerate}
	\item[(i)] $\alpha\in [0,1)\mapsto \mv^*(\alpha)$ is convex.
	\item[(ii)] Let $F_f$ be the cumulative distribution function of the random variable $f(X)$ and let $a = \lim_{\alpha \rightarrow 0^+}F_f^{\dag}(\alpha) = \sup\{t, F_f(t)=0\} \geq -\infty$. If $F_f$ is invertible on $(a,F_f^{\dag}(1)]$ then $\mv^*$ is differentiable on $[0, 1)$ and
	\begin{equation*}
	\forall \alpha \in [0,1), \quad \mv^{*\prime}(\alpha) = \frac{1}{Q^*(\alpha)} \, .
	\end{equation*}
\end{enumerate}
\end{proposition}

The convexity of $\mv^*$ might be explained as follows: when considering density level sets with increasing probabilities $\alpha$, their volumes $\lambda(\Omega^*_{\alpha})$ increase as well, but they increase more and more quickly since the mass of the distribution becomes less and less concentrated. 


Elementary properties of $\mv$ curves are summarized in the following proposition.

 \begin{proposition}\label{prop:properties} {\sc (Properties of $\mv$ curves)}  For any $s\in \S$, the following assertions hold true.
 \begin{enumerate}
 \item[(i)] {\bf Invariance.} For any strictly increasing function $\psi:\RR_+ \rightarrow \RR_+$, we have $\mv_s=\mv_{\psi \circ s}$.
 \item[(ii)] {\bf Monotonicity.} The mapping $\alpha\in (0,1)\mapsto \mv_s(\alpha)$ is increasing.
 \end{enumerate}
 \end{proposition}
Assertion $(i)$ may be proved using the following property of the quantile function (see \textit{e.g.} Property 2.3 in \citep{Embrechts2013}): for any cumulative distribution function $H$, any $z \in \mathbb{R}$ and any $\alpha \in (0, 1)$, $z \geq H^{\dag}(\alpha)$ if and only if $H(z) \geq \alpha$. Assertion $(ii)$ derives from the fact that the quantile function $F_s^{\dag}$ is increasing. Details are left to the reader.

%% file: Sec4_StatEstimation.tex

\section{Statistical Estimation}\label{sec:stat}
In practice, $\mv$ curves are unknown, just like the probability distribution of the random variable $X$, and must be estimated based on the observed sample $X_1,\dots,X_n$. Replacing the mass of each level set by its statistical counterpart in Definition \ref{def:mv} leads to define the notion of \textit{empirical $\mv$ curve}. We set, for all $t\geq 0$,
\begin{equation}\label{eq:mass_emp}
\widehat{\alpha}_s(t)\overset{def}{=}\frac{1}{n}\sum_{i=1}^n\mathbb{I}\{s(X_i) \geq t\}=\int_{\X}\mathbb{I}\{u \geq t\}\widehat{F}_s(du) \, ,
\end{equation}
where $\widehat{F}_s=(1/n)\sum_{i\leq n}\delta_{s(X_i)}$ denotes the empirical distribution of the $s(X_i), 1 \leq i \leq n$.
Notice that $\widehat{\alpha}_s$ takes its values in the set $\{k/n:\;\; k=0,\dots,n\}$.

\begin{definition}\label{def:emp_mv}{\sc (Empirical $\mv$ curve)}
Let $s\in \S$. By definition, the empirical $\mv$ curve of $s$ is the graph of the (piecewise constant) function
\begin{equation*}
\widehat{\mv}_s: \alpha\in[0, 1)\mapsto \lambda_s \circ \widehat{\alpha}^{-1}_s(\alpha) \, .
\end{equation*}

\end{definition}

\begin{remark}{\sc (On volume estimation)}. Except for very specific choices of the scoring function $s$ (\textit{e.g.} when $s$ is piecewise constant and the volumes of the subsets of $\mathcal{X}$ on which $s$ is constant can be explicitly computed), no closed analytic form for the volume $\lambda_s(t)$ is available in general and Monte-Carlo procedures should be used to estimate it (which may be practically challenging in a high-dimensional setting). Computation of volumes in high dimensions is an active topic of research \citep{Lovasz2006} and for simplicity, this is not taken into account in the subsequent analysis. 

\end{remark}

In order to obtain a smoothed version $\widetilde{\alpha}_s(t)$, a typical strategy consists in replacing the empirical distribution estimate $\widehat{F}_s$ involved in \eqref{eq:mass_emp} by the continuous distribution $\widetilde{F}_s$ with density $\widetilde{f}_s(t)=(1/n)\sum_{i=1}^nK_h(s(X_i)-t)$, with $K_h(t)=h^{-1}K(t/h)$ where $K\geq 0$ is a regularizing Parzen-Rosenblatt kernel (\textit{i.e.} a bounded square integrable function such that $\int K(v)dv=1$) and $h>0$ is the smoothing bandwidth (see for instance \citep{WandJones1994}).
 
\subsection{Consistency and Asymptotic Normality}
The theorem below reveals that, under mild assumptions, the empirical $\mv$ curve is a consistent and asymptotically Gaussian estimate of the $\mv$ curve, uniformly over any subinterval of $[0,1)$. It involves the assumptions listed below.

\begin{assump}[resume]
\item \label{as:bounded_score} The scoring function $s$ is bounded: $\Vert s \Vert_{\infty} < +\infty$.
\item \label{as:density_regularity_positive} The random variable $s(X)$ has a continuous cumulative distribution function $F_s$. Let $a = \sup\{t \in \mathbb{R}, \, F_s(t) =0\} \geq -\infty$ and $b = \inf\{t \in \mathbb{R}, \, F_s(t) =1\} \leq +\infty$. The distribution function $F_s$ is twice differentiable on $(a,b)$ and
\begin{equation}
\label{eq:fs_pos}
\forall t \in (a, b), \;\; F_s'(t) = f_s(t) > 0 \, .
\end{equation}
\item \label{as:tail} There exists $c>0$ such that
\begin{equation*}
\sup_{t \in (a,b)}F_s(t)(1-F_s(t)) \Bigl\vert \frac{f_s^{\prime}(t)}{f_s^2(t)} \Bigr\vert \leq c<+\infty \, .
\end{equation*}
\item $f_s$ has a limit $A > 0$ when $t$ tends towards $b$ from the left:
\begin{equation*}
\lim_{t \rightarrow b^-} f_s(t) = A < \infty \, .
\end{equation*}
\label{as:limit_alpha_1}
\item The mapping $\lambda_s$ is of class $\mathcal{C}^2$.
\label{as:regulartiy_lambda}
\end{assump}

Assumption \ref{as:bounded_score} implies that $F_s^{\dag}(1)$ exists and $F_s^{\dag}(1) = b$. Note also that if $s$ is not bounded, one can always consider $s' = \arctan \circ s$. Indeed $s'$ takes its values in $[0, \pi/2)$ and as $\arctan$ is strictly increasing $\mv_s = \mv_{s'}$ thanks to Proposition \ref{prop:properties}. Assumptions \ref{as:density_regularity_positive} and \ref{as:tail} are common assumptions for the strong approximation of the quantile process (see for instance \citep{CR78}). Condition \eqref{eq:fs_pos} and assumption \ref{as:tail} are respectively equivalent to $f_s(F_s^{\dag}(\alpha)) > 0$ for all $\alpha \in (0,1)$ and
\begin{equation*}
\sup_{\alpha \in (0,1)}\alpha(1-\alpha) \Bigl\vert \frac{f_s^{\prime}(F_s^{\dag}(\alpha))}{f_s^2(F_s^{\dag}(\alpha))} \Bigr\vert \leq c<+\infty \, .
\end{equation*}
Eventually, assumption \ref{as:limit_alpha_1} is equivalent to: $f_s$ has a limit $A > 0$ when $x$ tends towards $b$. The density $f_s$ can therefore be extended by continuity to the interval $(a,b]$. As $A > 0$, \eqref{eq:fs_pos} can then be replaced by $f_s(t)>0$ for all $t \in (a, b]$ which also gives $f_s(F_s^{\dag}(\alpha))>0$ for all $\alpha \in (0, 1]$. One can also note that $F_s$ is thus invertible on $(a, b]$ with inverse $F_s^{\dag}$. \newline

In order to state part $(ii)$ of the following theorem, we assume that the probability space is rich enough in the sense that an infinite sequence of Brownian bridges can be defined on it.

\begin{theorem}\label{thm:estimation}
Let $\varepsilon \in (0,1]$ and $s\in \S$. Assume that assumptions \ref{as:bounded_score}-\ref{as:regulartiy_lambda} are fulfilled. The following assertions hold true.
\begin{enumerate}
\item[(i)] {\sc (Consistency)} With probability one, we have uniformly over $[0,1-\varepsilon]$:
\begin{equation*}
\lim_{n\rightarrow+\infty}\widehat{\mv}_s(\alpha)=\mv_s(\alpha) \, .
\end{equation*}
\item[(ii)] {\sc (Strong approximation)} There exists a sequence of Brownian bridges $\{B_n(\alpha), \alpha \in [0,1]\}_{n \geq 1}$ such that we almost surely have, uniformly over the compact interval $[0, 1-\varepsilon]$: as $n\rightarrow \infty$,
\begin{equation*}
\sqrt{n}\left( \widehat{\mv}_s(\alpha)-\mv_s(\alpha) \right)=Z_n(\alpha) +O(n^{-1/2}\log n) \, ,
\end{equation*}
where
\begin{equation*}
Z_n(\alpha)=\frac{\lambda^{\prime}_s(\alpha_s^{-1}(\alpha))}{f_s(\alpha_s^{-1}(\alpha))}B_n(\alpha), \text{ for }\alpha\in [0,1-\varepsilon] \, .
\end{equation*}
\end{enumerate}
\end{theorem}

The technical proof is given in Appendix~\ref{sec:appendix_proofs}. One can note from the proof of assertion $(i)$ that this assertion does not require assumption \ref{as:tail} to be satisfied. The proof of assertion $(ii)$ relies on standard strong approximation results for the quantile process \citep{CR78,CR81,Csorgo1983}. Assertion $(ii)$ means that the fluctuation process $\{\sqrt{n}(\widehat{\mv}_s(\alpha)-\mv_s(\alpha)), \alpha\in [0,1-\varepsilon]\}$ converges in the space of c\`ad-l\`ag functions on $[0,1-\varepsilon]$ equipped with the sup norm, to the law of a Gaussian stochastic process $\{Z_1(\alpha), \alpha\in [0,1-\varepsilon]\}$.

\begin{remark} Assumption \ref{as:limit_alpha_1} is required in order to state the results of assertions $(i)$ and $(ii)$ on the interval $[0, 1-\varepsilon]$. However this assumption is restrictive as it implies that $f_s$ is discontinuous at $b$ since $A > 0$ whereas $f_s(t) = 0$ for all $t > b$. Instead of assumption \ref{as:limit_alpha_1}, if one assumes that $f_s$ is decreasing on an interval to the left of $b$, then the result of assertion $(ii)$ can be obtained on the interval $(0,1-\varepsilon]$ with the same rate of convergence if $c < 2$ and with the same rate of convergence up to $\log n$ and $\log \log n$ factors if $c \geq 2$. In this case, assertion $(i)$ also holds on $(0,1-\varepsilon]$ (see for instance \citep{CR78}).
\end{remark}

\begin{remark} {\sc (Asymptotic normality)} It results from assertion $(ii)$ in Theorem \ref{thm:estimation} that, for any $\alpha\in (0,1)$, the pointwise estimator $\widehat{\mv}_s(\alpha)$ is asymptotically Gaussian under assumptions \ref{as:bounded_score}-\ref{as:regulartiy_lambda}. For all $\alpha\in (0,1)$
$\sqrt{n}(\widehat{\mv}_s(\alpha)-\mv_s(\alpha))$ converges in distribution towards $\mathcal{N}(0,\sigma^2_s)$, as $n\rightarrow +\infty$, with $\sigma^2_s=\alpha(1-\alpha)(\lambda_s'(\alpha_s^{-1}(\alpha))/f_s(\alpha_s^{-1}(\alpha)))^2$.
\end{remark}
 
\subsection{Confidence Regions in the Mass Volume Space}\label{subsec:band}
The true $\mv$ curve of a given scoring function is unknown in practice and its performance must be statistically assessed based on a data sample. Beyond consistency of the empirical curve in sup norm and the asymptotic normality of the fluctuation process, we now tackle the question of constructing confidence bands in the $\mv$ space.
\begin{definition}
Based on a sample $\mathcal{D}_n=(X_1,\; \ldots,\; X_n)$, a (random) confidence region for the $\mv$ curve of a given scoring function $s\in \S$ at confidence level $\eta\in (0,1)$ is any borelian set $\mathcal{R}_{\eta}\subset [0,1]\times \mathbb{R}_+$ of the $\mv$ space (possibly depending on $\mathcal{D}_n$) that covers the curve $\mv_s$ with probability larger than $1-\eta$:
$$
\mathbb{P}\{ \mv_s \in \mathcal{R}_{\eta} \}\geq 1-\eta \, .
$$
\end{definition}
In practice, confidence regions shall be of the form of balls in the Skorohod's space $\mathbb{D}([0,1-\varepsilon])$ of c\`ad-l\`ag functions on $[0,1-\varepsilon]$ with respect to the $\sup$ norm and with the estimate $\widehat{\mv}_s$ introduced in Definition \ref{def:emp_mv} as center for some fixed $\varepsilon\in (0,1)$:
$$
\mathcal{B}(\widehat{\mv}_s, \nu)=\left\{g\in \mathbb{D}([0,1-\varepsilon]):\; \sup_{\alpha\in [0,1-\varepsilon]}\vert g(\alpha)- \widehat{\mv}_s(\alpha) \vert \leq \nu  \right\}.
$$
Constructing confidence regions based on the approximation stated in Theorem \ref{thm:estimation} would require to know the density $f_s$. Hence a bootstrap approach \citep{Efron} should be preferred. Following in the footsteps of \citet{SY87}, it is recommended to implement a smoothed bootstrap procedure. The asymptotic validity of such a resampling method derives from a strong approximation result similar to the one of Theorem \ref{thm:estimation}.
Let $r_n$ be the fluctuation process defined for all $\alpha \in [0, 1-\varepsilon]$ by
\begin{equation*}
r_n(\alpha)=\sqrt{n}(\widehat{\mv}_s(\alpha)-\mv_s(\alpha)) \, .
\end{equation*}
The bootstrap approach suggests to consider, as an estimate of the law of the fluctuation process $r_n$, the conditional law given the original sample $\mathcal{D}_n=(X_1,\ldots,X_n)$ of the \textit{naive bootstrapped fluctuation process}
\begin{equation}\label{eq:naive_fluctuat}
r^{Boot}_{n}=\{\sqrt{n}(\widehat{\mv}^{Boot}_s(\alpha)-\widehat{\mv}_s(\alpha))\}_{\alpha \in [0,1)} \, ,
\end{equation}
where, given $\mathcal{D}_n$, $\widehat{\mv}^{Boot}_s$ is the empirical $\mv$ curve of the scoring function $s$ based on a sample of \textit{i.i.d.} random variables with distribution $\widehat{F}_s=(1/n)\sum_{i\leq n}\delta_{s(X_i)}$.
The difficulty is twofold. First, the target is a distribution on a \textit{path space}, namely a subspace of the Skorohod's space $\mathbb{D}([0,1-\varepsilon])$ equipped with the $\sup$ norm. Second, $r_n$ is a functional of the quantile process $\{\widehat{F}_s^{\dag}(\alpha)\}_{\alpha\in [\varepsilon,1]}$. The \textit{naive bootstrap}, which consists in resampling from the raw empirical distribution $\widehat{F}_s$, generally provides bad approximations of the distribution of empirical quantiles: the rate of convergence for a given quantile is indeed of order $O_{\mathbb{P}}(n^{-1/4})$ \citep{Reiss89} whereas the rate of the Gaussian approximation is $O(n^{-1/2}\log n)$ (see \eqref{eq:first_bootstrap_sup} in Appendix~\ref{sec:appendix_proofs}). The same phenomenon may be naturally observed for $\mv$ curves. In a similar manner to what is usually recommended for empirical quantiles, a \textit{smoothed version} of the bootstrap algorithm shall be implemented in order to improve the approximation rate of the distribution of $\sup_{\alpha\in[0,1-\varepsilon]}\vert r_n(\alpha)\vert$, namely, to resample the data from a smoothed version $\widetilde{F}_s$ of the empirical distribution $\widehat{F}_s$. We thus consider the \emph{smooth boostrapped fluctuation process}
\begin{equation}\label{eq:smooth_fluctuat}
r^*_n=\{\sqrt{n}(\mv^{Boot}_s(\alpha)-\widetilde{\mv}_s(\alpha))\}_{\alpha \in [0,1)} \, ,
\end{equation}
where, given $\mathcal{D}_n$, $\mv^{Boot}_s=\lambda_s\circ (\alpha^{Boot}_s)^{-1}$ is the empirical $\mv$ curve of the scoring function $s$ based on a sample of \textit{i.i.d.} random variables with distribution $\widetilde{F}_s$ and where $\widetilde \mv_s = \lambda_s \circ \widetilde \alpha_s^{-1}$ is the smooth version of the empirical $\mv$ curve, $\widetilde \alpha_s^{-1}$ being the generalized inverse of $\widetilde \alpha_s$.
The algorithm for building a confidence band at level $1-\eta$ in the $\mv$ space from sampling data $\mathcal{D}_n=\{X_i:\; i=1,\; \ldots, \; n\}$ is described in Algorithm \ref{algo:bootstrap}.

\begin{algorithm}
\caption{Smoothed $\mv$ curve bootstrap}
\label{algo:bootstrap}
\begin{algorithmic}[1]
\State Based on the sample $\mathcal{D}_n$, compute the smoothed version $\{\widetilde{\alpha}^{-1}_s(\alpha), \alpha \in [0,1)\}=\{\widetilde{F}^{\dag}_s(1-\alpha), \alpha \in [0,1)\}$ of the empirical estimate $\{\widehat \alpha^{-1}_s(\alpha), \alpha \in [0,1)\}$.
\State Plot the smooth $\mv$ curve estimate, \textit{i.e.}, the graph of the mapping
\begin{equation*}
\alpha \in [0,1) \mapsto \widetilde{\mv}_s(\alpha)=\lambda_s \circ \widetilde\alpha_s^{-1}(\alpha) \, .
\end{equation*}
\State Draw a bootstrap sample $\mathcal{D}_n^{Boot} \sim \widetilde{F}_s \; \vert \; \mathcal{D}_n $.
\State Based on $\mathcal{D}_n^{Boot}$, compute the smoothed bootstrap version $\{(\alpha^{Boot}_s)^{-1}(\alpha), \alpha \in [0,1)\}$ of the empirical estimate $\{\widehat \alpha^{-1}_s(\alpha), \alpha \in [0,1)\}$.
\State Plot the bootstrap $\mv$ curve, \textit{i.e.} the graph of the mapping
\begin{equation*}
\alpha\in [0,1)\mapsto {\mv}^{Boot}_s(\alpha)=\lambda_s\circ (\alpha^{Boot}_s)^{-1}(\alpha) \, .
\end{equation*}
\State Get the \textit{bootstrap confidence bands at level $1-\eta$} defined by the ball of center $\widehat{\mv}_s$ and radius $\nu_{\eta}/\sqrt{n}$ in $\mathbb{D}([\varepsilon,1-\varepsilon])$, where $\nu_{\eta}$ is defined by
\begin{equation*}
\mathbb{P}^*(\sup_{\alpha\in[\varepsilon,1-\varepsilon]}\vert r^*_n(\alpha) \vert \leq \nu_{\eta}) =1-\eta \, ,
\end{equation*}
denoting by $\mathbb{P}^*(.)$ the conditional probability given the original data $\mathcal{D}_n$.
\end{algorithmic}
\end{algorithm}



Before turning to the theoretical analysis of this algorithm, its description calls a few comments. From a computational perspective, the smoothed bootstrap distribution $\mathbb{P}^*(\sup_{\alpha\in[\varepsilon,1-\varepsilon]}\vert r^*_n(\alpha) \vert \leq \cdot)$ must be approximated in its turn, by means of a Monte-Carlo approximation scheme. Based on the $N$ bootstrap fluctuation processes obtained, $r_n^{*(j)}$ with $j=1,\; \ldots,\; N$, the radius $\nu_{\eta}$ then coincides with the empirical quantile at level $1-\eta$ of the statistical population $\{ \sup_{\alpha\in[\varepsilon,1-\varepsilon]}\vert r_n^{*(j)}(\alpha) \vert:\; j=1,\; \ldots,\; N \}$. Concerning the number of bootstrap replications, picking $N=n$ does not modify the rate of convergence. However, choosing $N$ of magnitude comparable to $n$ so that $(1+N)(1-\eta)$ is an integer may be more appropriate: the $(1-\eta)$-quantile of the approximate bootstrap distribution is then uniquely defined and this does not impact the rate of convergence neither \citep{Hall86}.

\subsubsection{Bootstrap consistency}
The next theorem reveals the asymptotic validity of the bootstrap estimate proposed above where we assume that the smoothed version $\widetilde F_s$ of the distribution function $\widehat F_s$ is computed at step 1 of Algorithm \ref{algo:bootstrap} using a kernel $K_{h_n}$. It requires the following assumptions.

\begin{assump_boot}[resume]
\item \label{as:density_bias} The density $f_s$ is bounded and of class $\mathcal{C}^3$.
\item \label{as:bandwidth_stute} The bandwidth $h_n$ decreases towards $0$ as $n \rightarrow +\infty$ in a way that $nh_n \rightarrow + \infty$, $nh_n/\log(h_n^{-1}) \rightarrow +\infty$ and $\log (h_n^{-1})/\log\log n\rightarrow +\infty$.
\item \label{as:kernel_bias} The kernel $K$ has finite support and satisfies the following conditions: $\int K(y)dy = 1$, $\int yK(y)dy = 0$ and $\int y^2K(y)dy < +\infty$.
\item \label{as:kernel_gine} The kernel $K$ is such that $\Vert K \Vert_2 < + \infty$ and is of the form $K(y) = \Phi_1(P_1(y))$, $P_1$ being a polynomial and $\Phi_1$ a bounded real function of bounded variation.
\item \label{as:kernelderivative_gine} The kernel $K$ is differentiable with derivative $K'$ such that $\Vert K' \Vert_2 < + \infty$ and of the form $K'(y) = \Phi_2(P_2(y))$, $P_2$ being a polynomial and $\Phi_2$ a bounded real function of bounded variation.
\item \label{as:bandwidth_derivative} The bandwidth $h_n$ is such that $\log(h_n^{-1})/(nh_n^3)$ tends to 0 as $n \rightarrow + \infty$.
\end{assump_boot}

Assumptions \ref{as:density_bias} and \ref{as:kernel_bias} are needed to control the biases $\sup_{\mathbb{R}}\vert \mathbb{E}[\widetilde f_s] - f_s \vert$ and $\sup_{\mathbb{R}}\vert \mathbb{E}[\widetilde{f_s}'] - f_s' \vert$. Assumption \ref{as:bandwidth_stute} on the bandwidth $h_n$ and assumption \ref{as:kernel_gine} (respectively assumption \ref{as:kernelderivative_gine}) are needed to control $\sup_{\mathbb{R}} \vert \widetilde f_s - \mathbb{E}[\widetilde f_s] \vert$ (respectively $\sup_{\mathbb{R}} \vert \widetilde{f_s}' - \mathbb{E}[\widetilde f_s'] \vert$) thanks to the result of \citet{Gine}. Assumption \ref{as:bandwidth_derivative} ensures that $\sup_{\mathbb{R}}\vert \widetilde{f_s}' - f_s'\vert$ tends to $0$ almost surely as $n \rightarrow \infty$. Assumptions \ref{as:kernel_bias}-\ref{as:kernelderivative_gine} are fulfilled, for instance, by the biweight kernel defined for all $t \in \mathbb{R}$ by:
\begin{equation}
\label{eq:biweight_kernel}
K(t) = \frac{15}{16}(1-t^2)^2 \cdot \mathbb{I}\{\vert t \vert \leq 1 \} \, .
\end{equation}

\begin{theorem} \label{thm:bootstrap} ({\sc Asymptotic validity}) Let $\varepsilon\in (0,1)$ and let $\mathbb{P}^*(.)$ denote the conditional probability given the original data $\mathcal{D}_n$. Suppose that assumptions \ref{as:bounded_score}-\ref{as:tail}, \ref{as:regulartiy_lambda} and \ref{as:density_bias}-\ref{as:bandwidth_derivative} are fulfilled. Then, the distribution estimate $\mathbb{P}^*(\sup_{\alpha\in[\varepsilon,1-\varepsilon]}\vert r^*_n(\alpha) \vert \leq \cdot)$ given by Algorithm \ref{algo:bootstrap} is such that we almost surely have
\begin{equation*}
\sup_{t\in \mathbb{R}}\left\vert\mathbb{P}^*\left(\sup_{\alpha\in[\varepsilon,1-\varepsilon]}\vert r^*_n(\alpha) \vert \leq t\right)-\mathbb{P}\left(\sup_{\alpha\in[\varepsilon,1-\varepsilon]}\vert r_n(\alpha) \vert \leq t\right)\right\vert = O(w_n).
\end{equation*}
where $w_n = \sqrt{\log(h_n^{-1})/nh_n} + h_n^2$.
\end{theorem}
The proof is given in Appendix~\ref{sec:appendix_proofs}. Note that under assumption \ref{as:bandwidth_stute} $w_n$ tends to $0$ as $n \rightarrow + \infty$. The primary tuning parameters of Algorithm \ref{algo:bootstrap} concern the bandwidth $h_n$. The optimal bandwidth is obtained by minimizing the term $w_n$ with respect to $h_n$. This leads to $h_n \sim (\log n / n)^{1/5}$ and an approximation error of order $O((\log n/n)^{2/5})$ for the bootstrap estimate (see \citep{Stute1982} for details on the derivation of the optimal bandwidth). Notice that assumptions \ref{as:bandwidth_stute} and \ref{as:bandwidth_derivative} are fulfilled by such a bandwidth.

Although the rate of the bootstrap estimate is slower than that of the \mbox{Gaussian} approximation, the smoothed bootstrap method remains very appealing from a computational perspective: it is indeed very difficult to build confidence bands from simulated Brownian bridges in practice. Finally, as said above, it should be noticed that a non smoothed bootstrap of the $\mv$ curve would lead to worse approximation rates, of the order $O_{\mathbb{P}}(n^{-1/4})$ namely, see \citep{Reiss89}.

\begin{remark}
In the result of Theorem \ref{thm:bootstrap} we consider the supremum over $\alpha \in [\varepsilon, 1-\varepsilon]$ instead of the supremum over $\alpha \in [0, 1-\varepsilon]$ for several reasons. One of the reasons is that to obtain a rate of convergence for the bootstrap approximation we need to have the boundedness of the density of the supremum of the absolute value of the Gaussian process $Z_1$. This is almost immediate from the result of \citet{Pitt1979} if the supremum is over $[\varepsilon, 1-\varepsilon]$ (see the proof of Theorem \ref{thm:bootstrap} in Appendix~\ref{sec:appendix_proofs} for more details). This is more complicated if the supremum is over $[0, 1-\varepsilon]$. Using the result of \citet{Lifshits}, a closer look at the Gaussian process $Z_1$ might lead to the boundedness of the density of the supremum of $\vert Z_1 \vert$ over $[0, 1-\varepsilon]$. Another reason is that several arguments in the proof also use the fact that $\inf_{[\varepsilon, 1-\varepsilon]}f_s \circ \alpha_s^{-1} > 0$. However we do not have $\inf_{[0, 1-\varepsilon]}f_s \circ \alpha_s^{-1} > 0$ as under the assumptions of Theorem \ref{thm:bootstrap} $f_s$ is continuous on $\mathbb{R}$ and therefore $f_s(\alpha_s^{-1}(0)) = 0$.
\end{remark}

\subsubsection{Illustrative Numerical Experiments}

Let $x \in \mathbb{R} \mapsto \mathcal{N}(\mu, \Sigma)(x) \in \mathbb{R}$ denote the density of a Gaussian distribution with mean $\mu$ and covariance $\Sigma$. We consider a two-dimensional Gaussian mixture whose density is given by:
\begin{equation}
\label{eq:density_gm}
\forall x \in \mathbb{R}, \quad f(x) = 0.5\mathcal{N}(\mu_1,\Sigma_1)(x) + 0.5\mathcal{N}(\mu_2,\Sigma_2)(x)
\end{equation}
where $\mu_1 = (0, 0)$, $\mu_2 = (-1, -1)$,
\begin{equation*}
\Sigma_1 = 
 \begin{pmatrix}
  2 & 2 \\
  2 & 4
 \end{pmatrix}
\text{ and }
\Sigma_2 = 
 \begin{pmatrix}
  2 & 0 \\
  0 & 2
 \end{pmatrix}
 \, .
\end{equation*}
Density level sets of such a distribution are shown in Fig.~\ref{fig:disc}. We draw a sample of size $n=500$ from this two-dimensional Gaussian mixture and compute $\widehat \mv_f$.
We then apply Algorithm \ref{algo:bootstrap} using the biweight kernel defined in \eqref{eq:biweight_kernel} with a bandwidth $h=0.005$ to obtain a $90$\% confidence band. We take $\varepsilon=0.05$ and $N=n$ bootstrap replications to approximate $\mathbb{P}^*(\sup_{\alpha \in [\varepsilon, 1-\varepsilon]}\vert r_n^*(\alpha)\vert \leq \cdot)$. All the volumes required to compute the $\mv$ curves are estimated using Monte-Carlo integration, drawing uniformly $1,\!000,\!000$ samples in the hypercube enclosing the data. Fig.~\ref{fig:bootstrap_exp} shows the $90$\% confidence curves.

\begin{figure}[!ht]
\centering
\includegraphics[width=10cm]{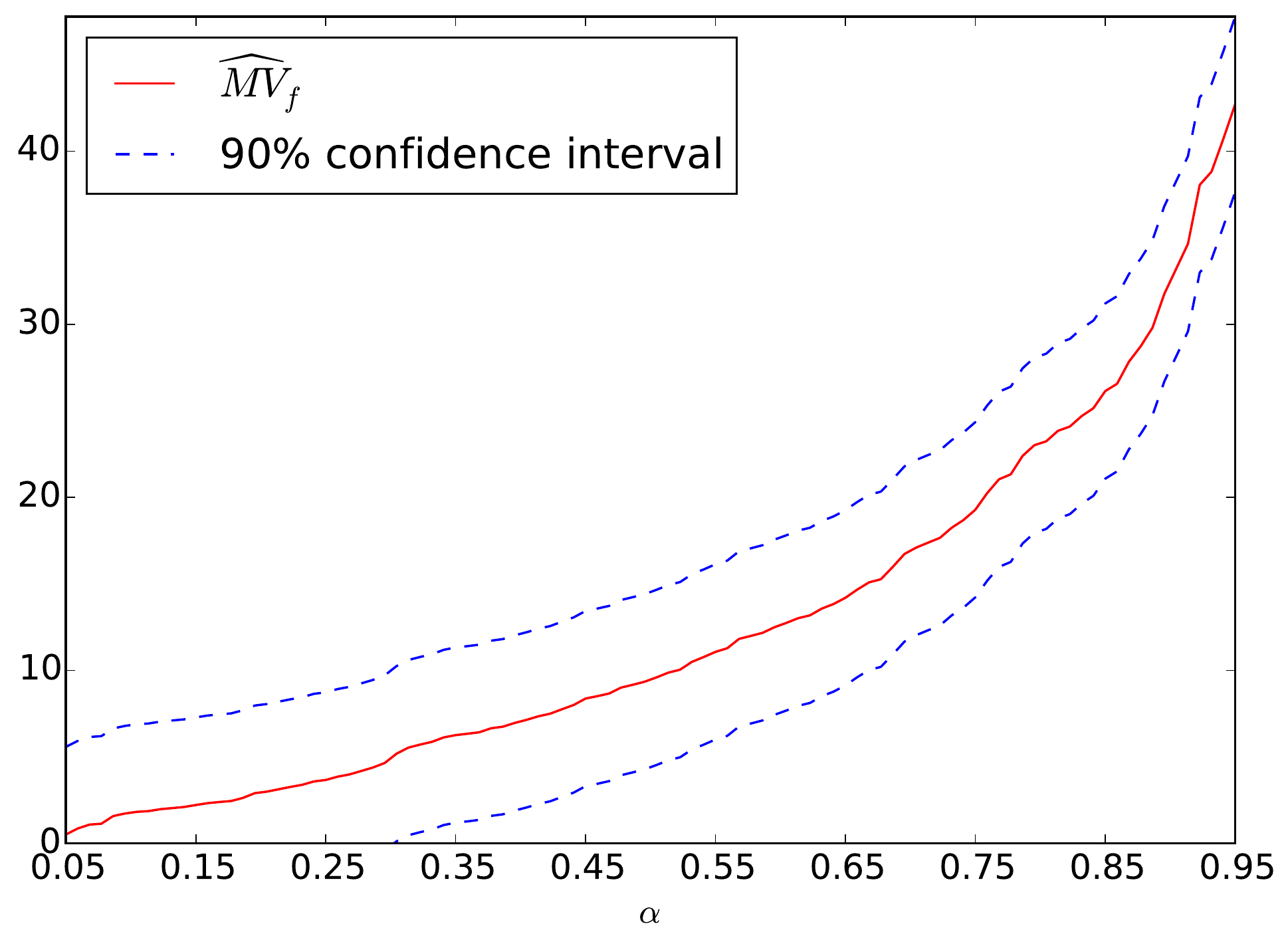}
\caption{\label{fig:bootstrap_exp} Empirical MV curve and $90$\% confidence interval curves obtained with the smooth bootstrap approach. Note that it seems that the $90$\% confidence interval curves get closer to the empirical $\mv$ curve as $\alpha$ increases but this is just an optical illusion.}
\end{figure}


In this section, statistical estimation of the true $\mv$ curve of a given scoring function $s$ has been investigated, as well as the problem of building confidence regions in the Mass Volume space. In all the rest of the paper, focus is on statistical learning of a nearly optimal scoring function $s$ based on a training sample $X_1, \dots, X_n$ with respect to the $\mv$ curve criterion.

%% file: Sec5_MEstimation.tex

\section{A M-estimation Approach to Anomaly Scoring}\label{sec:form}
Now we are are equipped with the concept of Mass Volume curve, the anomaly scoring task can be formulated as the building of a scoring function $s$, based on the training set $X_1,\ldots,X_n$, such that $\mv_s$ is as close as possible to the optimum $\mv^*$. Due to the functional nature of the criterion performance, there are many ways of measuring how close the $\mv$ curve of a scoring function candidate and the optimal one are. The $L_p$-distances, for $1\leq p\leq +\infty$, provide a relevant collection of risk measures. Let $\varepsilon\in (0,1)$ be fixed (take $\varepsilon=0$ if $\lambda(\supp(f))<+\infty$) and consider the losses related to the $L_1$-distance and that related to the sup norm:
\begin{eqnarray*}
d_1(s,f)&=&\int_0^{1-\varepsilon}\vert \mv_s(\alpha)-\mv^*(\alpha)\vert d\alpha \, ,\\
d_{\infty}(s,f)&=&\sup_{\alpha\in [0,1-\varepsilon]}\{\mv_s(\alpha)-\mv^*(\alpha)\} \, .
\end{eqnarray*}
Observe that, by virtue of Proposition \ref{prop:opt}, the `excess-risk' decomposition applies in the $L_1$ case and the learning problem can be directly tackled through standard $M$-estimation arguments:
\begin{equation*}
d_1(s,f)=\int_0^{1-\varepsilon}\mv_s(\alpha)d\alpha-\int_0^{1-\varepsilon}\mv^*(\alpha)d\alpha \, .
\end{equation*}
Hence, possible learning techniques could be based on the minimization, over a set $\S_0\subset \S$ of candidates, of empirical counterparts of the area under the $\mv$ curve, such as $\int_0^{1-\varepsilon}\widehat{\mv}_s(\alpha)d\alpha$. In contrast, the approach cannot be straightforwardly extended to the sup norm situation. A possible strategy is to combine $M$-estimation with approximation methods so as to `discretize' the functional optimization task. This strategy can be implemented as follows. First, we replace the unknown target curve $\mv^*$ by an approximation that can be described by a finite number of scalar parameters, by a piecewise constant approximant $\mv^*_{\sigma}$ whose breakpoints are given by a subdivision $\sigma:\; 0<\alpha_1<\dots<\alpha_K=1-\varepsilon$ precisely, and that is itself a $\mv$ curve, namely the $\mv$ curve of the piecewise constant scoring function
\begin{equation}\label{eq:approximant}
s^*_{\sigma}(x)=\sum_{k=1}^K(K-k+1)\cdot \mathbb{I}\{x\in\Omega^*_{\alpha_{k}}\setminus\Omega^*_{\alpha_{k-1}} \}
\end{equation}
which is indeed a piecewise constant approximant of $\mv^*$ related to the meshgrid $\sigma$ (see \eqref{eq:piecewise_mv_curve}). Then, the $L_{\infty}$-risk can be decomposed as the sum of two terms 
\begin{equation*}
d_{\infty}(s,f)\leq  \sup_{\alpha\in [0,1-\varepsilon]}\vert \mv_s(\alpha) - \mv^*_{\sigma}(\alpha)\vert +
\sup_{\alpha\in [0,1-\varepsilon]}\{\mv^*_{\sigma}(\alpha)-\mv^*(\alpha)\} \, ,
\end{equation*}
the second term on the right-hand side being viewed as the \textit{bias} of the statistical method. Restricting optimization to the first term on the right-hand side of the $L_{\infty}$-risk decomposition, the problem thus boils down to recovering the bilevel sets $\mathcal{R}^*_k=\Omega^*_{\alpha_{k}}\setminus\Omega^*_{\alpha_{k-1}}$ for $k=1,\ldots,K$ as we obviously have
\begin{equation*}
s^*_{\sigma} \in \argmin_{s \in \mathcal{S}} \sup_{\alpha\in [0,1-\varepsilon]}\vert \mv_s(\alpha) - \mv^*_{\sigma}(\alpha)\vert \, .
\end{equation*}
This simple observation paves the way for designing scoring strategies relying on the estimation of a finite number of minimum volume sets.

In the next section, we describe a learning algorithm for anomaly ranking, that can be viewed to a certain extent as a statistical version of an adaptive approximation method by piecewise constants introduced in \citep{DeVore87} (see also Section~3.3 in \citep{DeVore98}), to build a piecewise constant estimate of the optimal curve $\mv^*$ and a nearly optimal piecewise constant scoring function, mimicking $s^*_{\sigma}$. The subdivision $\sigma$ is entirely learnt from the data, in order to produce accurate estimates of $\mv^*$ in an adaptive fashion: looking at Fig. \ref{fig:mv_curves_examples}, an ideal meshgrid should be loose where $\mv^*$ is nearly flat or grows very slowly (`near' $0$) and refined when it exhibits high degrees of variability (as one gets closer to $1$).

Before describing and analyzing a prototypal approach to $\mv$ curve optimization, a few remarks are in order.
\begin{remark}\label{rk:connect_sup} {\sc (Connections with supervised ranking)} Based on the observation made in Remark \ref{rk:ROC}, one may see that, in the specific case where the support of $F(dx)$ coincides with the unit square $[0,1]^d$, the $\mv$ curve of any scoring function $s(x)$ corresponds to the reflection about the first diagonal of the $\roc$ curve of $s(x)$ when the `negative distribution' is the uniform distribution on $[0,1]^d$ and the `positive distribution' is $F(dx)$. As shown in \citep{ClemRob14}, this permits to turn unsupervised ranking into supervised ranking in the compact support situation and to exploit supervised ranking algorithms combined with random sampling to solve the $\mv$ curve minimization problem. A similar idea had been proposed in \citep{SHS05} to turn anomaly detection into supervised binary classification.
\end{remark}

\begin{remark} {\sc (Plug-in)} As the density $f$ is an optimal scoring function, a natural strategy would be to estimate first the unknown density function $f$ by means of (non-) parametric techniques and next use the resulting estimator as a scoring function. Beyond the computational difficulties one would be confronted to for large or even moderate values of the dimension, we point out that the goal pursued in this paper is by nature very different from density estimation: the local properties of the density function are useless here, only the ordering of the possible observations $x\in \X$ it induces is of importance (see Proposition \ref{prop:opt}). One may also show that a candidate $f_1$ can be a better approximation of the density $f$ than a candidate $f_2$ for the $L_q$ loss say, but a worse approximation for the $\mv$ curve criterion (see Example \ref{ex:plugin}).
\end{remark}

\begin{example}
\label{ex:plugin}
Let $f$ be the density of the truncated normal distribution with support $[0, 1]$, mean equal to 0.5 and variance equal to $0.0225$. Let's consider $f_1 : x \in \mathbb{R} \mapsto f(x - 0.05) \in \mathbb{R}$ and $f_2$ defined for all $x \in \mathbb{R}$ as
\begin{equation*}
f_2(x)=
\begin{cases}
f(x) + \Vert f_1 - f\Vert_{L_2} + 2 \text{ if } x \in [0,1] \\
0 \text{ otherwise} \, .
\end{cases}
\end{equation*}

We thus have $\Vert f_2 - f \Vert_{L_2} = \Vert f_1 - f \Vert_{L_2} + 2 > \Vert f_1 - f \Vert_{L_2}$. Hence $f_1$ is a better approximation of $f$ with respect to the $L_2$ distance. However $\mv_{f_2} < \mv_{f_1}$ and $f_2$ is a better scoring function than $f_1$ with respect to the $\mv$ curve criterion. Indeed one can first show that $\mv_{f_2} = \mv_{f}$. This thus gives $\mv_{f_2} = \mv^* \leq \mv_{f_1}$. Now let $\alpha \in (0, 1)$, $\Omega_{f_1, Q(f_1, \alpha)} = \{x, f_1(x) \geq Q(f_1, \alpha) \}$ is of the form $[0.05 - x_{\alpha}, 0.05 + x_{\alpha}]$ with $x_{\alpha} > 0$ and $F(\Omega_{s, Q(f_1, \alpha)}) \geq \alpha$. The set $\Omega_{f, Q^*(\alpha)}$ is of the form $[- x^*_{\alpha}, x^*_{\alpha}]$ with $x^*_{\alpha} > 0$ and $F(\Omega_{f, Q^*(\alpha)}) = \alpha$. If $\mv_{f_1}(\alpha) = \mv^*(\alpha)$, i.e., $\lambda(\Omega_{f_1, Q(f_1, \alpha)}) = \lambda(\Omega_{f, Q^*(\alpha)})$, then on the one hand $x_{\alpha} = x^*_{\alpha}$ and on the other hand $\Omega_{f_1, Q(f_1, \alpha)} = \Omega_{f, Q^*(\alpha)}$ $\lambda$ almost everywhere by the uniqueness of the solution of $\eqref{eq:mvpb}$ which is impossible. Eventually, $\mv_{f_2} < \mv_{f_1}$. One can also observe that in contrary to $f_1$, $f_2$ preserves the order induced by $f$.
\end{example}

%% file: Sec6_AdaptiveARank.tex

\section{The {\sc A-Rank} Algorithm} \label{sec:learning}
 
Now that the anomaly scoring problem has been rigorously formulated, we propose a statistical method to solve it and establish learning rates for the sup norm loss.

\subsection{Piecewise Constant Scoring Functions}
We focus on scoring functions of the simplest form, piecewise constant functions. Let $K\geq 1$ and consider a partition $\mathcal{W}$ of the feature space $\mathcal{X}$ in $K$ pairwise disjoint subsets of finite Lebesgue measure: $\mathcal{C}_1,\dots,\mathcal{C}_K$ and the subset $\mathcal{X} \setminus \cup_{k=1}^K \mathcal{C}_k$. When $\lambda(\supp(f))$ is finite, one may suppose $\X \setminus \cup_{k=1}^K \mathcal{C}_k$ of finite Lebesgue measure. Then, define the piecewise constant scoring function given by:
\begin{equation*}
\forall x\in \X, \quad s_{\mathcal{W}}(x)=\sum_{k=1}^K (K-k+1)\cdot\mathbb{I}\{x\in \mathcal{C}_{k}\} \, .
\end{equation*}
Its piecewise constant $\mv$ curve is given by: $\forall \alpha \in [0, F(\cup_{k=1}^K \mathcal{C}_k)),$
\begin{equation}
\label{eq:piecewise_mv_curve}
\mv_{s_{\mathcal{W}}}(\alpha)=\sum_{k=0}^{K-1}\lambda_{k+1}\cdot \mathbb{I}\{\alpha\in [\alpha_{k},\; \alpha_{k+1})\} \, ,
\end{equation}
where $\alpha_{0}=0$, $\alpha_{k}=F(\cup_{j=1}^k\mathcal{C}_j)$ and $\lambda_{k}=\lambda(\cup_{j=1}^k\mathcal{C}_j)$ for all $k \in \{1,\dots,K\}$. If $\lambda(\supp(f))$ is finite, $\mv_{s_{\mathcal{W}}}$ can be defined on the whole interval $[0, 1]$.


\subsection{Adaptive Approximation of the Optimal MV Curve}\label{subsec:approx}

Using ideas developed in \citep{DeVore87} (see also Section~3.3 in \citep{DeVore98}), we propose to design an adaptive approximation scheme instead of a subdivision fixed in advance. In such a procedure, the subdivision is progressively refined by adding new breakpoints, as further information about the local variation of the target $\mv^*$ is gained: the subdivision will be coarse where the optimal $\mv^*$ is almost flat and fine where it grows rapidly.

We restrict ourselves only to dyadic subdivisions with breakpoints $\alpha_{j,k} = k(1-\varepsilon)/2^j$, with $j \in \mathbb{N}$ and $k \in \{0, \dots, 2^j \}$ and to partitions of the interval $[0,1-\varepsilon]$ produced by recursive dyadic partitioning: any dyadic subinterval $I_{j,k} = [\alpha_{j,k}, \alpha_{j, k+1})$ is possibly split into two halves, producing two siblings $I_{j+1,2k}$ and $I_{j+1,2k+1}$, depending on the local properties of the target $\mv^*$. This adaptive estimation algorithm can be viewed as a top-down search strategy through a binary tree structure.

A binary tree $\T$ is a tree where all internal nodes have exactly two children. The root of the tree $\T$ represents the interval $[0,1-\varepsilon]$ and each node of depth $j$ a subinterval $I_{j,k}$. The two children resulting of a split of the node $I_{j,k}$ are the nodes $I_{j+1,2k}$ and $I_{j+1,2k+1}$. A terminal node or a leaf is a node without children. In the sequel, $I_{j,k}$ will denote interchangeably the interval and the related node. Equipped with this flexible tree structure, $\T$ define a possibly very heterogeneous subdivision $\sigma_{\T}$ of $[0,1-\varepsilon]$

The approximant associated with the the binary tree $\T$ is given by: $\forall \alpha\in[0,1-\varepsilon]$,
\begin{equation*}
\mv_{\sigma_{\T}}(\alpha)=\sum_{\substack{
j,k \\
I_{j,k} \text{ leaf}}
} \mv^*(\alpha_{j,k+1}) \cdot \mathbb{I}\{\alpha \in I_{j,k}  \}
\end{equation*}
where the sum is over all the $j, k$ such that $I_{j,k}$ is a leaf. In order to decide whether a subinterval $I = [\alpha_1, \alpha_2] \subset [0, 1-\varepsilon]$ should be split or not, we use the local error on the subinterval $I$ defined by
\begin{equation*}
\mathcal{E}(I) = \mv^*(\alpha_2) -\mv^*(\alpha_1) \, .
\end{equation*}

The local error $\mathcal{E}(I)$ provides a simple way of estimating the variability of the nondecreasing function $\mv^*$ on the interval $I$. This measure is nonnegative and additive: for any siblings $I_1$ and $I_2$ of the same subinterval $I$,
\begin{equation*}
\mathcal{E}(I) = \mathcal{E}(I_1) + \mathcal{E}(I_2) \, .
\end{equation*}
From \citep{DeVore87}, we know that $\mathcal{E}$ controls the approximation rate of $\mv^*$ by a constant on any interval $I$ in the sense that:
\begin{equation*}
\inf_{c \in [0,\infty)} \Vert \mv^*(.) - c \Vert_{I} \leq \mathcal{E}(I) \, ,
\end{equation*}
where for any function $g : \mathbb{R} \rightarrow \mathbb{R}$, $\Vert g \Vert_{I}$ denotes the sup norm over the interval $I$.

Given a tolerance $\tau > 0$, the general concept of the algorithm generating a binary tree $\mathcal{T}$ and leading to its associated approximant is as follows: if $\mathcal{E}(I) > \tau$, the interval $I$ is split in two siblings, otherwise $I$ is not split and becomes a leaf. If the algorithm ends we thus have:
\begin{equation*}
\sup_{\alpha \in [0, 1-\varepsilon]} \{\mv^*_{\sigma_{\T}}(\alpha) - \mv^*(\alpha)\} \leq \tau.
\end{equation*}

An illustration is given in Fig. \ref{fig:adaptive_example}: Fig. \ref{fig:disc} shows the piecewise constant approximant of $\mv^*$ as well as the corresponding piecewise scoring function $s_{\T}(x)$ and Fig. \ref{fig:tree} shows the related binary tree.

\begin{figure}[ht!]
\centering
\subfigure[Piecewise constant adaptive approximant of $\mv^*$ (on the left) and associated piecewise constant scoring function (on the right).]{%
\includegraphics[width=0.99\textwidth]{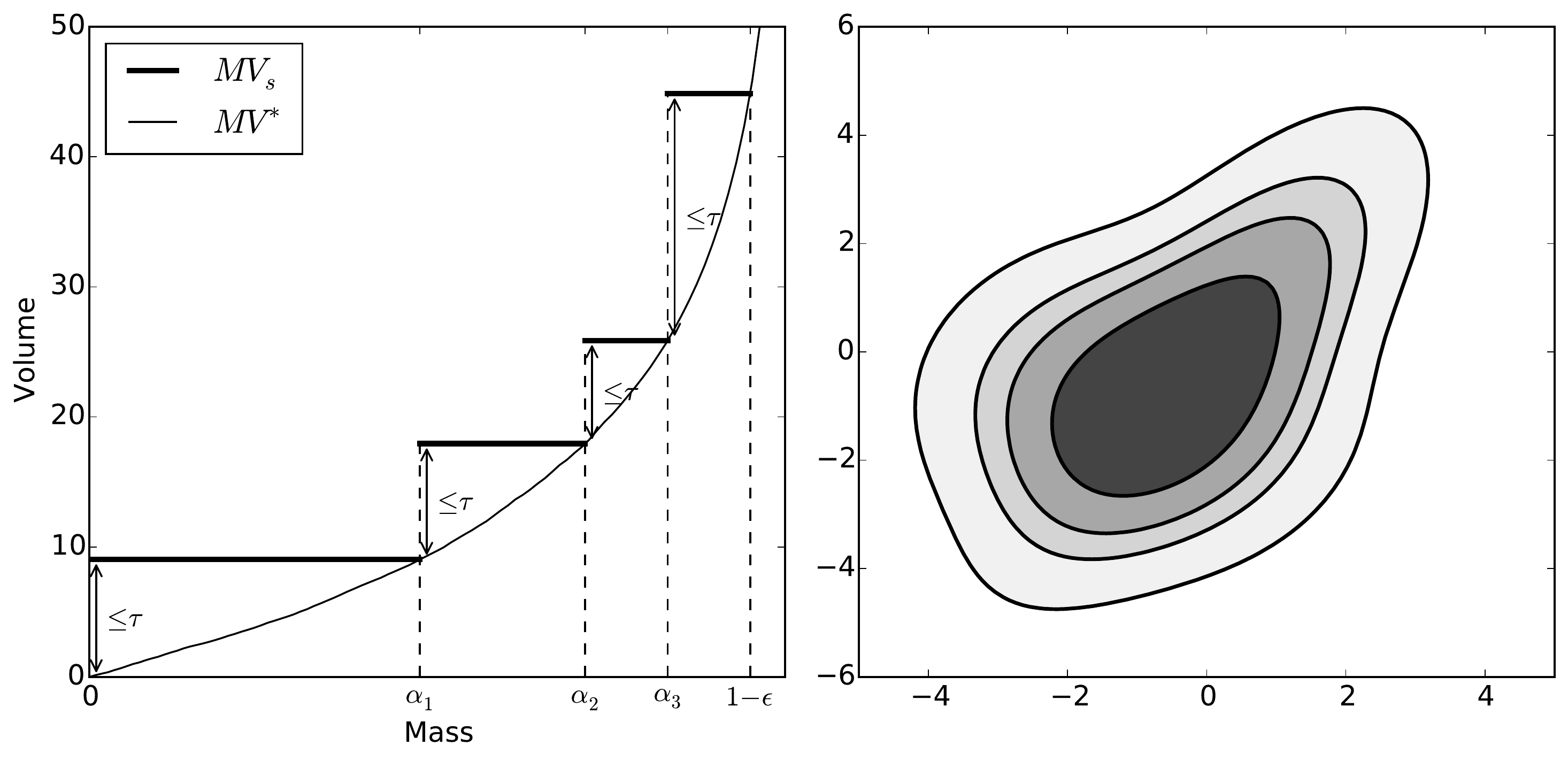}
\label{fig:disc}}
\subfigure[Associated binary tree]{%
\includegraphics[width=0.75\textwidth]{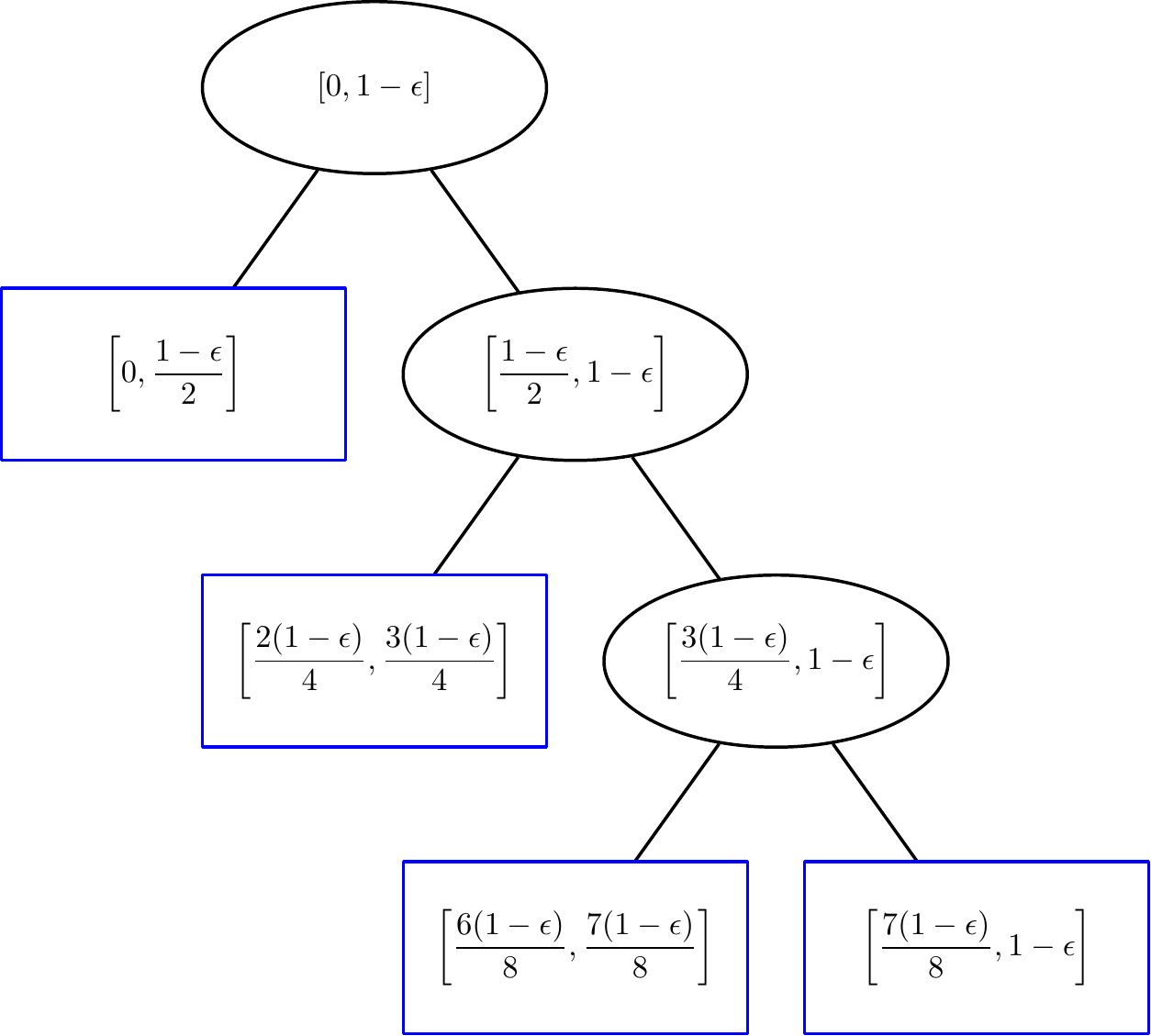}
\label{fig:tree}}
\caption{Illustration of the output of the adaptive approximation by piecewise constants of $\mv^*$ for a tolerance $\tau=20$ and where $\mv^*=\mv_f$, $f$ being the density given in \eqref{eq:density_gm}. Note that this algorithm is run assuming that we know $f$ and therefore that we have access to $\mv^*$ and the minimum volume sets $\{\Omega^*_{\alpha}$, $\alpha \in (0,1)\}$.}
\label{fig:adaptive_example}
\end{figure}

\subsection{Empirical Adaptive Estimation of the Optimal MV Curve}
The adaptive approximation procedure described above can be turned itself into an adaptive estimation technique.
As the optimal curve $\mv^*$ is unknown, the local error $\mathcal{E}(I)$ on a subinterval $I$ is estimated by its empirical counterpart
\begin{equation*}
\widehat{\mathcal{E}}(I) = \lambda(\widehat \Omega_{\alpha_2}) - \lambda(\widehat \Omega_{\alpha_1}) \, ,
\end{equation*}
where $\widehat{\Omega}_{\alpha}$ is an estimation of the minimum volume set $\Omega^*_{\alpha}$. 
Let $\mathcal{G}$ be a class of measurable subsets of the feature space $\X$ of finite Lebesgue measure and $\alpha\in (0,1)$. The empirical minimum volume set $\widehat \Omega_{\alpha}$ related to the class $\mathcal{G}$ and level $\alpha$ is the solution of the optimization problem:
\begin{equation}\label{eq:empminvolset}
\inf_{\Omega\in \mathcal{G}} \lambda(\Omega) \text{ subject to } \widehat F(\Omega)\geq \alpha-\phi \, ,
\end{equation}
where $\phi$ is a penalty term related to the complexity of the class, referred to as the \textit{penalty parameter}. The accuracy of the solution depends on the choices for the class $\mathcal{G}$ and for the penalty parameter (see \citep{ScottNowak06}). In this paper, we do not address the issue of designing algorithms for empirical minimum volume set estimation, we refer to Section~7 of \citep{ScottNowak06} for a description of off-the-shelf methods documented in the literature, including partition-based techniques.

As $\mathcal{E}(I)$, the empirical local error $\widehat{\mathcal{E}}(I)$ is nonnegative and additive:
\begin{equation*}
\widehat{\mathcal{E}}(I) = \widehat{\mathcal{E}}(I_1) + \widehat{\mathcal{E}}(I_2)
\end{equation*}
for any siblings $I_1$ and $I_2$ of the same subinterval.

\paragraph{Splitting criterion.} The empirical local error $\widehat{\mathcal{E}}$ will serve as a splitting criterion. Given a tolerance $\tau > 0$, if $\widehat{\mathcal{E}}(I) > \tau$, the interval $I$ is split in two subintervals $I_1$ and $I_2$. If $\widehat{\mathcal{E}}(I) \leq \tau$, the interval $I$ is not split and is a leaf of the tree $\T$.

\paragraph{Stopping criterion.} One can observe that the estimation algorithm does not necessarily terminate for any given tolerance $\tau > 0$. Indeed, as $\widehat F(\Omega) \in \{k/n: k = 0, \dots , n\}$ for any $\Omega \in \mathcal{G}$, the function $\widehat \lambda: \alpha \in [0,1) \mapsto \lambda(\widehat \Omega_{\alpha}) \in \mathbb{R}$ is a piecewise constant function with breakpoints $\{k/n + \phi: k = 0, \dots, n-1 \}$. This follows by observing that, for any $\alpha \in (k/n + \phi, (k+1)/n + \phi]$, $0 \leq k \leq n-1$, the solution of the empirical optimization problem \eqref{eq:empminvolset} is $\widehat \Omega_{(k+1)/n + \phi}$. Therefore if $\tau$ is strictly lower than the minimum of the amplitude of the jumps of the piecewise constant function $\widehat \lambda$, the algorithm does not stop.

However, the fact that the function $\widehat \lambda$ is a piecewise constant function tells us that the estimation algorithm should stop when the level $j_{\text{max}} = \lfloor \log n/\log 2 \rfloor + 1$ is reached, where $\lfloor \cdot \rfloor$ denotes the floor part function. Indeed if $j = j_{\text{max}}$, the interval $I_{j,k} = [\alpha_1, \alpha_2]$ is such that $\alpha_2 - \alpha_1 < 1/n$. This either means that $\alpha_1$ and $\alpha_2$ belong to the same interval $(k/n + \phi, (k+1)/n + \phi]$ or that $\alpha_1 \in (k/n + \phi, (k+1)/n + \phi]$ and $\alpha_2 \in ((k+1)/n + \phi, (k+2)/n + \phi]$. In the former case $\lambda(\widehat \Omega_{\alpha_2}) = \lambda(\widehat \Omega_{\alpha_1})$ and $\widehat{\mathcal{E}}(I_{j,k}) = 0 \leq \tau$. In the latter case, $\widehat{\mathcal{E}}(I_{j,k})$ is equal to the amplitude of the jump of $\widehat \lambda$ at $(k+1)/n + \phi$. If this amplitude is lower than $\tau$, $I_{j,k}$ is a leaf. If this amplitude is strictly greater than $\tau$, we should normally split $I_{j,k}$. However this split will not improve the error $\widehat{\mathcal{E}}$ as $\widehat{\mathcal{E}}$ cannot be lower than the amplitude of the jump of $\widehat \lambda$. We therefore add a condition $j < j_{\text{max}}$ to the algorithm. This ensures that the algorithm terminates while not changing the output when $\tau$ is greater that the maximum of the amplitudes of the jumps of $\widehat \lambda$. Note however that we do not necessarily have $\widehat{\mathcal{E}}(I_{j,k}) \leq \tau$ for all $j, k$ such that $I_{j,k}$ is a leaf. \newline

\begin{algorithm}[h]
\caption{Adaptive estimation of the optimal curve $\mv^*$}
\label{algo:adaptive}
\begin{algorithmic}[1]
\State \textbf{Inputs}: training set, penalty $\phi$, tolerance $\tau$
	\State Create a binary tree $\T$ with root node $I_{0,0} = [0,1-\varepsilon]$
	\State Create an empty list of nodes $\mathcal{L}$ 
	\State Add the node $I_{0,0}$ to the list $\mathcal{L}$
	\While{$\mathcal{L}$ is not empty}
		\State Get first element of $\mathcal{S}$: node $I_{j,k} = [\alpha_1, \alpha_2]$
		\If{$j < j_{\text{max}}$}
			\State Compute $\widehat{\mathcal{E}}(I_{j,k}) = \lambda(\widehat{\Omega}_{\alpha_2}) - \lambda(\widehat{\Omega}_{\alpha_1})$ where $\widehat{\Omega}_{\alpha}$ is the solution of \eqref{eq:empminvolset}
				\If{$\widehat{\mathcal{E}}(I_{j,k}) > \tau$} 
					\State Split $I_{j,k}$ in two siblings $I_{j+1,2k}$ and $I_{j+1,2k+1}$
					\State Add the node $I_{j+1,2k}$ to $\mathcal{L}$
					\State Add the node $I_{j+1,2k+1}$ to $\mathcal{L}$
				\EndIf
		\EndIf
	\EndWhile
\State \textbf{Output}: Let $\widehat \sigma_{\tau}$ denote the collection of the dyadic levels $\alpha_{j,k}$ corresponding to the leafs of the tree $I_{j,k}$.
\begin{equation*}
\widehat{\mv^*}(\alpha) = \sum_{\substack{
j,k \\
I_{j,k} \text{ leaf}}
} \lambda(\widehat{\Omega}_{\alpha_{j,{k+1}}}) \cdot \mathbb{I}\{\alpha \in I_{j,k}  \} = \sum_{\substack{
j,k \\
\alpha_{j,k} \in \widehat \sigma_{\tau}}
} \lambda(\widehat{\Omega}_{\alpha_{j,{k+1}}}) \cdot \mathbb{I}\{\alpha \in I_{j,k}  \} \, .
\end{equation*}
\end{algorithmic}
\end{algorithm}

The adaptive estimation method is summarized in Algorithm \ref{algo:adaptive}. We denote by $\hat \sigma_{\tau}: 0 = \alpha_0 < \dots < \alpha_{\widehat K} = 1-\varepsilon$ the subdivision output by the algorithm such that the empirical piecewise constant estimator of $\mv^*$ (and of $\mv^*_{\hat \sigma_{\tau}}$) is given by
\begin{equation*}
\forall \alpha \in [0, 1-\varepsilon],\; \widehat{\mv^*}(\alpha) = \sum_{k=0}^{\widehat K -1} \lambda(\widehat \Omega_{\alpha_{k+1}}) \cdot \mathbb{I} \{\alpha \in [\alpha_k, \alpha_{k+1}) \} \, .
\end{equation*}
Its bias incorporates in particular that caused by the approximation/discretization stage inherent to the method.

\subsection{The Anomaly Ranking algorithm {\sc A-Rank}}

The goal pursued is actually to build a scoring function $\hat s(x)$ whose $\mv$ curve is asymptotically close to the empirical estimate $\widehat{\mv^*}$. In this respect, one should pay attention to the fact that $\widehat{\mv^*}$ is not a $\mv$ curve in general. Indeed, the sequence of estimated minimum volume sets $\widehat{\Omega}_{\alpha_k}, k \in \{0, \dots, \widehat K \}$ sorted by increasing order of their mass $\alpha_k$, is not necessarily increasing (for the inclusion), in contrast to the true minimum level sets $\Omega^*_{\alpha_k}$. This explains the monotonicity step of the following algorithm. 

The {\sc A-Rank} algorithm is implemented in two stages. Stage 1 consists in running Algorithm \ref{algo:adaptive} to find the breakpoints $\alpha_0=0<\alpha_1<\dots<\alpha_{\widehat K}=1-\varepsilon<1$ and the corresponding empirical minimum volume sets $\widehat{\Omega}_{\alpha_0}, \dots, \widehat{\Omega}_{\alpha_{\widehat K}}$. Stage 2 consists in the monotonicity step before overlaying the estimated sets. The statistical learning method is described at length in Algorithm \ref{algo:arank}.

Before analyzing the statistical performance of the {\sc A-rank} algorithm, a few remarks are in order.

Notice first that we could alternatively build a monotone sequence of subsets $(\widetilde{\Omega}_k)_{0\leq k\leq \widehat K}$, recursively through: $\widetilde{\Omega}_{\widehat K}=\widehat{\Omega}_{\widehat K}$ and $\widetilde{\Omega}_{k}=\widehat{\Omega}_{k} \cap \widetilde{\Omega}_{k+1}$ for $k\in \{\widehat K-1,\ldots, 0\}$. The results established in the sequel straightforwardly extend to this construction.

One obtains as a byproduct of the algorithm a rough piecewise constant estimator of the (optimal scoring) function $(F_f\circ f)(x)$, namely $\sum_{k=1}^K (\alpha_{k+1}-\alpha_k)\cdot \mathbb{I}\{x \in \widetilde{\Omega}_k\}$, the Riemann sum approximating the integral \eqref{eq:form_opt} when $\mu$ is the Lebesgue measure.

\begin{algorithm}
\caption{{\sc A-Rank}}
\label{algo:arank}
\begin{algorithmic}[1]
\State \textbf{Inputs}: training set, penalty $\Phi$, tolerance $\tau$
	\State Get $(\alpha_k,\widehat{\Omega}_{\alpha_k})_{0 \leq k \leq \widehat K}$ from Algorithm \ref{algo:adaptive}
	\State Monotonicity step: $\widetilde{\Omega}_{\alpha_0} = \widehat{\Omega}_{\alpha_0}$
	\For{$k \in \{1, \dots, \widehat K \}$}
		\begin{equation*}
		\widetilde{\Omega}_{\alpha_k} = \widehat{\Omega}_{\alpha_k} \cup \widetilde{\Omega}_{\alpha_{k-1}}
		\end{equation*}
	\EndFor
\State \textbf{Output}:
\begin{equation*}
\hat s(x) = \sum_{k=0}^{\widehat K} (\widehat K - k + 1) \cdot \mathbb{I}\{ x \in \widetilde{\Omega}_{\alpha_k}\setminus \widetilde{\Omega}_{\alpha_{k-1}}\}
\end{equation*}
\end{algorithmic}
\end{algorithm}

\subsection{Performance Bounds for the {\sc A-Rank} Algorithm}
We now prove a result describing the performance of the scoring function produced by the algorithm proposed in the previous section to solve the anomaly scoring problem, as formulated in Section~\ref{sec:form}. The following assumptions are required.
 
\begin{assump_ml}[resume]
\item \label{as:mvset_in_class} $\forall \alpha \in[0, 1-\varepsilon]$, $\Omega^*_{\alpha}\in \mathcal{G}$.
\item \label{as:rademacher_order} Let $(\varepsilon_i)_{i\geq 1}$ be a Rademacher chaos independent from the $X_i$'s. The Rademacher average given by
\begin{equation}\label{eq:Rademacher}
\mathcal{R}_n=\mathbb{E}\left[ \sup_{\Omega\in \mathcal{G}}\frac{1}{n}\left\vert \sum_{i=1}^n\varepsilon_i \cdot \mathbb{I}\{X_i\in \Omega\} \right\vert \right] \, ,
\end{equation}
where the expectation is over all the random variables, is such that for all $n \geq 1$, $\mathcal{R}_n \leq Cn^{-1/2}$, where $C > 0$ is a constant.
\end{assump_ml}

Assumption \ref{as:rademacher_order} is very general, it is satisfied in particular when the class $\mathcal{G}$ is of finite {\sc VC} dimension, see \citep{Kolt06}. We recall here that for $\delta \in (0,1)$ we have \citep{ScottNowak06}:
\begin{equation}
	F\Bigl( \Bigl\{\sup_{\Omega \in \mathcal{G}} \vert \widehat F(\Omega) - F(\Omega)\vert > \phi_n(\delta) \Bigl\} \Bigr) \leq \delta
\end{equation}
with $\forall n \geq 1$,
\begin{equation}
\label{eq:phi_penalty}
\phi_n(\delta)=2\mathcal{R}_n+\sqrt{\frac{\log(1/\delta)}{2n}} \, .
\end{equation}
Therefore assumption \ref{as:rademacher_order} says that the rate of uniform convergence of true to empirical probabilities is of the order $O_{\mathbb{P}}(n^{-1/2})$. Assumption \ref{as:mvset_in_class} is in contrast very restrictive, it could be however relaxed at the price of a much more technical analysis, involving the study of the bias in empirical minimum volume set estimation under adequate assumptions on the smoothness of the boundaries of the $\Omega^*_{\alpha}$'s, like in \cite{Tsybakov97} for instance (see Remark \ref{rk:bias} below).

We state a rate of convergence for Algorithm \ref{algo:adaptive}. The following assumption shall be required.

\begin{assump_ml}[resume]
\item \label{as:llogl} The function $\mv^*$ is differentiable on $[0,1)$ with derivative $\mv^{*\prime}(\alpha)=1/Q^*(\alpha)$ for all $\alpha \in [0,1)$. Futhermore the derivative $\mv^{* \prime}$ belongs to the space $L \log L$ of Borel functions $g : (0,1) \rightarrow \mathbb{R}$ such that:
\begin{equation*}
\Vert g \Vert_{L \log L} \overset{def}{=} \int_0^1 (1 + \log \vert g(\alpha) \vert)\vert g(\alpha) \vert d\alpha < +\infty \, .
\end{equation*}
\end{assump_ml}

As explained in \citep{DeVore87}, the space $L \log L$ contains all spaces $L_p$, $p > 1$ but is strictly contained in $L_1$.
 
\begin{theorem}\label{thm:rate_bound}
Let $\delta \in (0,1)$. Suppose that assumptions \ref{as:bounded_density},\ref{as:flat_parts} and \ref{as:mvset_in_class}-\ref{as:llogl} hold true. Take $\tau = 5\phi_n(\delta)/Q^*(1-\varepsilon)$ with $\phi_n(\delta)$ as in \eqref{eq:phi_penalty}.
Then we have with probability at least $1 - \delta$:
\begin{equation*}
\forall n \geq 1, \sup_{\alpha \in [0, 1- \varepsilon]} \Bigl\{ \widehat{\mv^*}(\alpha) - \mv^*(\alpha) \Bigr\}
\leq \frac{1}{Q^*(1-\varepsilon)}\left(11\phi_n(\delta) + \frac{1}{n} \right) \, .
\end{equation*}
In addition, there exists a constant $C > 0$ such that with probability at least $1-\delta$ the number of terminal nodes $\text{\emph{card}}(\hat \sigma_{\tau})$ output by Algorithm \ref{algo:adaptive} is such that
\begin{equation*}
\text{\emph{card}}(\hat \sigma_{\tau}) \leq CQ^*(1-\varepsilon) \frac{\Vert \mv^{* \prime} \Vert_{L\log L}}{\phi_n(\delta)} \, .
\end{equation*}
\end{theorem}

The proof of this result can be found in Appendix~\ref{sec:appendix_proofs}. Its argument essentially combines the analysis of the approximation scheme described in Section~\ref{subsec:approx} and the generalization bounds for empirical minimum volume sets established in in \citep{ScottNowak06}, see Theorem 3 and Lemma 19 therein. Here, the rate obtained is not proved to be optimal in the minimax sense, lower bounds will be investigated in a future work.

\begin{corollary}
Suppose that assumptions of Theorem \ref{thm:rate_bound} are satisfied. Given assumption \ref{as:rademacher_order}, we have with probability at least $1-\delta$, $\forall n \geq 1$,
\begin{equation*}
\sup_{\alpha \in [0, 1- \varepsilon]} \Bigl\{ \widehat{\mv^*}(\alpha) - \mv^*(\alpha) \Bigr\}
\leq \frac{1}{Q^*(1-\varepsilon)}\left(11\cdot \left(\frac{2C}{\sqrt{n}} + \sqrt{\frac{\log(1/\delta)}{2n}} \right) + \frac{1}{n} \right)  \, .
\end{equation*}
\end{corollary}

To derive the rate of convergence of the $\mv$ curve of the estimated scoring function $\hat s$ towards $\mv^*$ we need the following assumption on the density $f$:

\begin{assump_ml}[resume]
\item \label{as:fast_rates} There exist constants $\gamma \geq 0$ and $C > 0$ such that for all $t > 0$ sufficiently small
\begin{equation*}
\sup_{q > 0} \lambda(\{x, \vert f(x) - q \vert \leqslant t\}) \leqslant Ct^{\gamma} \, .
\end{equation*}
\end{assump_ml}

This condition (stated with $\mathbb{P}$ instead of $\lambda$) was first introduced in \citep{Polonik95} and used in \citep{Polonik97} to derive rates of convergence of $F(\widehat \Omega_{\alpha}\Delta \Omega^*_{\alpha})$. It is also used in \citep{Rigollet09} to obtain fast rates for plug-in estimators of density level sets. The exponent $\gamma$ controls the slope of the function $f$ around any level $q > 0$. It is related to Tsybakov's margin assumption stated for the binary classification framework. The relation between Tsybakov's margin assumption and the $\gamma$-exponent assumption is given in \citep{SHS05}.

\begin{theorem}\label{thm:rate_bound_score}
Let $\delta \in (0,1)$. Suppose that assumptions \ref{as:bounded_density},\ref{as:flat_parts} and \ref{as:mvset_in_class}-\ref{as:fast_rates} hold true. Then, with probability at least $1 - \delta$,
\begin{equation*}
\sup_{\alpha \in [0, 1- \varepsilon]} \Bigl\{ \mv_{\hat s}(\alpha) - \mv^*(\alpha) \Bigr\} = O\left(n^{-\frac{\gamma}{4(1+\gamma)}}\right) \, .
\end{equation*}
\end{theorem}

Notice that the convergence rate is always slower than $n^{-1/4}$. Furthermore, because of assumption \ref{as:mvset_in_class}, the convergence rate only depends on the dimension through the parameter $\gamma$ (refer to \citep{Polonik95} for examples of this parameter). However, assumption \ref{as:mvset_in_class}, made here for simplicity's sake, is very restrictive and relaxing this assumption is discussed in Remark \ref{rk:bias}.
 
\begin{remark}\label{rk:bias}{\sc (Bias analysis)} Whereas the bias resulting from the discretization of the anomaly ranking problem discussed in Section~\ref{subsec:approx} is taken into account in the present rate bound analysis, that related to the approximation of the minimum volume sets $\Omega^*_{\alpha}$ has been neglected here for simplicity's sake (assumption \ref{as:mvset_in_class}). We point out that, under smoothness assumptions on the density sublevel sets, it is possible to control such a bias term. Indeed, consider for instance the case where the boundary $\partial \Omega_{\alpha}^*$ is of finite perimeter $per(\partial \Omega^*_{\alpha})<\infty$ (note that this is the case as soon as $f(x)$ is of bounded variation, the boundary being then $\partial \Omega^*_{\alpha}=\{x\in \X:\; f(x)=t^*_{\alpha}  \}$ by virtue of $f$'s continuity). In this case, if $\mathcal{G}=\mathcal{G}_j$ is the collection of all subsets of $\mathbb{R}^d$ obtained by binding together an arbitrary number of hypercubes of side length $2^{-j}$, cartesian products of intervals of the form $[k/2^j,\; (k+1)/2^j)$ for $k\in\mathbb{Z}$, the bias term inherent to the estimation of the minimum volume set $\Omega^*_{\alpha}$ is bounded by $per(\partial \Omega_{\alpha}^*)2^{-jd}$, up to a multiplicative constant (see Proposition 9.7 in \citep{Mallat90}). Smaller bounds can be naturally established under more restrictive assumptions involving a regularity parameter $\theta$ of $\partial \Omega^*_{\alpha}$, such as its \textit{box dimension}. Although the optimal choice for $j$ would then depend on $\theta$, a standard fashion of nearly achieving the optimal rate of convergence is to perform model selection (see Section~4 in \citep{ScottNowak06}).
\end{remark}


The general approach described above can be extended in various ways. The class $\mathcal{G}$ over which minimum volume set estimation is performed (and the penalty parameter as well) could vary depending on the mass target $\alpha$. Additionally assumption \ref{as:mvset_in_class} could be relaxed as discussed in Remark \ref{rk:bias}. In such a case, if $\Omega^{\mathcal{G}}_{\alpha}$ denotes the solution of the minimum volume set optimization problem \eqref{eq:mvpb} over the class $\mathcal{G}$, the error term $\lambda(\widehat \Omega_{\alpha}) -  \lambda(\Omega^*_{\alpha})$ could be decomposed into the sum of a stochastic error and an approximation error:
\begin{equation*}
\bigl(\lambda(\widehat \Omega_{\alpha}) - \lambda(\Omega^{\mathcal{G}}_{\alpha})\bigr) + \bigl(\lambda(\Omega^{\mathcal{G}}_{\alpha}) - \lambda(\Omega^*_{\alpha})\bigr) \, .
\end{equation*}
The choice of the class $\mathcal{G}$ should then balanced the two errors. On the one hand, if the class $\mathcal{G}$ is small, \textit{e.g.}~of finite VC dimension, then the stochastic error can be controlled thanks to assumption \ref{as:rademacher_order} but the approximation error may be very large for some distributions. On the other hand, if the class $\mathcal{G}$ is large, \textit{e.g.}~of infinite VC dimension, then the stochastic error will be very large for some distributions. A model selection approach could thus be incorporated, so as to select an adequate class $\mathcal{G}$ (refer to Section~4 in \citep{ScottNowak06}).
  
In the present analysis, we have restricted our attention to a truncated part of the $\mv$ space, corresponding to mass levels in the interval $[0,1-\varepsilon]$ and avoiding thus out-of-sample tail regions (in the case where $\supp(f)$ is of infinite Lebesgue measure). An interesting direction for further research could consist in investigating the accuracy of anomaly scoring algorithms when letting $\varepsilon=\varepsilon_n$ slowly decay to zero as $n\rightarrow +\infty$, under adequate assumptions on the asymptotic behavior of $\mv^*$. 

Finally, we emphasize that the {\sc A-Rank} algorithm is prototypal, the essential objective of its description/analysis in this paper is to provide an insight into the nature of the anomaly scoring issue, viewed here as a continuum of minimum volume set estimation problems.  Beyond the possible improvements mentioned above (\textit{e.g.} target mass level dependent classes $\mathcal{G}$), the main limitation for the practical implementation of such an approach arises from the apparent lack of minimum volume set estimation techniques that can be scaled to high-dimensional settings documented in the literature (except the dyadic recursive partitioning method considered in \citep{ScottNowak06}, see Section~6.3 therein). However, as observed in \citep{SHS05} (see also \citep{ClemRob14}), supervised methods combined with sampling procedures can be used for this purpose in moderate dimensions. Refer also to Remark \ref{rk:connect_sup}.

%% file: Sec7_Conclusion.tex

\section{Conclusion}

Motivated by a wide variety of applications including health monitoring of complex infrastructures or fraud detection for instance, we have formulated the issue of learning how to rank observations in the same order as that induced by the density function, which we called \textit{anomaly scoring} here. For this problem, much less ambitious than estimation of the local values taken by the density, a functional performance criterion, the Mass Volume curve namely, is proposed. What $\mv$ curve analysis achieves for unsupervised anomaly detection is quite akin to what ROC curve analysis accomplishes in the supervised setting.

The statistical estimation of $\mv$ curves has been investigated from an asymptotic perspective and we have provided a strategy, where the feature space is overlaid with a few well-chosen empirical minimum volume sets, to build a scoring function with statistical guarantees in terms of rate of convergence for the $\sup$ norm in the $\mv$ space.

%% file: Sec_AppendixAdaptive.tex

\appendix

\section{Proofs}
\label{sec:appendix_proofs}

\subsection{Properties of the MV Curve}

To prove Proposition \ref{prop:opt} we need the following lemma.

\begin{lemma}
\label{lem:support_z_f}
The support of the density $\supp(f) = \{x, f(x) > 0\}$ is equal to the set $\mathcal{Z}_f = \{x, F_f(f(x)) > 0 \}$ up to subsets of null $\lambda$ measure, \textit{i.e.}
\begin{equation*}
\lambda(\supp(f) \Delta \mathcal{Z}_f) = 0 \, .
\end{equation*}
\end{lemma}

\noindent The proof is deferred to Appendix~\ref{sec:appendix_technical_results}.

\begin{proof}[Proof of Proposition \ref{prop:opt}]
We start by proving that all elements of $\mathcal{S}^*$ have the same $\mv$ curve equal to $\mv_f$. One can first show that for all $s \in \mathcal{S}$ and all $\alpha \in (0,1)$, $\alpha^{-1}_s(\alpha) = F_s^{\dag}(1 - \alpha) = Q(s,\alpha)$. Let $s^* \in \mathcal{S}^*$ and let $\alpha \in (0, 1)$,
\begin{equation*}
\begin{aligned}
\mv_{s^*}(\alpha) &= \lambda\{x, s^*(x) \geq F^{\dag}_{s^*}(1-\alpha)\} = \lambda\{x, F_{s^*}(s^*(x)) \geq 1-\alpha\} \\
&= \lambda\{x \in \supp(f), F_{s^*}(s^*(x)) \geq 1-\alpha\} \\
&\phantom{=}+ \lambda\{x \in \overline{\supp(f)}, F_{s^*}(s^*(x)) \geq 1-\alpha\}
\end{aligned}
\end{equation*}
where the second equality holds because for any $x$ and any $\alpha \in (0, 1)$, $s^*(x) \geq F^{\dag}_{s^*}(1-\alpha)$ if and only if $F_{s^*}(s^*(x)) \geq 1-\alpha$ (see \textit{e.g.} Property 2.3 in \citep{Embrechts2013}). Thanks to assertion $(iv)$ in the definition of $\mathcal{S}^*$, for $x \in \overline{\supp(f)}$, $F_{s^*}(s^*(x)) = 0$ and as $1- \alpha > 0$, the term of the last line of the previous equation is equal to $0$. Using the notations introduced in the definition of $\mathcal{S}^*$, the term $\lambda\{x \in \supp(f), F_{s^*}(s^*(x)) \geq 1-\alpha\}$ can be decomposed as follows:
\begin{equation*}
\lambda\{x \in \supp(f) \cap \overline{\mathcal{Z}}, F_{s^*}(s^*(x)) \geq 1-\alpha\} + \lambda\{x \in \supp(f) \cap \mathcal{Z}, F_{s^*}(s^*(x)) \geq 1-\alpha\} \, .
\end{equation*}
The first term is lower than $\lambda(\supp(f) \cap \overline{\mathcal{Z}}) = 0$ (thanks to assertion $(i)$ in the definition of $\mathcal{S}^*$) and is therefore equal to $0$. We deal with the second term as follows. Let $x \in \supp(f) \cap \mathcal{Z}$,
\begin{equation*}
\begin{aligned}
F_{s^*}(s^*(x)) &= \mathbb{P}(s^*(X) \leq s^*(x)) \\
&= \mathbb{P}(X \in \supp(f), s^*(X) \leq s^*(x)) + \mathbb{P}(X \in \overline{\supp(f)}, s^*(X) \leq s^*(x)) \, .
\end{aligned}
\end{equation*}
The second term is lower than $\mathbb{P}(X \in \overline{\supp(f)}) = 0$ and is therefore equal to 0. Thanks again to assertion $(i)$, the first term is equal to
\begin{equation*}
\mathbb{P}(X \in \supp(f) \cap \mathcal{Z}, s^*(X) \leq s^*(x))
= \mathbb{P}(X \in \supp(f) \cap \mathcal{Z}, f(X) \leq f(x))
\end{equation*}
where we also used assertions $(ii)$ and $(iii)$. As done above for $s^*$, this last term can be shown to be equal to $F_f(f(x))$. Therefore,
\begin{equation*}
\lambda\{x \in \supp(f) \cap \mathcal{Z}, F_{s^*}(s^*(x)) \geq 1-\alpha\} = \lambda\{x \in \supp(f) \cap \mathcal{Z}, F_f(f(x)) \geq 1-\alpha\}
\end{equation*}
and as done above for $s^*$, this can be shown to be equal to $\mv_f(\alpha)$. \newline

We now show \eqref{eq:bound1} and $\implies$ in \eqref{eq:opt}. Let $s \in \mathcal{S}$, $\alpha \in (0,1)$. We have,
\begin{equation*}
\mv_s(\alpha) - \mv^*(\alpha) = \lambda(\Omega_{s, Q(s, \alpha)}) - \lambda(\Omega^*_{\alpha}) \leq \lambda(\Omega^*_{\alpha} \Delta \Omega_{s, Q(s, \alpha)}) \, .
\end{equation*}
Now let $s^* \in \mathcal{S^*}$. As shown above, we have $\mv_{s^*} = \mv_f = \mv^*$. Furthermore, $F(\Omega_{s, Q(s, \alpha)})= \mathbb{P}(s(X) \geq F_s^{\dag}(1-\alpha)) = 1 - \mathbb{P}(s(X) < F_s^{\dag}(1-\alpha))$ and one can show that $\mathbb{P}(s(X) < F_s^{\dag}(1-\alpha)) \leq 1 - \alpha$. We therefore have $\mathbb{P}(\Omega_{s, Q(s, \alpha)}) \geq \alpha$. As $\Omega^*_{\alpha}$ minimizes $\lambda$ over all sets with mass at least $\alpha$, we have $\mv_s(\alpha) - \mv^*(\alpha) \geq 0$. \newline

Finally we prove $\impliedby$ in \eqref{eq:opt}. On the one hand, for $s = f$ we have $\mv_{s^*} \leq \mv_f$. On the other hand, as $f \in \mathcal{S}^*$, thanks to $\implies$ in \eqref{eq:opt} (which has been proven above), for all scoring function $s$, $\mv_f \leq \mv_s$. Therefore $\mv_f \leq \mv_{s^*}$. We thus have $\mv_f = \mv_{s^*}$ on $(0,1)$. From now on, we replace $s^*$ by $s$ for the sake of clarity. We thus want to prove that $\mv_f = \mv_{s}$ on $(0,1)$ implies $s \in \mathcal{S}^*$. The proof is decomposed in three steps.

\paragraph{First step.} We first show that for $\lambda$-almost all $x \in \mathcal{X}$,
\begin{equation*}
\mathbb{I}\{x \in \Omega^*_{\alpha} \} = \mathbb{I}\{x \in \Omega_{s, Q(s,\alpha)} \} \text{ for all } \alpha \in (0, 1) \,
\end{equation*}
and
\begin{equation*}
F_f(f(x)) = F_s(s(x)) \, .
\end{equation*}

Let $\alpha \in (0, 1)$ be fixed. We have $\lambda(\Omega^*_{\alpha}) = \lambda(\Omega_{s, Q(s,\alpha)})$. Since $\mathbb{P}(X \in \Omega_{s, Q(s,\alpha)}) \geq \alpha$, by uniqueness of the solution of the minimum volume set optimization problem up to subsets of null $\lambda$ measure, we have $\lambda(\Omega^*_{\alpha} \Delta \Omega_{s, Q(s,\alpha)}) = 0$ which is equivalent to: for $\lambda$-almost all $x \in \mathcal{X}$
\begin{equation*}
\mathbb{I}\{x \in \Omega^*_{\alpha} \} = \mathbb{I}\{x \in \Omega_{s, Q(s,\alpha)} \}
\end{equation*}
\textit{i.e.}~there exists $\mathcal{Z}_{\alpha} \subset \mathcal{X}$ such that $\lambda(\overline{\mathcal{Z}_{\alpha}}) = 0$ and for all $x \in \mathcal{Z}_{\alpha}$, $\mathbb{I}\{x \in \Omega^*_{\alpha} \} = \mathbb{I}\{x \in \Omega_{s, Q(s,\alpha)} \}$. Note that $\overline{\mathcal{Z}_{\alpha}}$ depends on $\alpha$ and we cannot consider to take the union over all $\alpha \in (0, 1)$ of the $\overline{\mathcal{Z}_{\alpha}}$ because this union would not necessarily be of null $\lambda$ measure (as we have uncountably many $\alpha$). However considering the union over the rationals of $(0,1)$ is sufficient. Indeed, let $\mathcal{Z}_{\mathbb{Q}}$ denote such a union:
\begin{equation*}
\overline{\mathcal{Z}_{\mathbb{Q}}} = \bigcup_{\alpha \in (0,1) \cap \mathbb{Q}} \overline{\mathcal{Z}_{\alpha}} \, .
\end{equation*}
As it is a union over a countable set of $\alpha$, $\lambda(\overline{\mathcal{Z}_{\mathbb{Q}}}) = 0$ and for all $x \in \mathcal{Z}_{\mathbb{Q}} = \cap_{\alpha \in (0,1) \cap \mathbb{Q}} \mathcal{Z}_{\alpha}$ and all $\alpha \in (0,1) \cap \mathbb{Q}$,
\begin{equation*}
\mathbb{I}\{x \in \Omega^*_{\alpha} \} = \mathbb{I}\{x \in \Omega_{s, Q(s,\alpha)} \} \, .
\end{equation*}

Let $x \in \mathcal{Z}_{\mathbb{Q}}$ and let $\alpha \in (0,1) \cap \mathbb{R} \setminus \mathbb{Q}$. We know that there exists a decreasing sequence $(\alpha_m)_{m \geq 0}$ of elements of $(0,1) \cap \mathbb{Q}$ such that $\alpha_m$ converges towards $\alpha$ when $m$ tends to infinity. Now, thanks to a property of the quantile function (see e.g.~assertion (5) of Property 2.3 in (Embrechts, 2013)), we have, for all $\alpha' \in (0,1)$,
\begin{equation*}
\Omega^*_{\alpha'} = \{x', f(x') \geq F_f^{\dag}(1-\alpha') \} = \{x', F_f(f(x')) \geq 1-\alpha' \} \, .
\end{equation*}
Thus the sequence $(\Omega^*_{\alpha_m})_{m \geq 0}$ is decreasing (with respect to set inclusion) and
\begin{equation*}
\lim_{m \rightarrow +\infty} \mathbb{I}\{x \in \Omega^*_{\alpha_m} \} = \mathbb{I}\{x \in \bigcap_{m \geq 0} \Omega^*_{\alpha_m} \} = \mathbb{I}\{x \in \Omega^*_{\alpha} \}
\end{equation*}
as one can show that
\begin{equation*}
\bigcap_{m \geq 0} \Omega^*_{\alpha_m} = \Omega^*_{\alpha} \, .
\end{equation*}
Similarly,
\begin{equation*}
\lim_{m \rightarrow +\infty} \mathbb{I}\{x \in \Omega_{Q(s,\alpha_m)} \} = \mathbb{I}\{x \in \Omega_{Q(s,\alpha)} \} \, .
\end{equation*}
For all $m \geq 0$, $\mathbb{I}\{x \in \Omega^*_{\alpha_m} \} = \mathbb{I}\{x \in \Omega_{Q(s,\alpha_m)} \}$ (because $\alpha_m \in (0,1) \cap \mathbb{Q}$) and thus the two limits are equal. Hence the first result.

Now this also implies that for $x \in \mathcal{Z}_{\mathbb{Q}}$,
\begin{equation*}
\int_0^1 \mathbb{I}\{x \in \Omega^*_{\alpha} \} \mathrm{d}\alpha = \int_0^1 \mathbb{I}\{x \in \Omega_{s, Q(s,\alpha)} \} \mathrm{d}\alpha \, .
\end{equation*}
If $U$ denotes a random variable with uniform distribution on $(0,1)$, the left-hand side can be rewritten as
\begin{equation*}
\mathbb{E}_U[\mathbb{I}\{f(x) \geq F_f^{\dag}(1-U) \}] = \mathbb{E}_X[\mathbb{I}\{f(x) \geq f(X) \}] = F_f(f(x)) \, .
\end{equation*}
Similarly, the right-hand side is equal to $F_s(s(x))$ and therefore for all $x \in \mathcal{Z}_{\mathbb{Q}}$, $F_f(f(x)) = F_s(s(x))$.

\paragraph{Second step.} We can now prove assertions $(i)$, $(ii)$, $(iii)$ and $(iv)$. We will first show that for almost all $x \in \supp(f)$, there exists $\alpha' \in (0,1]$ such that $f(x) = F_f^{\dag}(\alpha')$.

One may first observe that $\lambda(\overline{\mathcal{Z}_f} \cap \supp(f)) = 0$ where $\mathcal{Z}_f = \{x, F_f(f(x)) > 0\}$. Indeed, $\mathcal{Z}_f \subset \supp(f)$ (if $f(x) = 0$ then $F_f(f(x)) = 0$) and therefore $\lambda(\overline{\mathcal{Z}_f} \cap \supp(f)) = \lambda(\supp(f) \setminus \mathcal{Z}_f)$ which is equal to 0 since, thanks to Lemma \ref{lem:support_z_f}, $\lambda(\supp(f) \Delta \mathcal{Z}_f) = 0$.

Let $x \in \supp(f) \cap \mathcal{Z}_f$. If $F_f^{\dag}$ is continuous at $\alpha' = F_f(f(x)) \in (0, 1]$ then $F_f^{\dag}(\alpha') = F_f^{\dag}(F_f(f(x))) = f(x)$. Now, since $F_f^{\dag}$ is increasing, $F_f^{\dag}$ has at most countably many discontinuities and each discontinuity $m$ corresponds to a jump of $F_f^{\dag}$ between two values $t^m_1$ and $t^m_2$. We also know that this jump of $F_f^{\dag}$ corresponds to a flat part of $F_f$ between $t^m_1$ and $t^m_2$. However if $F_f$ has a flat part between $t^m_1$ and $t^m_2$ then $\mathbb{P}(t^m_1 \leq f(X) \leq t^m_2) = F_f(t^m_2) - F_f(t^m_1) = 0$ (recall that $F_f$ is continuous as $f$ has no flat parts). This gives
\begin{equation*}
\int_{\{u, t^m_1 \leq f(u) \leq t^m_2 \}} f(u) \mathrm{d}u = 0 \, .
\end{equation*}
Thus $f = 0$ $\lambda$-almost everywhere on $\{u, t^m_1 \leq f(u) \leq t^m_2 \}$. This implies that $\lambda(\{x, t^m_1 \leq f(x) \leq t^m_2 \} \cap \supp(f)) = 0$. Let $\overline{\mathcal{Z}_0}$ be the union of such sets over $m \geq 0$:
\begin{equation*}
\overline{\mathcal{Z}_0} = \bigcup_{m \geq 0} \{x, t^m_1 \leq f(x) \leq t^m_2 \}
\end{equation*}
and let $\mathcal{Z}_1 = \mathcal{Z}_0 \cap \mathcal{Z}_f$. We have $\lambda(\overline{\mathcal{Z}_1} \cap \supp(f)) = 0$ and for all $x \in \mathcal{Z}_1 \cap \supp(f)$, there exists $\alpha' \in (0,1]$ such that $f(x) = F_f^{\dag}(\alpha')$ because $F_f^{\dag}$ is continuous at $\alpha' = F_f(f(x))$ (otherwise this would imply that $x$ belongs to a region corresponding to a flat part of $F_f$, \textit{i.e.}~one of the sets $\{x, t_1^m \leq f(x) \leq t_2^m\}$).

Similarly, we can show that there exists a set $\mathcal{Z}_2$ such that $\lambda(\overline{\mathcal{Z}_2} \cap \supp(f)) = 0$ and for all $x \in \mathcal{Z}_2 \cap \supp(f)$, there exists $\alpha' \in (0,1]$ such that $s(x) = F_s^{\dag}(\alpha')$. The only difference with $f$ is that we here have to consider $\mathcal{Z}_{\mathbb{Q}}$ so that $F_s(s(x)) = F_f(f(x))$.

Let $\mathcal{Z} = \mathcal{Z}_{\mathbb{Q}} \cap \mathcal{Z}_1 \cap \mathcal{Z}_2$. We have $\lambda(\overline{\mathcal{Z}} \cap \supp(f)) = 0$. Let $x_1, x_2 \in \mathcal{Z} \cap \supp(f)$ such that $f(x_1) > f(x_2)$ and let $\Omega_1 = \{x \in \mathcal{Z}, f(x) \geq f(x_1)\}$. There exists $\alpha_1 \in [0, 1)$ such that $F_f^{\dag}(1-\alpha_1) = f(x_1)$. Let's first consider the case where $\alpha_1 \in (0, 1)$. We have
\begin{equation}
\label{eq:set_equality}
\Omega_1 = \{x \in \mathcal{Z}, f(x) \geq F_f^{\dag}(1-\alpha_1)\} = \{x \in \mathcal{Z}, s(x) \geq F_s^{\dag}(1-\alpha_1)\} \, .
\end{equation}
As $x_1 \in \Omega_1$, $x_1 \in \{x \in \mathcal{Z}, s(x) \geq F_s^{\dag}(1-\alpha_1)\}$ and $s(x_1) \geq F_s^{\dag}(1-\alpha_1)$. Analogously, as $x_2 \notin \Omega_1$, $x_2 \notin \{x \in \mathcal{Z}, s(x) \geq F_s^{\dag}(1-\alpha_1)\}$ and $s(x_2) < F_s^{\dag}(1-\alpha_1)$. Thus $s(x_1) > s(x_2)$.
Let's now consider the case $\alpha_1 = 0$. We do not know that \eqref{eq:set_equality} still holds. However we can always assume without loss of generality that $s$ is bounded. Indeed, if $s$ is not bounded, consider the function $s' = \arctan \circ s$. As $\arctan$ is strictly increasing, thanks to Proposition 3, $\mv_s = \mv_{s'}$. Besides, as $s$ takes its values in $\mathbb{R}^+$, $s'$ takes its values in $[0 , \pi/2)$ and is bounded. Finally, if we show that $s' = T \circ f$ with $T$ strictly increasing then we also have that $s$ is a strictly increasing transform of $f$. Then as $s$ is bounded, $F_s^{\dag}(1-\alpha_1)$ is defined at $\alpha_1 = 0$ and we can replace $\mv_s = \mv^*$ on $(0,1)$ by $\mv_s = \mv^*$ on $[0, 1)$ in the proposition. We can then show that $\eqref{eq:set_equality}$ holds for $\alpha_1=0$.

Similarly, $s(x_1) > s(x_2)$ implies $f(x_1) > f(x_2)$. Therefore for all $x_1, x_2 \in \mathcal{Z} \cap \supp(f)$,
\begin{equation*}
s(x_1) > s(x_2) \iff f(x_1) > f(x_2) \, .
\end{equation*}
and hence the existence of $T$. One can for instance take $T: f(x) \in f(\mathcal{Z} \cap \supp(f)) \mapsto s(x) \in s(\mathcal{Z} \cap \supp(f))$.

\paragraph{Third step.} Finally, assertion $(iv)$ derives from the fact that for $x \in \mathcal{Z}_{\mathbb{Q}} \cap \overline{\supp(f)}$ we have $F_s(s(x)) = F_f(f(x)) = 0$.
\end{proof}

\begin{proof}[Proof of Proposition \ref{prop:convex_derivative}]
To prove the convexity of $\mv^*$ we show that the slopes are increasing. Let $\alpha_1 < \alpha_2 < \alpha_3$ be in $[0, 1)$,
\begin{equation*}
\begin{aligned}
\mv^*(\alpha_2) - \mv^*(\alpha_1) &= \lambda(\Omega^*_{\alpha_2}) - \lambda(\Omega^*_{\alpha_1}) \\
&= \lambda\{ f \geq F_f^{\dag}(1-\alpha_2)\} - \lambda\{ f \geq F_f^{\dag}(1-\alpha_1)\} \\
&= \lambda\{ x, F_f(f(x)) \geq 1-\alpha_2\} - \lambda\{ x, F_f(f(x)) \geq 1-\alpha_1\} \\
&= \lambda\{x,  1-\alpha_2 \leq F_f(f(x)) < 1-\alpha_1 \} \\
&= \int_{\mathcal{X}} \mathbb{I}{\{x, 1-\alpha_2 \leq F_f(f(x)) < 1-\alpha_1\}}dx \\
&= \int_{\mathcal{X}} \mathbb{I}{\{x, 1-\alpha_2 \leq F_f(f(x)) < 1-\alpha_1\}}\frac{1}{f(x)}f(x)dx
\end{aligned}
\end{equation*}
where the third equality holds because for any $x$ and any $\alpha \in [0, 1)$, $f(x) \geq F^{\dag}_{f}(1-\alpha)$ if and only if $F_f(f(x)) \geq 1-\alpha$ (see \textit{e.g.} Property 2.3 in \citep{Embrechts2013}). Now as $x \in \{x, F_f^{\dag}(1-\alpha_2) \leq f(x) < F_f^{\dag}(1-\alpha_1)\}$, we obtain
\begin{equation*}
\begin{aligned}
\mv^*(\alpha_2) - \mv^*(\alpha_1) &\leq \frac{1}{F_f^{\dag}(1-\alpha_2)} \mathbb{P}(F_f(f(X)) \in [1-\alpha_2, 1-\alpha_1)) \\
& = \frac{\alpha_2 - \alpha_1}{F_f^{\dag}(1-\alpha_2)}
\end{aligned}
\end{equation*}
because the distribution of $F_f(f(X))$ is the uniform distribution on $(0,1)$ (as $F_f$ is continuous since $f$ has no flat parts). Similarly, one can show that $\mv^*(\alpha_3) - \mv^*(\alpha_2) \geq (\alpha_3 - \alpha_2)/F_f^{\dag}(1-\alpha_2)$.
Hence, we finally have
\begin{equation*}
\frac{\mv^*(\alpha_2) - \mv^*(\alpha_1)}{\alpha_2 - \alpha_1} \leq \frac{1}{F_f^{\dag}(1-\alpha_2)} \leq \frac{\mv^*(\alpha_3) - \mv^*(\alpha_2)}{\alpha_3 - \alpha_2} \, .
\end{equation*}

We now prove the differentiability of $\mv^*$ and the formula of its derivative. First as $F_f$ is invertible on $(a, F_f^{\dag}(1)]$, $F_f^{\dag}$ is the ordinary inverse of $F_f$ on $(0, 1]$. Let $\alpha \in (0, 1)$ and $h > 0$. Proceeding similarly as above we can show that
\begin{equation*}
\frac{1}{F_f^{\dag}(1-\alpha)}\leq \frac{\mv^*(\alpha + h) - \mv^*(\alpha)}{h} \leq \frac{1}{F_f^{\dag}(1-(\alpha + h))} \, .
\end{equation*}
Therefore as $F_f^{\dag}$ is continuous (because $F_f^{\dag}$ is the ordinary inverse of $F_f$),
\begin{equation*}
\lim_{h \rightarrow 0^+} \frac{\mv^*(\alpha + h) - \mv^*(\alpha)}{h} = \frac{1}{F_f^{\dag}(1-\alpha)} \, .
\end{equation*}
Analogously, we also have
\begin{equation*}
\lim_{h \rightarrow 0^+} \frac{\mv^*(\alpha - h) - \mv^*(\alpha)}{-h} = \frac{1}{F_f^{\dag}(1-\alpha)}
\end{equation*}
which implies that $\mv'^{*}(\alpha) = 1/F_f^{\dag}(1-\alpha)$.
For $\alpha=0$, we have $\mv^*(0)=0$ and proceeding as for $\alpha \in (0,1)$, if $h > 0$,
\begin{equation*}
\mv^*(h) = \int_{\X}\mathbb{I}\{x, F_f(f(x)) \geq 1-h\}\frac{1}{f(x)}f(x)dx
\end{equation*}
and
\begin{equation*}
\frac{h}{F_f^{\dag}(1)}\leq \mv^*(h) \leq \frac{h}{F_f^{\dag}(1-\alpha)} \, .
\end{equation*}
Therefore $\lim_{h \rightarrow 0^+} \mv^*(h)/h = 1/F_f^{\dag}(1)$.
Thus, as $Q^*(\alpha) = F_f^{\dag}(1-\alpha)$ for all $\alpha \in [0, 1)$ we obtain the formula of the derivative of $\mv^*$.
\end{proof}

\subsection{Statistical Estimation of the MV curve}

\subsubsection{Strong approximation: proof of Theorem \ref{thm:estimation}}

To prove Theorem \ref{thm:estimation} we need the following lemma.
\begin{lemma}
\label{lem:dkw_taylor_remainder}
Under assumptions \ref{as:bounded_score}-\ref{as:density_regularity_positive} and \ref{as:limit_alpha_1} there exists a constant $C > 0$ such that we almost surely have for $n$ large enough,
\begin{equation*}
\sup_{\alpha' \in [\varepsilon, 1]}\vert \widehat F^{\dag}_s(\alpha') - F^{\dag}_s(\alpha') \vert^2 \leq C \frac{\log n}{n} \, .
\end{equation*}
\end{lemma}

\begin{proof}[Proof of Lemma \ref{lem:dkw_taylor_remainder}]
Let $U_i = F_s(s(X_i))$ for all $i \in \{1,\dots,n\}$. As $F_s$ is continuous, the random variables $U_1, \dots, U_n$ are \textit{i.i.d.} with uniform distribution on $(0,1)$. Their empirical cumulative distribution function is given by $\widehat U(\alpha) = 1/n\sum_{i=1}^n\mathbb{I}\{U_i \leq \alpha\}$ for all $\alpha \in [0,1]$. Furthermore we have for all $\alpha \in (0,1]$, $F_s(\widehat F_s^{\dag}(\alpha)) = \widehat U^{\dag}(\alpha)$. We also define $\widehat U^{\dag}(0) = 0$. Let $U_n$ denote the uniform empirical process defined for all $\alpha \in [0,1]$ by $U_n(\alpha) = \sqrt{n}(\widehat U(\alpha) - \alpha)$ and $u_n$ denote the uniform empirical quantile process defined for all $\alpha \in [0,1]$ by $u_n(\alpha)=\sqrt{n}(\widehat U^{\dag}(\alpha) - \alpha)$.
Applying the Dvoretzky-Kiefer-Wolfowitz (DKW) inequality \citep{Massart1990} gives for all $z > 0$ and all $n \geq 1$,
\begin{equation}
\label{eq:DKW_uniform}
\mathbb{P}(\sup_{\alpha \in [0,1]} \vert U_n(\alpha) \vert \geq z) \leq 2\exp(-2z^2) \, .
\end{equation}

As $\sup_{0 \leq \alpha \leq 1}\vert U_n(\alpha) \vert = \sup_{0 \leq \alpha \leq 1}\vert u_n(\alpha) \vert$ for all $\omega$ in the sample space (see equation (1.4.5) in \citep{Csorgo1983}), the DKW inequality \eqref{eq:DKW_uniform} holds when replacing $U_n$ by $u_n$. Taking $z=\sqrt{\log n}$, we obtain for all $n \geq 1$,
\begin{equation}
\label{eq:dkw_quantile_uniform_squared}
\mathbb{P}\Bigl(\sup_{\alpha \in [0,1]} \bigl\vert\widehat U^{\dag}(\alpha) - \alpha \bigr\vert^2 \geq \frac{\log n}{n}\Bigr) \leq \frac{2}{n^2} \, .
\end{equation}
As $\sum_{n=1}^{\infty} 1/n^2 < \infty$, by the Borel-Cantelli lemma, we almost surely have for $n$ large enough,
\begin{equation}
\label{eq:emp_quantile_as_bound}
\sup_{\alpha \in [0,1]} \bigl\vert\widehat U^{\dag}(\alpha) - \alpha \bigr\vert^2 \leq \frac{\log n}{n} \, .
\end{equation}
Using the mean value theorem we can write for $\alpha' \in [\varepsilon, 1]$,
\begin{equation*}
\begin{aligned}
\vert \widehat F^{\dag}_s(\alpha') - F^{\dag}_s(\alpha') \vert &=\vert F^{\dag}_s(\widehat U^{\dag}(\alpha'))) - F^{\dag}_s(\alpha') \vert \\
&= \left\vert \frac{1}{f_s(F_s^{\dag}(\xi))}\right\vert \cdot \vert \widehat U^{\dag}(\alpha') - \alpha' \vert
\end{aligned}
\end{equation*}
where $\xi$ is between $\alpha'$ and $\widehat U^{\dag}(\alpha')$. As $\widehat U^{\dag}$ is increasing we have for all $\alpha' \in [\varepsilon,1]$, $\widehat U^{\dag}(\varepsilon) \leq \widehat U^{\dag}(\alpha') \leq 1$. From \eqref{eq:emp_quantile_as_bound} we know that there exists a constant $c_1 > 0$ such that we almost surely have for $n$ large enough, $\widehat U^{\dag}(\varepsilon) \geq \varepsilon - c_1 > 0$. Therefore, almost surely and for $n$ large enough,
\begin{equation*}
f_s(F_s^{\dag}(\xi)) \geq \inf_{y' \in [\varepsilon - c_1, 1]} f_s(F_s^{\dag}(y')) = \inf_{y \in [0, 1-\varepsilon + c_1]} f_s(\alpha_s^{-1}(y)) > 0 \, .
\end{equation*}
Note that the last infimum is strictly positive because as $f_s \circ \alpha_s^{-1}$ is continuous, the infimum is reached at a point $y_0 \in [0, 1-\varepsilon + c_1] \subset [0,1)$ and as $f_s \circ \alpha_s^{-1} > 0$ on $[0, 1)$ (assumptions \ref{as:density_regularity_positive} and \ref{as:limit_alpha_1}), $f_s(\alpha_s^{-1}(y_0)) > 0$.

This eventually leads to: almost surely and for $n$ large enough,
\begin{equation*}
\begin{aligned}
\sup_{\alpha' \in [\varepsilon, 1]}\vert \widehat F^{\dag}_s(\alpha') - F^{\dag}_s(\alpha') \vert^2 &\leq \frac{1}{\left(\inf_{[0, 1-\varepsilon + c_1]} f_s \circ \alpha_s^{-1}\right)^2}\frac{\log n}{n} \, .
\end{aligned}
\end{equation*}
\end{proof}

\noindent \textbf{Proof of assertion $(i)$ of Theorem \ref{thm:estimation}.}

\begin{proof}
Assume that assumptions \ref{as:bounded_score}-\ref{as:density_regularity_positive} and \ref{as:limit_alpha_1}-\ref{as:regulartiy_lambda} are fulfilled. Using the mean value theorem we can write for all $\alpha \in [0,1-\varepsilon]$
\begin{equation*}
\widehat\mv_s(\alpha) - \mv_s(\alpha) = \lambda_s(\widehat \alpha_s^{-1}(\alpha)) - \lambda_s(\alpha_s^{-1}(\alpha)) = \lambda'_s(\xi)(\widehat \alpha_s^{-1}(\alpha) - \alpha_s^{-1}(\alpha))
\end{equation*}
where $\xi$ is between $\widehat\alpha_s^{-1}(\alpha)$ and $\alpha_s^{-1}(\alpha)$. Therefore,
\begin{equation*}
\sup_{\alpha \in [0,1-\varepsilon]} \vert \widehat\mv_s(\alpha) - \mv_s(\alpha) \vert = \sup_{\alpha \in [0,1-\varepsilon]} \vert \lambda'_s(\xi) \vert \cdot \vert \widehat \alpha_s^{-1}(\alpha) - \alpha_s^{-1}(\alpha) \vert \, .
\end{equation*}

Now for all $\alpha \in [0, 1-\varepsilon]$, $\alpha_s^{-1}(\alpha)$ and $\widehat \alpha_s^{-1}(\alpha)$ belong to the support of $f_s$, \textit{i.e.}, $[0, \Vert s \Vert_{\infty}]$. This is true for $\widehat \alpha_s^{-1}(\alpha)$ because it is equal to one of the $s(X_i)$, $1\leq i \leq n$. Thus $\xi \in [0, \Vert s \Vert_{\infty}]$. As $\lambda_s'$ is continuous, we have
\begin{equation*}
\sup_{\alpha \in [0,1-\varepsilon]} \vert \lambda_s'(\xi)\vert \leq \sup_{[0,\Vert s\Vert_{\infty}]} \vert \lambda_s' \vert < +\infty
\end{equation*}
and
\begin{equation*}
\begin{aligned}
\sup_{\alpha \in [0,1-\varepsilon]} \vert \widehat\mv_s(\alpha) - \mv_s(\alpha) \vert &\leq \sup_{[0,\Vert s\Vert_{\infty}]} \vert \lambda_s' \vert \cdot \sup_{\alpha \in [0,1-\varepsilon]} \vert \widehat \alpha_s^{-1}(\alpha) - \alpha_s^{-1}(\alpha) \vert \\
&= \sup_{[0,\Vert s\Vert_{\infty}]} \vert \lambda_s' \vert \cdot \sup_{\alpha' \in [\varepsilon,1]} \vert \widehat F_s^{\dag}(\alpha') - F_s^{\dag}(\alpha') \vert \, .
\end{aligned}
\end{equation*}

From Lemma \ref{lem:dkw_taylor_remainder} we know that there exists a constant $C > 0$ such that we almost surely have for $n$ large enough,
\begin{equation*}
\sup_{\alpha' \in [\varepsilon, 1]} \vert \widehat F_s^{\dag}(\alpha') - F_s^{\dag}(\alpha')\vert \leq \sqrt{C \frac{\log n}{n}} \, .
\end{equation*}
Therefore $\sup_{\alpha' \in [\varepsilon, 1]} \vert \widehat F_s^{\dag}(\alpha') - F_s^{\dag}(\alpha')\vert$ converges almost surely to $0$ and so is $\sup_{\alpha \in [0,1-\varepsilon]} \vert \widehat\mv_s(\alpha) - \mv_s(\alpha) \vert$.
\end{proof}

\noindent \textbf{Proof of assertion $(ii)$ of Theorem \ref{thm:estimation}.}

\begin{proof}[Proof] For all $n \geq 1$, let $v_n = n^{-1/2}\log n$. Assume that assumptions \ref{as:bounded_score}-\ref{as:regulartiy_lambda} are fulfilled. By virtue of Theorem 3.1.2 in \citep{Csorgo1983}, there exists a sequence of Brownian bridges $\{B_n^1(\alpha), \alpha \in [0,1]\}_{n \geq 1}$ such that, we almost surely have:
\begin{equation}
\label{eq:strong_approx_uniform_quantile_process}
\sup_{0 \leq \alpha^{\prime} \leq 1} \vert u_n(\alpha^{\prime}) - B_n^1(\alpha^{\prime}) \vert = O(v_n)
\end{equation}
where $u_n$ is the uniform quantile process as defined in the proof of Lemma \ref{lem:dkw_taylor_remainder}. From (3.3) of Theorem 3 in \citep{CR78}, we also have almost surely, for $n$ large enough,
\begin{equation*}
\sup_{a_n \leq \alpha^{\prime} \leq 1 - a_n} \vert f_s\bigl(F_s^{\dag}(\alpha^{\prime})\bigr)\sqrt{n}\bigl(\widehat{F}_s^{\dag}(\alpha^{\prime}) - F_s^{\dag}(\alpha^{\prime})\bigr) - u_n(\alpha^{\prime}) \vert \leq C_1n^{-1/2}\log\log n
\end{equation*}
where $C_1=40c10^c$ and $a_n = 25n^{-1} \log \log n$. Hence for $n$ large enough we have $a_n \leq \varepsilon$ and therefore we almost surely have
\begin{equation}
\label{eq:diff_quantile_a_n_epsilon}
\sup_{\varepsilon \leq \alpha^{\prime} \leq 1 - a_n} \vert f_s\bigl(F_s^{\dag}(\alpha^{\prime})\bigr)\sqrt{n}\bigl(\widehat{F}_s^{\dag}(\alpha^{\prime}) - F_s^{\dag}(\alpha^{\prime})\bigr) - u_n(\alpha^{\prime}) \vert = O(n^{-1/2}\log\log n)
\end{equation}
Now from the proof of theorem 3.2.1 in \citep{Csorgo1983}, we also almost surely have,
\begin{equation}
\label{eq:diff_quantile_a_n_1}
\sup_{1 - a_n \leq \alpha^{\prime} \leq 1} \vert f_s\bigl(F_s^{\dag}(\alpha^{\prime})\bigr)\sqrt{n}\bigl(\widehat{F}_s^{\dag}(\alpha^{\prime}) - F_s^{\dag}(\alpha^{\prime})\bigr) - u_n(\alpha^{\prime}) \vert = O(n^{-1/2}\log\log n) \, .
\end{equation}
Combining \eqref{eq:strong_approx_uniform_quantile_process}, \eqref{eq:diff_quantile_a_n_epsilon} and \eqref{eq:diff_quantile_a_n_1}, we almost surely have:
\begin{equation*}
\sup_{\varepsilon \leq \alpha^{\prime} \leq 1}\vert f_s\bigl(F_s^{\dag}(\alpha^{\prime})\bigr)\sqrt{n}\bigl(\widehat{F}_s^{\dag}(\alpha^{\prime}) - F_s^{\dag}(\alpha^{\prime})\bigr) - B_n^1(\alpha^{\prime}) \vert = O(v_n) \, .
\end{equation*}

As one can show that $F_s^{\dag}(\alpha) = \alpha_s^{-1}(1-\alpha)$ and $\widehat{F}_s^{\dag}(\alpha) = \widehat{\alpha}_s^{-1}(1-\alpha)$ for all $\alpha \in (0,1]$ we almost surely have
\begin{equation*}
\sup_{\varepsilon \leq \alpha^{\prime} \leq 1}\vert f_s\bigl(\alpha_s^{-1}(1-\alpha^{\prime})\bigr)\sqrt{n}\bigl(\widehat{\alpha}_s^{-1}(1-\alpha^{\prime}) - \alpha_s^{-1}(1-\alpha^{\prime})\bigr) - B_n^1(\alpha^{\prime}) \vert = O(v_n) \, .
\end{equation*}
With the change of variable $\alpha = 1-\alpha'$ this leads to
\begin{equation*}
\sup_{0 \leq \alpha \leq 1-\varepsilon}\vert f_s\bigl(\alpha_s^{-1}(\alpha)\bigr)\sqrt{n}\bigl(\widehat{\alpha}_s^{-1}(\alpha) - \alpha_s^{-1}(\alpha)\bigr) - B_n(\alpha) \vert = O(v_n)
\end{equation*}
where $\{B_n(\alpha), \alpha \in [0,1]\}_{n \geq 1}$ is the sequence of stochastic processes defined by $\{B_n(\alpha), \alpha \in [0,1]\}_{n \geq 1} = \{B_n^1(1-\alpha), \alpha \in [0,1]\}_{n \geq 1}$. Notice that for all $n \geq 1$, $B_n$ is a Brownian Bridge as it has the same distribution as $B_n^1$. Thus there exists a constant $C_2$ independent of $\alpha$ such that almost surely for $n$ large enough and for all $\alpha \in [0, 1-\varepsilon]$,
\begin{equation*}
\vert f_s\bigl(\alpha_s^{-1}(\alpha)\bigr)\sqrt{n}\bigl(\widehat{\alpha}_s^{-1}(\alpha) - \alpha_s^{-1}(\alpha)\bigr) - B_n(\alpha) \vert \leq C_2v_n
\end{equation*} 
Hence, dividing by $f_s\bigl(\alpha_s^{-1}(\alpha)\bigr)$ which is strictly positive for all $\alpha \in [0, 1-\varepsilon]$ (assumptions \ref{as:density_regularity_positive} and \ref{as:limit_alpha_1}), we almost surely have, for $n$ large enough and for all $\alpha \in [0,1-\varepsilon]$,
\begin{equation}
\label{eq:quantile_process_mass}
\Bigl\vert \sqrt{n}\bigl(\widehat{\alpha}_s^{-1}(\alpha) - \alpha_s^{-1}(\alpha)\bigr) - \frac{B_n(\alpha)}{f_s\bigl(\alpha_s^{-1}(\alpha)\bigr)} \Bigr\vert
\leq \frac{C_2}{\inf_{[0, 1-\varepsilon]} f_s \circ \alpha_s^{-1}}v_n
\end{equation}
where $\inf_{0 \leq \alpha \leq 1-\varepsilon} f_s(\alpha_s^{-1}(\alpha))$ is strictly positive by the same argument as the one given in the proof of Lemma \ref{lem:dkw_taylor_remainder}. Now using a Taylor expansion of $\lambda_s$ which is of class $\mathcal{C}^2$ (assumption \ref{as:regulartiy_lambda}) we can write for all $\alpha \in [0,1-\varepsilon]$:
\begin{equation*}
\lambda_s(\widehat{\alpha}_s^{-1}(\alpha)) = \lambda_s(\alpha_s^{-1}(\alpha)) + \lambda_s^{\prime}(\alpha_s^{-1}(\alpha))(\widehat{\alpha}_s^{-1}(\alpha)-\alpha_s^{-1}(\alpha)) + R_2(\alpha)
\end{equation*}
where $R_2(\alpha)=\lambda_s^{''}(\xi)(\widehat{\alpha}_s^{-1}(\alpha)-\alpha_s^{-1}(\alpha))^2/2$ is the remainder of the Taylor expansion, $\xi$ being between $\alpha_s^{-1}(\alpha)$ and $\widehat{\alpha}_s^{-1}(\alpha)$. We thus have
\begin{equation*}
\begin{aligned}
\sqrt{n}(&\widehat{\mv}_s(\alpha) - \mv_s(\alpha)) - \frac{\lambda_s^{\prime}(\alpha_s^{-1}(\alpha))}{f_s\bigl(\alpha_s^{-1}(\alpha)\bigr)}B_n(\alpha)\\
&= \lambda_s^{\prime}(\alpha_s^{-1}(\alpha))\left(\sqrt{n}(\widehat{\alpha}_s^{-1}(\alpha)-\alpha_s^{-1}(\alpha)) - \frac{B_n(\alpha)}{f_s\bigl(\alpha_s^{-1}(\alpha)\bigr)}\right) + \sqrt{n}R_2(\alpha) \\
&= C_n(\alpha) + D_n(\alpha)
\end{aligned}
\end{equation*}
For the first term we almost surely have, for all $\alpha \in [0, 1-\varepsilon]$ and for $n$ large enough, 
\begin{equation*}
\vert C_n(\alpha) \vert \leq \vert \lambda_s^{\prime}(\alpha_s^{-1}(\alpha)) \vert \frac{C_2}{\inf_{[0, 1-\varepsilon]} f_s \circ \alpha_s^{-1}}v_n \, .
\end{equation*}
For all $\alpha \in [0, 1-\varepsilon]$, $\alpha_s^{-1}(\alpha) \in [0, \Vert s\Vert_{\infty}]$. As $\lambda_s'$ is continuous, $\vert \lambda_s^{\prime}(\alpha_s^{-1}(\alpha)) \vert$ is bounded by $\sup_{[0, \Vert s \Vert_{\infty}]} \vert \lambda_s'\vert$.

We treat the second term as follows. For all $\alpha \in [0, 1-\varepsilon]$, $\alpha_s^{-1}(\alpha)$ and $\widehat{\alpha}_s^{-1}(\alpha)$ are in $[0, \Vert s\Vert_{\infty}]$. Therefore $\xi \in [0, \Vert s\Vert_{\infty}]$ and as $\lambda_s''$ is continuous, $\vert \lambda_s''(\xi) \vert$ is bounded by $\sup_{[0, \Vert s \Vert_{\infty}]} \vert \lambda_s''\vert$. Thus,
\begin{equation*}
\begin{aligned}
\sup_{\alpha \in [0, 1-\varepsilon]}\vert D_n(\alpha) \vert &\leq \sup_{[0, \Vert s \Vert_{\infty}]} \vert \lambda_s''\vert \frac{\sqrt{n}}{2} \sup_{\alpha \in [0, 1-\varepsilon]} \vert \widehat{\alpha}_s^{-1}(\alpha)-\alpha_s^{-1}(\alpha) \vert^2 \\
&\leq \sup_{[0, \Vert s \Vert_{\infty}]} \vert \lambda_s''\vert\frac{C}{2} \frac{\log n}{\sqrt{n}} \quad \text{(thanks to Lemma \ref{lem:dkw_taylor_remainder})}
\end{aligned}
\end{equation*}
where the last inequality holds almost surely and for $n$ large enough and where $C > 0$ is the constant of Lemma \ref{lem:dkw_taylor_remainder}.

Eventually, combining the bounds on $\sup_{[0, 1-\varepsilon]}\vert C_n \vert$ and $\sup_{[0, 1-\varepsilon]}\vert D_n \vert$ we almost surely have
\begin{equation*}
\sup_{\alpha \in [0, 1-\varepsilon]}\left\vert \sqrt{n}(\widehat{\mv}_s(\alpha) - \mv_s(\alpha)) - \frac{\lambda_s^{\prime}(\alpha_s^{-1}(\alpha))}{f_s\bigl(\alpha_s^{-1}(\alpha)\bigr)}B_n(\alpha) \right\vert = O(v_n) \, .
\end{equation*}
This concludes the proof.
\end{proof}

\subsubsection{Bootstrap consistency: Proof of Theorem \ref{thm:bootstrap}}

Recall that for all $n \geq 1$, $v_n = n^{-1/2}\log n$ and $w_n = \sqrt{\log(h_n^{-1})/(nh_n)} + h_n^2$. Let $y_s(\alpha)$ and $\widetilde y_s(\alpha)$ be respectively defined for all $\alpha \in [\varepsilon, 1-\varepsilon]$ as $y_s(\alpha) = \lambda_s'(\alpha_s^{-1}(\alpha))/f_s(\alpha_s^{-1}(\alpha))$ and $\widetilde y_s(\alpha)= \lambda_s'(\widetilde \alpha_s^{-1}(\alpha))/\widetilde f_s(\widetilde \alpha_s^{-1}(\alpha))$.

\paragraph{Sketch of proof.} The argument is based on a strong approximation type inequality for $r_n$ (see Lemma \ref{lem:strong_approx_ineq}) and $r_n^*$ (see Lemma \ref{lem:strong_approx_bootstrap_ineq}). These inequalities can then be used to obtain a result on the rate of convergence of the cumulative distribution function of $\sup_{[\varepsilon, 1-\varepsilon]} \vert r_n \vert $ (respectively, the conditional cumulative distribution function of $\sup_{[\varepsilon, 1-\varepsilon]} \vert r_n^* \vert$) towards the supremum of a Gaussian process which depends on $y_s$ (respectively the smoothed empirical version $\widetilde y_s$ of $y_s$). We finally obtain the result using the rate of the strong uniform convergence of $\widetilde y_s/y_s$ towards $1$ on $[\varepsilon, 1-\varepsilon]$ (see Lemma \ref{lemma:gine_guillou}). The technical proofs of Lemma \ref{lem:strong_approx_ineq}, Lemma \ref{lem:strong_approx_bootstrap_ineq} and Lemma \ref{lemma:gine_guillou} are deferred to Appendix~\ref{sec:appendix_technical_results} for the sake of clarity of the proof of the main result. \newline

We first need the following strong approximation result.
\begin{lemma}\label{lem:strong_approx_ineq}
Let $\varepsilon \in (0,1)$. There exists a positive constant $C$ such that
\begin{equation*}
\mathbb{P}\left(\sup_{\alpha \in [\varepsilon, 1-\varepsilon]} \vert r_n(\alpha) - Z_n(\alpha) \vert > Cv_n\right) = O(v_n)
\end{equation*}
\end{lemma}

We also need the counterpart of Lemma \ref{lem:strong_approx_ineq} conditionally on the data set $\mathcal{D}_n$.

\begin{lemma} \label{lem:strong_approx_bootstrap_ineq}
Let $\varepsilon \in (0,1)$. Under assumptions \ref{as:bounded_score}-\ref{as:tail}, \ref{as:regulartiy_lambda} and \ref{as:density_bias}-\ref{as:bandwidth_derivative}, there exists a constant $C$ such that we $\mathbb{P}$-almost surely have,
\begin{equation*}
\mathbb{P}^*\left(\sup_{\alpha \in [\varepsilon, 1-\varepsilon]} \left \vert r_n^*(\alpha) - Z^*_n(\alpha) \right \vert > Cv_n\right) =  O(v_n)
\end{equation*}
where
\begin{equation*}
\forall \alpha \in [\varepsilon, 1-\varepsilon], \quad Z_n^*(\alpha) = \frac{\lambda_s'(\widetilde\alpha_s^{-1}(\alpha))}{\widetilde f_s(\widetilde \alpha_s^{-1}(\alpha))}B_n^*(\alpha)
\end{equation*}
and where for each $n \geq 1$ the distribution of $B_n^*$ conditionally to $\mathcal{D}_n$ is the one of a Brownian bridge.
\end{lemma}

Eventually, to control the distance between the distribution of $\sup_{[\varepsilon, 1-\varepsilon]} \vert Z_n \vert$ and the conditional distribution of $\sup_{[\varepsilon, 1-\varepsilon]} \vert Z_n^* \vert$ we need the following lemma.

\begin{lemma}
\label{lemma:gine_guillou}
Under assumptions \ref{as:bounded_score} and \ref{as:density_bias}-\ref{as:kernel_gine} there exists a constant $C > 0$ such that we almost surely have, for $n$ large enough,
\begin{equation*}
\sup_{\alpha \in [\varepsilon, 1-\varepsilon]} \left\vert \frac{\widetilde y_s(\alpha)}{y_s(\alpha)} - 1 \right\vert \leq C w_n \, .
\end{equation*}
\end{lemma}

We can now proof the main result.

\begin{proof}[Proof of Theorem \ref{thm:bootstrap}] Let $\varepsilon \in (0,1)$. Let $C > 0$ be the constant of Lemma \ref{lem:strong_approx_ineq}. If $\sup_{\alpha \in [\varepsilon, 1-\varepsilon]}\vert r_n(\alpha) - Z_n(\alpha) \vert \leq Cv_n$ then
\begin{equation*}
\Bigl \vert \sup_{\alpha \in [\varepsilon, 1-\varepsilon]} \vert r_n(\alpha) \vert - \sup_{\alpha \in [0, 1-\varepsilon]} \vert Z_n(\alpha) \vert \Bigr \vert \leq Cv_n \, .
\end{equation*}
Thus thanks to the result of Lemma \ref{lem:strong_approx_ineq}
\begin{equation*}
\mathbb{P}\left( \Bigl \vert \sup_{[\varepsilon, 1-\varepsilon]} \vert r_n \vert - \sup_{[\varepsilon, 1-\varepsilon]} \vert Z_n \vert \Bigr \vert > Cv_n\right) \leq \mathbb{P}\left(\sup_{[\varepsilon, 1-\varepsilon]} \vert r_n - Z_n \vert > Cv_n\right) = O(v_n) \, .
\end{equation*}

Therefore using a result of \citet{Sargan1971}, we have for all $t \in \mathbb{R}$,
\begin{equation*}
\begin{aligned}
\Bigl\vert &\mathbb{P}\Bigl(\sup_{[\varepsilon, 1-\varepsilon]} \vert r_n \vert \leq t \Bigr) - \mathbb{P}\Bigl(\sup_{[\varepsilon, 1-\varepsilon]} \vert Z_n \vert \leq t \Bigr) \Bigr\vert \\
&\leq \mathbb{P}\Bigl(\Bigl\vert \sup_{[\varepsilon, 1-\varepsilon]} \vert r_n \vert - \sup_{[\varepsilon, 1-\varepsilon]} \vert Z_n \vert \Bigr\vert > Cv_n\Bigr)
+ \mathbb{P}\Bigl(\Bigl\vert \sup_{[\varepsilon, 1-\varepsilon]} \vert Z_n \vert - t \Bigr\vert < Cv_n\Bigr) \\
&= \mathbb{P}\Bigl(\Bigl\vert \sup_{[\varepsilon, 1-\varepsilon]} \vert Z_n \vert - t \Bigr\vert < Cv_n\Bigr) + O(v_n)
\end{aligned}
\end{equation*}
and
\begin{equation*}
\begin{aligned}
\mathbb{P}\Bigl(\Bigl\vert \sup_{[\varepsilon, 1-\varepsilon]} \vert Z_n \vert - t \Bigr\vert < Cv_n\Bigr)
&= \mathbb{P}\Bigl(t - Cv_n < \sup_{[\varepsilon, 1-\varepsilon]} \vert Z_1 \vert < t + Cv_n\Bigr) \\
&= \int_{t - Cv_n}^{t + Cv_n} \phi_{\bar Z}(x)dx \leq 2DCv_n
\end{aligned}
\end{equation*}
where $\phi_{\bar Z}$ denotes the density of $\sup_{[\varepsilon, 1-\varepsilon]} \vert Z_1 \vert$ and where for the first equality we use the fact that for all $n \geq 1$ the random variable $Z_n(\alpha)$ is equal in distribution to $Z_1(\alpha)$ and for the last inequality the result from \citep{Pitt1979} stating that $\sup_{\alpha \in [\varepsilon, 1-\varepsilon]} \vert Z_1(\alpha) \vert$ has a density bounded by a constant $D > 0$ as the supremum of a Gaussian process $\{ Z_1(\alpha), \varepsilon \leq \alpha \leq 1-\varepsilon \}$ such that for all $\alpha \in [\varepsilon, 1-\varepsilon]$, $\inf_{\alpha \in [\varepsilon, 1-\varepsilon]}\text{Var}[Z_1(\alpha)] > 0$ and $\mathbb{P}(\sup_{\alpha \in [\varepsilon, 1-\varepsilon]}\vert Z_1(\alpha)\vert < \infty)=1$. Indeed, as $B_1$ is a Brownian bridge it is almost surely continuous on $[0,1]$ and as $y_s$ is continuous on $[\varepsilon, 1-\varepsilon]$, $Z_1$ is almost surely continuous on $[\varepsilon, 1-\varepsilon]$ and therefore almost surely $\sup_{\alpha \in [\varepsilon, 1-\varepsilon]}\vert Z_1(\alpha)\vert < \infty$. Furthermore, $\alpha \in [\varepsilon, 1-\varepsilon] \mapsto \text{Var}[Z_1(\alpha)] = \alpha(1-\alpha)y_s^2(\alpha)$ is continuous and thanks to Lemma \ref{lemma:y_s_positive}, for all $\alpha \in [\varepsilon, 1-\varepsilon]$, $\text{Var}[Z_1(\alpha)] > 0$. Therefore $\inf_{\alpha \in [\varepsilon, 1-\varepsilon]}\text{Var}[Z_1(\alpha)]$ is attained and is strictly positive.

Therefore,
\begin{equation}
\label{eq:first_bootstrap_sup}
\sup_{t\in \mathbb{R}}\Bigl\vert \mathbb{P}\Bigl(\sup_{\alpha \in [\varepsilon, 1-\varepsilon]} \vert r_n(\alpha) \vert \leq t \Bigr) - \mathbb{P}\Bigl(\sup_{\alpha \in [\varepsilon, 1-\varepsilon]} \vert Z_1(\alpha) \vert \leq t \Bigr) \Bigr\vert = O(v_n)
\end{equation}

Reasoning similarly as above, thanks to the result of Lemma \ref{lem:strong_approx_bootstrap_ineq}, there exists a constant $C_1$ such that, $\mathbb{P}$-almost surely, as $n \rightarrow \infty$,
\begin{equation*}
\mathbb{P}^*\left( \Bigl \vert \sup_{\alpha \in [\varepsilon, 1-\varepsilon]} \vert r_n^*(\alpha) \vert - \sup_{\alpha \in [0, 1-\varepsilon]} \vert Z_n^*(\alpha) \vert \Bigr \vert > C_1v_n\right) = O(v_n) \, .
\end{equation*}

Therefore using the result of \citep{Sargan1971}, we $\mathbb{P}$-almost surely have, for all $t \in \mathbb{R}$,
\begin{equation*}
\begin{aligned}
\Bigl\vert &\mathbb{P}^*\Bigl(\sup_{[\varepsilon, 1-\varepsilon]} \vert r_n^*\vert \leq t \Bigl) - \mathbb{P}^*\Bigl(\sup_{[\varepsilon, 1-\varepsilon]} \vert Z^*_n \vert \leq t \Bigr) \Bigr\vert \\
&= \mathbb{P}^*\Bigl(t - C_1v_n < \sup_{[\varepsilon, 1-\varepsilon]} \vert Z^*_1 \vert < t + C_1v_n\Bigr) + O(v_n) \\
&= \mathbb{P}^*\Bigl( \sup_{[\varepsilon,1-\varepsilon]} \vert Z^*_1 \vert < t + C_1v_n\Bigr) - \mathbb{P}\Bigl( \sup_{[\varepsilon,1-\varepsilon]} \vert Z_1 \vert < t + C_1v_n\Bigr) \\
&\phantom{+}+ \mathbb{P}\Bigl( \sup_{[\varepsilon,1-\varepsilon]} \vert Z_1 \vert < t + C_1v_n\Bigr) - \mathbb{P}\Bigl( \sup_{[\varepsilon,1-\varepsilon]} \vert Z_1 \vert < t - C_1v_n\Bigr) \\
&\phantom{+}+ \mathbb{P}\Bigl( \sup_{[\varepsilon,1-\varepsilon]} \vert Z_1 \vert < t - C_1v_n\Bigr) - \mathbb{P}^*\Bigl(\sup_{[\varepsilon, 1-\varepsilon]} \vert Z^*_1\vert < t -C_1v_n\Bigr) + O(v_n) \\
&\leq 2 \sup_{t\in \mathbb{R}}\Bigl\vert\mathbb{P}^*\Bigl( \sup_{[\varepsilon, 1-\varepsilon]} \vert Z_1^* \vert \leq t\Bigr) - \mathbb{P}\Bigl( \sup_{[\varepsilon, 1-\varepsilon]} \vert Z_1 \vert \leq t\Bigr) \Bigr\vert \\
&\phantom{+} + \mathbb{P}\Bigl(t - C_1v_n < \sup_{[\varepsilon, 1-\varepsilon]} \vert Z_1 \vert < t + C_1v_n\Bigr) + O(v_n)\\
&\leq 2 \sup_{t\in \mathbb{R}}\Bigl\vert\mathbb{P}^*\Bigl( \sup_{[\varepsilon, 1-\varepsilon]} \vert Z_1^* \vert \leq t\Bigr) - \mathbb{P}\Bigl( \sup_{[\varepsilon, 1-\varepsilon]} \vert Z_1 \vert \leq t\Bigr) \Bigr\vert + 2DC_1v_n + O(v_n)\, .
\end{aligned}
\end{equation*}

We now need the following result.

\begin{lemma}
\label{lemma:rate_bootstrap_distribution}
The exists a constant $C_2 > 0$ such that we $\mathbb{P}$-almost surely have for $n$ large enough,
\begin{equation*}
\sup_{t\in \mathbb{R}}\Bigl\vert\mathbb{P}^*\Bigl( \sup_{\alpha \in [\varepsilon, 1-\varepsilon]} \vert Z_1^*(\alpha) \vert \leq t\Bigr) - \mathbb{P}\Bigl( \sup_{\alpha \in [\varepsilon, 1-\varepsilon]} \vert Z_1(\alpha) \vert \leq t\Bigr) \Bigr\vert \leq C_2w_n
\end{equation*}
\end{lemma}

\begin{proof}[Proof of Lemma \ref{lemma:rate_bootstrap_distribution}]
Let $t \in \mathbb{R}$. If $t<0$ both probabilities are equal to 0 as the two random variables involved are positive. Hence we have (for all $\omega$ in the sample space)
\begin{equation}
\label{eq:negative_t}
\sup_{t < 0}\Bigl\vert\mathbb{P}^*\Bigl( \sup_{\alpha \in [\varepsilon, 1-\varepsilon]} \vert Z_1^*(\alpha) \vert \leq t\Bigr) - \mathbb{P}\Bigl( \sup_{\alpha \in [\varepsilon, 1-\varepsilon]} \vert Z_1(\alpha) \vert \leq t\Bigr) \Bigr\vert = 0 \leq C_2w_n
\end{equation}
Let $t \geq 0$. Thanks to Lemma \ref{lemma:gine_guillou} there exists a constant $C_3 > 0$ independent of $\alpha$ such that we almost surely have for $n$ large enough:
\begin{equation}
\label{eq:from_lemma_gine_guillou}
\forall \alpha \in [\varepsilon, 1-\varepsilon], 1 - C_3w_n \leq \left\vert \frac{\widetilde y_s(\alpha)}{y_s(\alpha)} \right\vert \leq 1 + C_3w_n \, .
\end{equation}

Furthermore, we can assume, given $n$ large enough, that there exists a constant $C_4 > 0$ such that $1-C_3w_n > C_4$ as $w_n$ tends towards $0$. Observe that
\begin{equation*}
\begin{aligned}
\mathbb{P}^*&\Bigl( \sup_{\alpha \in [\varepsilon, 1-\varepsilon]} \vert Z_1^*(\alpha) \vert \leq t \Bigr) = \mathbb{P}^*\Bigl( \sup_{\alpha \in [\varepsilon, 1-\varepsilon]} \left\vert \frac{\widetilde y_s(\alpha)}{y_s(\alpha)}y_s(\alpha)B_1^*(\alpha) \right\vert \leq t \Bigr) \\
&= \mathbb{P}^*\left( \forall \alpha \in [\varepsilon, 1-\varepsilon], \left\vert \frac{\widetilde y_s(\alpha)}{y_s(\alpha)} \right\vert \left\vert y_s(\alpha)B_1^*(\alpha) \right\vert \leq t \right) \, .
\end{aligned}
\end{equation*}
and decompose the last term as follows
\begin{equation}
\label{eq:decomp_variance}
\begin{aligned}
\mathbb{P}^*&\left( \left\{ \forall \alpha \in [\varepsilon, 1-\varepsilon], \left\vert \frac{\widetilde y_s(\alpha)}{y_s(\alpha)} \right\vert \left\vert y_s(\alpha)B_1^*(\alpha) \right\vert \leq t \right\}, \mathcal{Z} \right) \\
&\quad \quad \quad + \mathbb{P}^*\left( \left\{ \forall \alpha \in [\varepsilon, 1-\varepsilon], \left\vert \frac{\widetilde y_s(\alpha)}{y_s(\alpha)} \right\vert \left\vert y_s(\alpha)B_1^*(\alpha) \right\vert \leq t \right\}, \overline{\mathcal{Z}} \right) \, .
\end{aligned}
\end{equation}
where the event $\mathcal{Z}$ is defined as
\begin{equation*}
\mathcal{Z} = \left\{ \forall \alpha \in [\varepsilon, 1-\varepsilon], \left\vert \frac{\widetilde y_s(\alpha)}{y_s(\alpha)} \right\vert \geq 1 - C_3w_n \right\}
\end{equation*}
and its complementary is given by
\begin{equation*}
\overline{\mathcal{Z}} = \left\{ \exists \alpha \in [\varepsilon, 1-\varepsilon], \left\vert \frac{\widetilde y_s(\alpha)}{y_s(\alpha)} \right\vert < 1 - C_3w_n \right\} \, .
\end{equation*}
The first term of \eqref{eq:decomp_variance} is lower than
\begin{equation*}
\mathbb{P}^*\left((1 - C_3w_n) \sup_{[\varepsilon, 1-\varepsilon]} \left\vert y_sB_1^* \right\vert \leq t \right)
= \mathbb{P}\left( \sup_{[\varepsilon, 1-\varepsilon]} \left\vert y_sB_1 \right\vert \leq \frac{t}{1 - C_3w_n} \right)
\end{equation*}
where the last equality stands from the fact that the law of $B^*_1$ conditionally on $\mathcal{D}_n$ is equal to the law of $B_1$. The second term of \eqref{eq:decomp_variance} is lower than
\begin{equation*}
\mathbb{P}^*\left( \exists \alpha \in [\varepsilon, 1-\varepsilon], \left\vert \widetilde y_s(\alpha)/y_s(\alpha) \right\vert < 1 - C_3w_n \right)
\end{equation*}
which is $\mathbb{P}$-almost surely equal to $0$ thanks to \eqref{eq:from_lemma_gine_guillou}. Note that the fact that this holds $\mathbb{P}$-almost surely is independent of $t$. Therefore we $\mathbb{P}$-almost surely have: for all $t \geq 0$,
\begin{equation*}
\begin{aligned}
\mathbb{P}^*\Bigl( \sup_{\alpha \in [\varepsilon, 1-\varepsilon]} \vert Z_1^*(\alpha) \vert \leq t \Bigr) &\leq \mathbb{P}\left( \sup_{\alpha \in [\varepsilon, 1-\varepsilon]} \left\vert y_s(\alpha)B_1(\alpha) \right\vert \leq \frac{t}{1 - C_3w_n} \right) \\
&=\mathbb{P}\left( \sup_{\alpha \in [\varepsilon, 1-\varepsilon]} \left\vert Z_1(\alpha) \right\vert \leq t \right) + \phi_{\bar Z}(\xi)t\frac{C_3w_n}{1-C_3w_n}
\end{aligned}
\end{equation*}
where we use a Taylor expansion and where $\xi$ is between $t$ and $t/(1-C_3w_n)$.

For $n$ large enough, we have $1 > 1 - C_3w_n > C_4 > 0$. On the one hand this gives $1 < 1/(1-C_3w_n)$ and, as $t \geq 0$, $t < t/(1-C_3w_n)$. Therefore, as $\xi$ is between $t$ and $t/(1-C_3w_n)$, we have in fact, $t < \xi < t/(1-C_3w_n)$. Thus $\phi_{\bar Z}(\xi)t \leq \phi_{\bar Z}(\xi)\xi$. On the other hand, we also have $1/(1-C_3w_n) < 1/C_4$. This eventually gives $\mathbb{P}$-almost surely, for all $t \geq 0$,
\begin{equation}
\label{eq:upper_bound_sup_dib}
\mathbb{P}^*\Biggl( \sup_{[\varepsilon, 1-\varepsilon]} \vert Z_1^* \vert \leq t \Biggr)
\leq \mathbb{P}\left( \sup_{[\varepsilon, 1-\varepsilon]} \left\vert Z_1 \right\vert \leq t \right) + \sup_{x \in \mathbb{R}^+}\vert \phi_{\bar Z}(x)x \vert \frac{C_3w_n}{C_4} \, .
\end{equation}

A lower bound can be obtained in a similar fashion: $\mathbb{P}$-almost surely, for all $t \geq 0$,
\begin{equation}
\label{eq:lower_bound_sup_dib}
\mathbb{P}\left( \sup_{[\varepsilon, 1-\varepsilon]} \vert Z_1 \vert \leq t \right) 
\leq \mathbb{P}^*\left( \sup_{[\varepsilon, 1-\varepsilon]} \left\vert Z^*_1 \right\vert \leq t \right) + 2\sup_{x \in \mathbb{R}^+}\vert x\phi_{\bar Z}(x)\vert C_3w_n \, .
\end{equation}

Combining the two inequalities \eqref{eq:upper_bound_sup_dib} and \eqref{eq:lower_bound_sup_dib} we obtain that there exists a constant $C_5 > 0$ such that we $\mathbb{P}$-almost surely have
\begin{equation*}
\sup_{t \geq 0}\left\vert \mathbb{P}^*\left( \sup_{[\varepsilon, 1-\varepsilon]} \left\vert Z^*_1 \right\vert \leq t \right) - \mathbb{P}\left( \sup_{[\varepsilon, 1-\varepsilon]} \vert Z_1 \vert \leq t \right) \right\vert \leq C_5 \sup_{x \in \mathbb{R}^+}\vert x\phi_{\bar Z}(x)\vert w_n \, .
\end{equation*}

Let $\varphi$ be the density of the Gaussian distribution $\mathcal{N}(0,1)$. Using a remark in the paper of \citep[p. 854]{Tsirelson1976}, there exists $M > 0$ such that for all $x \geq M$,
\begin{equation}
\label{eq:tsirelson_remark}
x\phi_{\bar Z}(x) \leq \frac{x}{v_1}\varphi\left(\frac{x-a-b}{v_1}\right) \, ,
\end{equation}
where $v_1$ is such that $v_1^2 = \sup_{\alpha \in [\varepsilon, 1-\varepsilon]} \text{Var}[\vert Z_1(\alpha) \vert]$, $a > 0$ is arbitrary and $b \geq 0$ only depends on $v_1$, $\phi_{\bar Z}$ and the cumulative distribution function of the Gaussian distribution $\mathcal{N}(0,1)$. Now the function $x \in \mathbb{R}^+ \mapsto x/v_1 \cdot \varphi((x-a-b)/v_1)$ is bounded on $\mathbb{R}^+$ as it is continuous and tends towards $0$ when $x$ tends to $+\infty$. Therefore from \eqref{eq:tsirelson_remark} we have that $\sup_{x \geq M} \vert x\phi_{\bar Z}(x)\vert < +\infty$. And as from \citep{Pitt1979} $\phi_{\bar Z}$ is bounded on $\mathbb{R}$, so is $x \mapsto x\phi_{\bar Z}(x)$ on $[0, M]$. Eventually, $\sup_{x \in \mathbb{R}^+} \vert x\phi_{\bar Z}(x)\vert < +\infty$ and there exists a constant $C_2 > 0$ such that $\mathbb{P}$-almost surely,

\begin{equation}
\label{eq:positive_t}
\sup_{t\geq 0}\left\vert \mathbb{P}^*\left( \sup_{\alpha \in [\varepsilon, 1-\varepsilon]} \left\vert Z^*_1(\alpha) \right\vert \leq t \right) - \mathbb{P}\left( \sup_{\alpha \in [\varepsilon, 1-\varepsilon]} \vert Z_1(\alpha) \vert \leq t \right) \right\vert \leq C_2 w_n
\end{equation}

Combining \eqref{eq:negative_t} and \eqref{eq:positive_t} we obtain the result of the Lemma.
\end{proof}

From Lemma \ref{lemma:rate_bootstrap_distribution}, we $\mathbb{P}$-almost surely have
\begin{equation*}
\sup_{t \in \mathbb{R}}\Bigl\vert \mathbb{P}^*\Bigl(\sup_{\alpha \in [\varepsilon, 1-\varepsilon]} \vert r_n^*(\alpha)\vert \leq t \Bigl) - \mathbb{P}^*\Bigl(\sup_{\alpha \in [\varepsilon, 1-\varepsilon]} \vert Z^*_n(\alpha) \vert \leq t \Bigr) \Bigr\vert = O(w_n)
\end{equation*}

and thus
\begin{equation*}
\begin{aligned}
\sup_{t \in \mathbb{R}}&\Bigl\vert \mathbb{P}^*\Bigl(\sup_{\alpha \in [\varepsilon, 1-\varepsilon]} \vert r_n^*(\alpha)\vert \leq t \Bigl) - \mathbb{P}\Bigl(\sup_{\alpha \in [\varepsilon, 1-\varepsilon]} \vert Z_1(\alpha) \vert \leq t \Bigr) \Bigr\vert \\
& \leq \sup_{t \in \mathbb{R}}\Bigl\vert \mathbb{P}^*\Bigl(\sup_{\alpha \in [\varepsilon, 1-\varepsilon]} \vert r_n^*(\alpha)\vert \leq t \Bigl) - \mathbb{P}^*\Bigl(\sup_{\alpha \in [\varepsilon, 1-\varepsilon]} \vert Z^*_n(\alpha) \vert \leq t \Bigr) \Bigr\vert \\
&\phantom{\leq} + \sup_{t \in \mathbb{R}}\Bigl\vert \mathbb{P}^*\Bigl(\sup_{\alpha \in [\varepsilon, 1-\varepsilon]} \vert Z_n^*(\alpha)\vert \leq t \Bigl) - \mathbb{P}\Bigl(\sup_{\alpha \in [\varepsilon, 1-\varepsilon]} \vert Z_1(\alpha) \vert \leq t \Bigr) \Bigr\vert \\
&= O(w_n)
\end{aligned}
\end{equation*}

Finally combining this last inequality with \eqref{eq:first_bootstrap_sup} we obtain the result of the theorem.

\end{proof}

\subsection{A-Rank Algorithm}

\subsubsection{Proof of Theorem \ref{thm:rate_bound}}
The proof follows closely the one of theorem 2 in \citep{ClemVay09}.
For clarity, we start off with recalling the following result.
\begin{proposition}\label{prop:SN06}(\cite{ScottNowak06}) Suppose that the assumptions of Theorem \ref{thm:rate_bound} are fulfilled. For any $\alpha \in (0,1)$, let $\widehat \Omega_{\alpha} \in \mathcal{G}$ be a solution of the constrained optimization problem: 
$$
\min_{\Omega\in \mathcal{G}}\lambda(\Omega) \text{ subject to } \widehat{F}(\Omega)\geq \alpha-\phi_n(\delta).
$$
Let $\delta \in (0,1)$. With probability at least $1-\delta$: $\forall \alpha \in (0,1)$, $\forall n\geq 1$,
\begin{equation}
\alpha-2\phi_n(\delta) \leq F(\widehat \Omega_{\alpha}) \leq \alpha \text{ and } \mv^*(\alpha)-\frac{2\phi_n(\delta)}{Q^*(\alpha)} \leq \lambda(\widehat \Omega_{\alpha})\leq \mv^*(\alpha).
\end{equation}
\end{proposition}

The major part of this result is given by Theorem 3 and Lemma 19 in \citep{ScottNowak06}. Although they prove the result for a given $\alpha$, it can be extended to all $\alpha \in (0,1)$. The upper bound on $F(\widehat \Omega_{\alpha})$ can be obtained with the following reasoning.
Suppose that we have $F(\widehat \Omega_{\alpha}) > \alpha$, then $F(\widehat \Omega_{\alpha}) \geq \alpha$ and $\lambda(\Omega^*_{\alpha}) \leq \lambda(\widehat \Omega_{\alpha})$ because $\Omega^*_{\alpha}$ minimizes $\lambda$ over all measurable sets $\Omega$ such that $F(\Omega) \geq \alpha$. The upper bound on the volume gives $\lambda(\widehat \Omega_{\alpha}) \leq \lambda(\Omega^*_{\alpha})$, thus $\lambda(\widehat \Omega_{\alpha}) = \lambda(\Omega^*_{\alpha})$. Hence $\widehat \Omega$ also minimizes $\lambda$ over all measurable sets $\Omega$ such that $F(\Omega) \geq \alpha$ and by uniqueness of the solution under assumptions \ref{as:bounded_density} and \ref{as:flat_parts}, we have $F(\widehat \Omega_{\alpha}) = F(\Omega^*_{\alpha}) = \alpha$ which contradicts $F(\widehat \Omega_{\alpha}) > \alpha$.

We now prove the following lemma quantifying the uniform deviation of the empirical local error from the true local error over all dyadic scales.

\begin{lemma}
\label{lemma:uniform_deviation}
Suppose that assumptions \ref{as:mvset_in_class} and \ref{as:rademacher_order} are satisfied . Let $\delta \in (0,1)$. With probability at least $1-\delta$, we have: $\forall n \geq 1$,
\begin{equation*}
\sup_{\substack{0 \leq j \leq j_{\text{max}} \\ 0 \leq k \leq 2^j}}  \vert \widehat{\mathcal{E}}(I_{j,k}) - \mathcal{E}(I_{j,k}) \vert \leq \frac{4\phi_n(\delta)}{Q^*(1 - \varepsilon)} \, .
\end{equation*}
\end{lemma}

\begin{proof}
Let $0 \leq j \leq j_{\text{max}}$, $0 \leq k \leq 2^j-1$,
\begin{equation*}
\begin{aligned}
\vert \widehat{\mathcal{E}}(I_{j,k}) - \mathcal{E}(I_{j,k}) \vert &\leq \vert \lambda(\widehat \Omega_{\alpha_{j, k+1}}) - \lambda(\Omega^*_{\alpha_{j, k+1}}) \vert + \vert \lambda(\widehat \Omega_{\alpha_{j, k}}) - \lambda(\Omega^*_{\alpha_{j, k}}) \vert \\
&\leq 2\sup_{\substack{0 \leq j \leq j_{\text{max}} \\ 0 \leq k \leq 2^j}} \vert \lambda(\widehat \Omega_{\alpha_{j, k}}) - \lambda(\Omega^*_{\alpha_{j, k}}) \vert \, .
\end{aligned}
\end{equation*}

From proposition \ref{prop:SN06}, with probability at least $1 - \delta$ we have: for all $0 \leq j \leq j_{\text{max}}, 0 \leq k \leq 2^j$,
\begin{equation*}
\vert \lambda(\widehat \Omega_{\alpha_{j, k}}) - \lambda(\Omega^*_{\alpha_{j, k}}) \vert \leq 2\frac{\phi_n(\delta)}{Q^*(\alpha_{j,k})} \leq \frac{2\phi_n(\delta)}{Q^*(1-\varepsilon)} \, ,
\end{equation*}

therefore with probability at least $1 - \delta$, for all $0 \leq j \leq j_{\text{max}}$, $0 \leq k \leq 2^j-1$,
\begin{equation*}
\vert \widehat{\mathcal{E}}(I_{j,k}) - \mathcal{E}(I_{j,k}) \vert \leq \frac{4\phi_n(\delta)}{Q^*(1 - \varepsilon)}
\end{equation*}
and the result follows.
\end{proof}

From Lemma \ref{lemma:uniform_deviation}, with probability at least $1 - \delta$, for all $0 \leq j \leq j_{\text{max}}$, $0 \leq k \leq 2^j-1$,
\begin{equation}
\label{eq:finer_coarser_ineq}
\mathcal{E}(I_{j,k}) - \frac{4\phi_n(\delta)}{Q^*(1 - \varepsilon)} \leq \widehat{\mathcal{E}}(I_{j,k}) \leq \mathcal{E}(I_{j,k}) + \frac{4\phi_n(\delta)}{Q^*(1 - \varepsilon)} \, .
\end{equation}

Let $E_{\sigma_{\tau}}(\mv^*)$ be the piecewise constant approximant of the optimal $\mv$ curve built from the same recursive strategy as the one implemented in Algorithm \ref{algo:adaptive} but based on the true local error $\mathcal{E}$ and without the condition $j < j_{\text{max}}$, denoting $\sigma_{\tau}$ the associated meshgrid. We thus have, for every $\alpha \in [0, 1 - \varepsilon]$:
\begin{equation*}
E_{\sigma_{\tau}}(\mv^*)(\alpha) = \sum_{\substack{
j,k \\
\alpha_{j,k} \in \sigma_{\tau}}
} \lambda(\Omega^*_{\alpha_{j,k+1}}) \cdot \mathbb{I}\{\alpha \in I_{j,k}  \} \, .
\end{equation*}

We also use the notation $E_{\widehat \sigma_{\tau}}$ to denote the function defined for all $\alpha \in [0, 1-\varepsilon]$ as 
\begin{equation*}
E_{\widehat \sigma_{\tau}}(\mv^*)(\alpha) = \sum_{\substack{
j,k \\
\alpha_{j,k} \in \widehat \sigma_{\tau}}
} \lambda(\Omega^*_{\alpha_{j,k+1}}) \cdot \mathbb{I}\{\alpha \in I_{j,k}  \}
\end{equation*}
where we recall that $\widehat \sigma_{\tau}$ is the meshgrid obtained in Algorithm \ref{algo:adaptive} implemented with the empirical error estimate $\widehat{\mathcal{E}}$.

Choosing $\tau = 5\phi_n(\delta)/Q^*(1-\varepsilon)$, we obtain that with probability at least $1-\delta$, the meshgrid $\widehat \sigma_{\tau}$ is finer than $\sigma_{\widetilde \tau_1}$ where $\widetilde \tau_1 = \tau + (4\phi_n(\delta) + 1/n)/Q^*(1 - \varepsilon)$, and coarser than $\sigma_{\widetilde \tau_0}$ where $\widetilde \tau_0 = \tau - 4\phi_n(\delta)/Q^*(1 - \varepsilon)$. Indeed, thanks to \eqref{eq:finer_coarser_ineq}, $\mathcal{E}(I_{j,k}) > \widetilde \tau_1$ implies that with probability $1-\delta$, $\widehat{\mathcal{E}}(I_{j,k}) > \widetilde \tau_1 - 4\phi_n/Q^*(1-\varepsilon) > \tau$. It also implies that $j < j_{\text{max}}$. Indeed, if $j \geq j_{\text{max}}$, then $I_{j,k} = [\alpha_1, \alpha_2]$ is such that $\alpha_2 - \alpha_1 < 1/n$ which gives, using the mean value theorem, $\mathcal{E}(I_{j,k}) = \lambda(\Omega_{\alpha_2}) - \lambda(\Omega_{\alpha_1}) < 1/(Q^*(1-\varepsilon)n)$ which contradicts $\mathcal{E}(I_{j,k}) > \widetilde \tau_1 > 1/(Q^*(1-\varepsilon)n$. Therefore with probability at least $1-\delta$, splitting $I_{j,k}$ for the true local error and a tolerance $\widetilde \tau_1$ implies splitting $I_{j,k}$ in Algorithm \ref{algo:adaptive}. Analogously, $\widehat{\mathcal{E}}(I_{j,k}) > \tau$ implies, thanks to \eqref{eq:finer_coarser_ineq}, that with probability at least $1-\delta$, $\mathcal{E}(I_{j,k}) > \widetilde \tau_0$. Therefore with probability at least $1-\delta$, splitting $I_{j,k}$ in Algorithm \ref{algo:adaptive} implies splitting $I_{j,k}$ for the true local error and a tolerance $\widetilde \tau_0$.

We now use the following bound
\begin{equation*}
\begin{aligned}
\sup_{\alpha \in [0, 1-\varepsilon]} \vert\widehat{\mv^*}(\alpha) - \mv^*(\alpha) \vert & \leq \sup_{\alpha \in [0, 1-\varepsilon]}\vert \widehat{\mv^*}(\alpha)- E_{\widehat \sigma_{\tau}}(\mv^*)(\alpha) \vert
\\ &+ \sup_{\alpha \in [0, 1-\varepsilon]} \vert E_{\widehat \sigma_{\tau}}(\mv^*)(\alpha) - \mv^*(\alpha) \vert \, .
\end{aligned}
\end{equation*}

The second term is bounded by $\widetilde \tau_1$ with probability at least $1 - \delta$ because with probability at least $1-\delta$ the meshgrid $\widehat \sigma_{\tau}$ is finer than $\sigma_{\widetilde \tau_1}$ and therefore
\begin{equation*}
\begin{aligned}
\sup_{\alpha \in [0, 1-\varepsilon]} \vert E_{\widehat \sigma_{\tau}}(\mv^*)(\alpha) - \mv^*(\alpha) \vert &\leq \sup_{\alpha \in [0, 1-\varepsilon]} \vert E_{\sigma_{\widetilde \tau_1}}(\mv^*)(\alpha) - \mv^*(\alpha) \vert  \\
&=\sup_{\substack{
\alpha \in I \\ I \text{ interval of } \sigma_{\widetilde \tau_1}} } \vert E_{\sigma_{\widetilde \tau_1}}(\mv^*)(\alpha) - \mv^*(\alpha) \vert \\
& \leq \sup_{\substack{
\alpha \in I \\ I \text{ interval of } \sigma_{\widetilde \tau_1}} } \mathcal{E}(I) \leq \widetilde \tau_1 \, .
\end{aligned}
\end{equation*}
On the same event, for the first term we have,
\begin{equation*}
\begin{aligned}
\sup_{\alpha \in [0, 1-\varepsilon]}\vert \widehat{\mv^*}(\alpha)- E_{\widehat \sigma_{\tau}}(\mv^*)(\alpha) \vert &\leq \sup_{\substack{0 \leq j \leq j_{\text{max}} \\ 0 \leq k \leq 2^j}} \vert \lambda(\Omega^*_{\alpha_{j,k+1}}) - \lambda(\widehat \Omega_{\alpha_{j,k+1}}) \vert \\
&\leq \frac{2\phi_n(\delta)}{Q^*(1 - \varepsilon)} \, .
\end{aligned}
\end{equation*}
Thus we finally have
\begin{equation*}
\sup_{\alpha \in [0, 1-\varepsilon]} \vert \widehat{\mv^*}(\alpha) - \mv^*(\alpha) \vert \leq \widetilde \tau_1 + \frac{2\phi_n(\delta)}{Q^*(1 - \varepsilon)} \leq \frac{1}{Q^*(1 - \varepsilon)}\left(11\phi_n(\delta) + \frac{1}{n} \right) \, .
\end{equation*}
For the second part of the theorem on the cardinality of the meshgrid $\widehat \sigma_{\tau}$ obtained with Algorithm \ref{algo:adaptive} we use the following lemma \citep{DeVore98}.
\begin{lemma}
\label{lemma:appro_rate}
Let $\text{\emph{card}}(\sigma_{\tau})$ be the number of terminal nodes obtained when the algorithm is implemented with the true local error $\mathcal{E}(I)$. Suppose that assumptions \ref{as:bounded_density}, \ref{as:flat_parts} and \ref{as:llogl} are fulfilled. There exists a universal constant $C > 0$ such that, for all $\tau > 0$:
\begin{equation*}
\text{\emph{card}}(\sigma_{\tau}) \leq \frac{C}{\tau} \Vert \mv^{*^\prime} \Vert_{L\log L} \, .
\end{equation*}
\end{lemma}
Applying this lemma to $\widetilde \tau_0 = \phi_n(\delta)/Q^*(1-\varepsilon)$ we obtain that the number of terminal nodes $\text{card}(\sigma_{\widetilde \tau_0})$ of $\sigma_{\widetilde \tau_0}$ is such that
\begin{equation*}
\text{card}(\sigma_{\widetilde \tau_0}) \leq C Q^*(1-\varepsilon) \frac{\Vert \mv^{*^\prime} \Vert_{L\log L}}{\phi_n(\delta)} \, .
\end{equation*}
As $\widehat \sigma_{\tau}$ is coarser than $\sigma_{\widetilde \tau_0}$ with probability at least $1-\delta$, the number of terminal nodes of the former is bounded by the number of terminal nodes of the latter with probability at least $1-\delta$.

\subsubsection{Proof of Theorem \ref{thm:rate_bound_score}}

We use the following notation: $\widetilde \alpha_k = F(\widetilde \Omega_{\alpha_k})$. We first prove a lemma quantifying the loss coming from the monotonicity step of Algorithm \ref{algo:arank}.

\begin{lemma}
\label{lemma:error_stacking}
Suppose that the assumptions of Theorem \ref{thm:rate_bound_score} are satisfied. Let $\delta \in (0, 1)$. For all $n \geq 1$, with probability at least $1 - \delta$, we have: $\forall k \in \{1, \dots, \widehat K \}$,
\begin{equation}
\vert \widetilde \alpha_k - \alpha_k \vert \leq \max(kC_1 \phi_n(\delta)^{\frac{\gamma}{\gamma +1}}, 2\phi_n(\delta)) 
\end{equation}
and
\begin{equation}
\label{eq:volume_monotonicity}
\vert \lambda(\widetilde \Omega_{\alpha_k}) - \mv^*(\alpha_k) \vert \leq \max(kC_2 \phi_n(\delta)^{\frac{\gamma}{\gamma +1}}, 2\phi_n(\delta)/Q^*(1-\varepsilon)) \, .
\end{equation}
where $C_1 > 0$ and $C_2 > 0$ are constants depending only on $\gamma$, $\Vert f \Vert_{\infty}$ and the constant $C$ of assumption \ref{as:fast_rates}.
\end{lemma}

\begin{proof}
First, from proposition \ref{prop:SN06}, we have with probability at least $1-\delta$, $\forall k \in \{0, \dots, \widehat K \}$,
\begin{equation}
\label{eq:SN_with_abs}
\vert \alpha_k - F(\widehat \Omega_{\alpha_k}) \vert \leq 2\phi_n(\delta) \text{ and } \vert \lambda(\widehat \Omega_{\alpha_k}) - \lambda(\Omega^*_{\alpha_k}) \vert \leq \frac{2\phi_n(\delta)}{Q^*(1-\varepsilon)} \, .
\end{equation}
For $k=0$, we have $\alpha_0 = 0$ and thus $\widetilde \Omega_{\alpha_0} = \widehat \Omega_{\alpha_0} = \emptyset$ and the inequalities of the Lemma are trivially satisfied.
For $k=1$, we have $\widetilde \Omega_{\alpha_1} = \widehat \Omega_{\alpha_1}$ and with \eqref{eq:SN_with_abs} the inequalities of the Lemma are satisfied.
For all $k \in \{0, \dots, \widehat K \}$,
\begin{equation*}
F(\widetilde \Omega_{\alpha_k}) = F\Bigl(\bigcup_{j=0}^k \widehat \Omega_{\alpha_j}\Bigr) \geq F(\widehat \Omega_{\alpha_k}) \geq \alpha_k - 2\phi_n(\delta)
\end{equation*}
and $\widetilde \alpha_k - \alpha_k \geq -2\phi_n(\delta)$.
The same argument gives
\begin{equation*}
\lambda(\widetilde \Omega_{\alpha_k}) - \lambda(\Omega^*_{\alpha_k}) \geq -\frac{2\phi_n(\delta)}{Q^*(1-\varepsilon)} \, .
\end{equation*}

We show how to obtain the upper bound for the case $k=2$. We have $\widetilde \Omega_{\alpha_2} = \widehat \Omega_{\alpha_1} \cup \widehat \Omega_{\alpha_2}$ and
\begin{equation*}
F(\widetilde \Omega_{\alpha_2}) - \alpha_2 = F(\widetilde \Omega_{\alpha_2}) - F(\Omega^*_{\alpha_2}) \leq F(\widetilde \Omega_{\alpha_2}) - F(\Omega^*_{\alpha_2} \cap \widetilde \Omega_{\alpha_2}) = F(\widetilde \Omega_{\alpha_2} \setminus \Omega^*_{\alpha_2}) \, .
\end{equation*}
Now,
\begin{equation*}
\widetilde \Omega_{\alpha_2} \setminus \Omega^*_{\alpha_2} = \{ \widehat \Omega_{\alpha_1} \setminus \Omega^*_{\alpha_2} \} \cup \{\widehat \Omega_{\alpha_2} \setminus \Omega^*_{\alpha_2} \} \subset \{ \widehat \Omega_{\alpha_1} \setminus \Omega^*_{\alpha_1} \} \cup \{\widehat \Omega_{\alpha_2} \setminus \Omega^*_{\alpha_2} \}
\end{equation*}
where the last inclusion comes from the fact that $\Omega^*_{\alpha_1} \subset \Omega^*_{\alpha_2}$. Hence,
\begin{equation*}
F(\widetilde \Omega_{\alpha_2}) - \alpha_2 \leq F(\widehat \Omega_{\alpha_1} \setminus \Omega^*_{\alpha_1}) + F(\widehat \Omega_{\alpha_2} \setminus \Omega^*_{\alpha_2})
\end{equation*}
and thus
\begin{equation*}
	F(\widetilde \Omega_{\alpha_2}) - \alpha_2 \leq F(\widehat \Omega_{\alpha_1} \Delta \Omega^*_{\alpha_1}) + F(\widehat \Omega_{\alpha_2} \Delta \Omega^*_{\alpha_2}) \, .
\end{equation*}

Using assumption \ref{as:fast_rates}, for any $\alpha \in (0,1)$, if $\widehat \Omega_{\alpha}$ denotes the solution of the optimization problem \eqref{eq:empminvolset}, we can as in \cite{Polonik97} (see proof of Lemma 3.2) bound $F(\widehat \Omega_{\alpha} \Delta \Omega^*_{\alpha})$: with probability at least $1-\delta$,
\begin{equation*}
F(\Omega^*_{\alpha} \Delta \widehat \Omega_{\alpha}) \leq \Vert f \Vert_{\infty}Ct^{\gamma} + 2\frac{\Vert f \Vert_{\infty}}{t}\phi_n(\delta)
\end{equation*}
Minimizing the right-hand side with respect to $t$ gives
\begin{equation*}
F(\widehat \Omega \Delta \Omega^*_{\alpha}) \leq C_1\phi_n(\delta)^{\frac{\gamma}{\gamma +1}} \, ,
\end{equation*}
where $C_1 > 0$ is a constant depending only on $\gamma$, $C$ and $\Vert f \Vert_{\infty}$.
We thus finally have, with probability at least $1-\delta$, 
\begin{equation*}
\widetilde \alpha_2 - \alpha_2 \leq 2C_1 \phi_n(\delta)^{\frac{\gamma}{\gamma +1}} 
\end{equation*}
and therefore,
\begin{equation*}
\vert \widetilde \alpha_2 - \alpha_2 \vert \leq \max(2C_1 \phi_n(\delta)^{\frac{\gamma}{\gamma +1}}, 2\phi_n(\delta)) \, .
\end{equation*}
For $k \geq 2$, we can obtain in a similar fashion,
\begin{equation*}
F(\widetilde \Omega_{\alpha_k}) - \alpha_k \leq \sum_{j=1}^k F(\widehat \Omega_{\alpha_j} \setminus \Omega^*_{\alpha_j}) \leq \sum_{j=1}^k F(\widehat \Omega_{\alpha_j} \Delta \Omega^*_{\alpha_j})
\end{equation*}
and derive the following result: with probability at least $1-\delta$, for all $k \geq 2$, 
\begin{equation*}
\vert \widetilde \alpha_k - \alpha_k \vert \leq \max(kC_1 \phi_n(\delta)^{\frac{\gamma}{\gamma +1}}, 2\phi_n(\delta)) \, .
\end{equation*}

For \eqref{eq:volume_monotonicity}, first note that for all $\alpha \in [0, 1-\varepsilon]$, $\lambda(\Omega^*_{\alpha} \Delta \widehat \Omega_{\alpha}) \leq C_2\phi_n(\delta)^{\frac{\gamma}{\gamma + 1}}$ where $C_2 > 0$ depends only on $\gamma$ and $C$. Therefore the same procedure leads to: with probability $1-\delta$, for all $k \geq 0$,
\begin{equation*}
\vert \lambda(\widetilde \Omega_{\alpha_k}) - \mv^*(\alpha_k) \vert \leq \max(kC_2 \phi_n(\delta)^{\frac{\gamma}{\gamma +1}}, 2\phi_n(\delta)/Q^*(1-\varepsilon)) \, .
\end{equation*}
\end{proof}
We can now prove Theorem \ref{thm:rate_bound_score}. We first write
\begin{equation*}
\begin{aligned}
\sup_{\alpha \in [0, 1- \varepsilon]} \Bigl\{ \mv_{\hat s}(\alpha) - \mv^*(\alpha) \Bigr\} &\leq \sup_{\alpha \in [0, 1- \varepsilon]} \Bigl\{ \mv_{\hat s}(\alpha) - E_{\widehat \sigma_{\tau}}(\mv^*)(\alpha) \Bigr\}
\\&+  \sup_{\alpha \in [0, 1- \varepsilon]} \Bigl\{ E_{\widehat \sigma_{\tau}}(\mv^*)(\alpha) - \mv^*(\alpha) \Bigr\}\, .
\end{aligned}
\end{equation*}
Taking $\tau = 4\phi_n(\delta)/Q^*(1-\varepsilon) + c_n$, $c_n > 0$ and using the same argument as in the proof of Theorem \ref{thm:rate_bound}, the second term of the right-hand side is bounded by $\widetilde \tau_1 = \tau + (4\phi_n(\delta) + 1/n)/Q^*(1-\varepsilon)$. We now deal with the first term. Let $\alpha \in [0, 1 - \varepsilon]$,
\begin{equation*}
\mv_{\hat s}(\alpha) - E_{\widehat \sigma_{\tau}}(\mv^*)(\alpha) = \sum_{k=0}^{\widehat K-1} \lambda(\widetilde \Omega_{\alpha_{k+1}}) \cdot \mathbb{I}\{ \alpha \in [\widetilde \alpha_k, \widetilde \alpha_{k+1})\} -  E_{\widehat \sigma_{\tau}}(\mv^*)(\alpha) \, .
\end{equation*}
But one can observe that
\begin{equation*}
E_{\widehat \sigma_{\tau}}(\mv^*)(\alpha) = \sum_{k=0}^{\widehat K-1} E_{\widehat \sigma_{\tau}}(\mv^*)(\alpha) \cdot \mathbb{I}\{ \alpha \in [\widetilde \alpha_k, \widetilde \alpha_{k+1})\}
\end{equation*}
to obtain
\begin{equation*}
\begin{aligned}
\mv_{\hat s}(\alpha) - E_{\widehat \sigma_{\tau}}(\mv^*)(\alpha) &= \sum_{k=0}^{\widehat K-1} \lambda(\widetilde \Omega_{\alpha_{k+1}}) \cdot \mathbb{I}\{ \alpha \in [ \widetilde \alpha_k, \widetilde \alpha_{k+1})\} -  E_{\widehat \sigma_{\tau}}(\mv^*)(\alpha) \\
&= \sum_{k=0}^{\widehat K-1} (\lambda(\widetilde \Omega_{\alpha_{k+1}}) - E_{\widehat \sigma_{\tau}}(\mv^*)(\alpha) ) \cdot \mathbb{I}\{ \alpha \in [ \widetilde \alpha_k, \widetilde \alpha_{k+1})\} \, .
\end{aligned}
\end{equation*}
Now $\alpha \mapsto E_{\widehat \sigma_{\tau}}(\mv^*)(\alpha)$ is increasing, therefore on $[ \widetilde \alpha_k, \widetilde \alpha_{k+1} )$, $E_{\widehat \sigma_{\tau}}(\mv^*)(\alpha) \geq E_{\widehat \sigma_{\tau}}(\mv^*)(\widetilde \alpha_k)$ and
\begin{equation*}
\begin{aligned}
\mv_{\hat s}(\alpha) -& E_{\widehat \sigma_{\tau}}(\mv^*)(\alpha) \leq 
\sum_{k=0}^{\widehat K-1} (\lambda(\widetilde \Omega_{\alpha_{k+1}}) - E_{\widehat \sigma_{\tau}}(\mv^*)(\widetilde \alpha_k) ) \cdot \mathbb{I}\{ \alpha \in [ \widetilde \alpha_k, \widetilde \alpha_{k+1})\} \\
&\leq \max_{0 \leq k \leq \widehat K -1} \vert \lambda(\widetilde \Omega_{\alpha_{k+1}}) - E_{\widehat \sigma_{\tau}}(\mv^*)(\widetilde \alpha_k) \vert \sum_{k=0}^{\widehat K-1} \mathbb{I}\{ \alpha \in [ \widetilde \alpha_k, \widetilde \alpha_{k+1})\} \\
&= \max_{0 \leq k \leq \widehat K -1} \vert \lambda(\widetilde \Omega_{\alpha_{k+1}}) - E_{\widehat \sigma_{\tau}}(\mv^*)(\widetilde \alpha_k) \vert \, .
\end{aligned}
\end{equation*}
We thus have
\begin{equation*}
\begin{aligned}
\sup_{[0, 1- \varepsilon]} \Bigl\{ \mv_{\hat s} - E_{\widehat \sigma_{\tau}}(\mv^*) \Bigr\} &\leq \max_{0 \leq k \leq \widehat K -1} \vert \lambda(\widetilde \Omega_{\alpha_{k+1}}) - E_{\widehat \sigma_{\tau}}(\mv^*)(\widetilde \alpha_k) \vert \\
&\leq \max_{0 \leq k \leq \widehat K -1} \vert \lambda(\widetilde \Omega_{\alpha_{k+1}}) - E_{\widehat \sigma_{\tau}}(\mv^*)(\alpha_k) \vert \\
&+ \max_{0 \leq k \leq \widehat K -1} \vert E_{\widehat \sigma_{\tau}}(\mv^*)(\alpha_k) - E_{\widehat \sigma_{\tau}}(\mv^*)(\widetilde \alpha_k) \vert \, .
\end{aligned}
\end{equation*}
The first term is equal to $\max_{0 \leq k \leq \widehat K -1} \vert \lambda(\widetilde \Omega_{\alpha_{k+1}}) - \mv^*(\alpha_{k+1}) \vert$ and thanks to Lemma \ref{lemma:error_stacking} we have with probability at least $1-\delta$
\begin{equation}
\label{eq:max_volume_stacking}
\max_{0 \leq k \leq \widehat K -1} \vert \lambda(\widetilde \Omega_{\alpha_{k+1}}) - \mv^*(\alpha_{k+1}) \vert \leq \max(\widehat K C_2 \phi_n(\delta)^{\frac{\gamma}{\gamma +1}}, 2\phi_n(\delta)/Q^*(1-\varepsilon)) \, .
\end{equation}
For the second term we write
\begin{equation*}
\begin{aligned}
\Bigl\vert E_{\widehat \sigma_{\tau}}(\mv^*)&(\alpha_k) - E_{\widehat \sigma_{\tau}}(\mv^*)(\widetilde \alpha_k) \Bigr\vert \\
&= \Bigl\vert \mv^*(\alpha_{k+1}) - \sum_{j=0}^{\widehat K-1} \mv^*(\alpha_{j+1})\cdot \mathbb{I}\{\widetilde \alpha_k \in [\alpha_j, \alpha_{j+1}) \} \Bigr\vert \\
&= \Bigl\vert \mv^*(\alpha_{k+1}) - \mv^*(\alpha_{j_0+1})\Bigr\vert \, ,
\end{aligned}
\end{equation*}
where $j_0 \in \{0, \dots, \widehat K -1 \}$ is such that $\widetilde \alpha_k \in [\alpha_{j_0}, \alpha_{j_0+1})$. We bound $\vert \mv^*(\alpha_{k+1}) - \mv^*(\alpha_{j_0+1})\vert$ by
\begin{equation*}
\Bigl\vert \mv^*(\alpha_{k+1}) - \mv^*(\alpha_{k})\Bigr\vert + \Bigl\vert \mv^*(\alpha_k) - \mv^*(\widetilde \alpha_k)\Bigr\vert + \Bigl\vert \mv^*(\widetilde \alpha_k) - \mv^*(\alpha_{j_0+1})\Bigr\vert \, .
\end{equation*}
We bound the first term as follows. Let $I=[\alpha_k, \alpha_{k+1}]$. If the depth $j_{\text{max}}$ has been reached, $\alpha_{k+1} - \alpha_{k} < 1/n$. Therefore the mean value theorem gives:
\begin{equation*}
\vert \mv^*(\alpha_{k+1}) - \mv^*(\alpha_{k}) \vert \leq \frac{1}{nQ^*(1-\varepsilon)} \, .
\end{equation*}
Otherwise, this means that $\widehat{\mathcal{E}}(I) \leq \tau$ and we decompose $\vert \mv^*(\alpha_{k+1}) - \mv^*(\alpha_{k})\vert = \mathcal{E}(I)$ into
\begin{equation*}
\mathcal{E}(I) = \bigl(\mathcal{E}(I) - \widehat{\mathcal{E}}(I)\bigr) + \widehat{\mathcal{E}}(I) .
\end{equation*}
Using Lemma \ref{lemma:error_stacking}, the right-hand side is bounded by $4\phi_n(\delta)/Q^*(1-\varepsilon) +  \tau$. Eventually,
\begin{equation*}
\vert \mv^*(\alpha_{k+1}) - \mv^*(\alpha_{k})\vert \leq \max\left(\frac{4\phi_n(\delta)}{Q^*(1-\varepsilon)} +  \tau, \frac{1}{nQ^*(1-\varepsilon)}\right)
\end{equation*}

Using the mean value theorem and Lemma \ref{lemma:error_stacking}, the second term can be bounded by
\begin{equation*}
\frac{1}{Q^*(1-\varepsilon)}\vert \alpha_k - \widetilde \alpha_k \vert \leq \frac{1}{Q^*(1-\varepsilon)} \max(\widehat K C_1 \phi_n(\delta)^{\frac{\gamma}{\gamma +1}}, 2\phi_n(\delta)) \, .
\end{equation*}
As $\widetilde \alpha_k \in [\alpha_{j_0}, \alpha_{j_0 +1})$ and as $\mv^*$ is increasing, the third term is bounded by $\vert \lambda(\Omega^*_{\alpha_{j_0+1}}) - \lambda(\Omega^*_{\alpha_{j_0}})\vert$ which is also bounded as the first term by $\max(4\phi_n(\delta)/Q^*(1-\varepsilon) +  \tau, 1/(nQ^*(1-\varepsilon)))$. We thus obtain
\begin{equation*}
\begin{aligned}
\Bigl\vert E_{\widehat \sigma_{\tau}}(\mv^*)(\alpha_k) - E_{\widehat \sigma_{\tau}}(\mv^*)(\widetilde \alpha_k) \Bigr\vert &\leq 2\max\left(\frac{4\phi_n(\delta)}{Q^*(1-\varepsilon)} +  \tau, \frac{1}{nQ^*(1-\varepsilon)}\right) \\
&\phantom{\leq} + \frac{\max(\widehat K C_1 \phi_n(\delta)^{\frac{\gamma}{\gamma +1}}, 2\phi_n(\delta))}{Q^*(1-\varepsilon)} \, .
\end{aligned}
\end{equation*}
and combining this inequality with \eqref{eq:max_volume_stacking}, we have
\begin{equation*}
\begin{aligned}
\sup_{\alpha \in [0, 1- \varepsilon]} \Bigl\{ \mv_{\hat s}(\alpha) - E_{\widehat \sigma_{\tau}}(\mv^*)(\alpha) \Bigr\} &\leq \max(\widehat K C_2 \phi_n(\delta)^{\frac{\gamma}{\gamma +1}}, 2\phi_n(\delta)/Q^*(1-\varepsilon)) \\
&\phantom{\leq}+ 2\max\left(\frac{4\phi_n(\delta)}{Q^*(1-\varepsilon)} +  \tau, \frac{1}{nQ^*(1-\varepsilon)}\right) \\
&\phantom{\leq} + \frac{\max(\widehat K C_1 \phi_n(\delta)^{\frac{\gamma}{\gamma +1}}, 2\phi_n(\delta))}{Q^*(1-\varepsilon)} \, .
\end{aligned}
\end{equation*}
Eventually, with probability at least $1 - \delta$,
\begin{equation}
\label{eq:bound_score}
\begin{aligned}
\sup_{\alpha \in [0, 1- \varepsilon]} \Bigl\{ \mv_{\hat s}(\alpha) - \mv^*(\alpha) \Bigr\} &\leq \max(\widehat K C_2 \phi_n(\delta)^{\frac{\gamma}{\gamma +1}}, 2\phi_n(\delta)/Q^*(1-\varepsilon)) \\
&+ 2\max\left(\frac{4\phi_n(\delta)}{Q^*(1-\varepsilon)} +  \tau, \frac{1}{nQ^*(1-\varepsilon)}\right) \\
&+ \tau + \frac{4\phi_n(\delta)}{Q^*(1-\varepsilon)} + \frac{1}{nQ^*(1-\varepsilon)}\\
&+ \frac{\max(\widehat K C_1 \phi_n(\delta)^{\frac{\gamma}{\gamma +1}}, 2\phi_n(\delta))}{Q^*(1-\varepsilon)}\, .
\end{aligned}
\end{equation}

Using the same argument as in the proof of Theorem \ref{thm:rate_bound}, on the same event, which holds with probability at least $1-\delta$, $\widehat \sigma_{\tau}$ is coarser than $\sigma_{\widetilde{\tau_0}}$ where $\widetilde{\tau_0} = \tau - 4\phi_n(\delta)/Q^*(1-\varepsilon) = c_n$ and therefore $\widehat K$ is of the order of $c_n$. Using assumption \ref{as:rademacher_order}, for $n$ large enough, $\phi_n(\delta) \leq \phi_n(\delta)^{\gamma/(\gamma + 1)}$. Therefore the right-hand side of \eqref{eq:bound_score} is of the order of:
\begin{equation*}
\frac{1}{c_n}\left(\frac{1}{n}\right)^{\frac{\gamma}{2(\gamma +1)}} + c_n \, .
\end{equation*}

Choosing $c_n=(1/n)^{\gamma/(4(1+\gamma))}$ to balance the two terms of this expression, we finally obtain that with probability at least $1-\delta$,
\begin{equation*}
\sup_{\alpha \in [0, 1- \varepsilon]} \Bigl\{ \mv_{\hat s}(\alpha) - \mv^*(\alpha) \Bigr\} = O\left(n^{-\frac{\gamma}{4(1+\gamma)}}\right) \, .
\end{equation*}

\section{Technical results}
\label{sec:appendix_technical_results}

This section contains the proofs of technical results that have been postponed for the sake of clarity.

\subsection{Proof of Lemma \ref{lem:support_z_f}}

\begin{proof} We recall here that $\mathcal{Z}_f = \{x, F_f(f(x))>0\}$ and $\supp(f) = \{x, f(x) > 0 \}$. If $f(x)=0$ then $F_f(f(x)) = 0$ thus $\overline{\supp(f)} \subset \overline{\mathcal{Z}_f}$ and $\overline{\mathcal{Z}_f} \Delta \overline{\supp(f)} = \overline{\mathcal{Z}_f} \setminus \overline{\supp(f)} = \overline{\mathcal{Z}_f} \cap \supp(f)$. Therefore we have to prove $\lambda(\overline{\mathcal{Z}_f} \cap \supp(f)) = 0$. Let's assume $\lambda(\overline{\mathcal{Z}_f} \cap \supp(f)) > 0$ and let $f^* = \sup_{u \in \overline{\mathcal{Z}_f} \cap \supp(f)} f(u)$. The intersection $\overline{\mathcal{Z}_f} \cap \supp(f)$ is not empty. Otherwise we would have $\lambda(\overline{\mathcal{Z}_f} \cap \supp(f)) = 0$. Therefore there exists $u \in \overline{\mathcal{Z}_f} \cap \supp(f)$ and we have $f^* \geq f(u) > 0$ as $u \in \supp(f)$. We also have $F_f(f^*) = 0$. Indeed, $f^*$ is in the closure of $\{f(u), u \in \overline{\mathcal{Z}_f} \cap \supp(f)\}$ and therefore there exists a sequence $(f_m)_{m \geq 0}$ of elements of $\{f(u), u \in \overline{\mathcal{Z}_f} \cap \supp(f)\}$ which converges towards $f^*$ as $m$ tends to infinity. Thus $F_f(f_m) = 0$ for all $m \geq 0$ and as $F_f$ is continuous, $F_f(f^*) = 0$, i.e.
\begin{equation*}
\int_{\{u, 0 \leq f(u) \leq f^* \}} f(u)\mathrm{d}u = 0 \, .
\end{equation*}
This implies that $f$ is equal to 0 $\lambda$-almost everywhere on $\{u, 0 \leq f(u) \leq f^* \}$. However, if $u \in \overline{\mathcal{Z}_f} \cap \supp(f)$ then, by definition of $f^*$, $f(u) \leq f^*$. Therefore $\overline{\mathcal{Z}_f} \cap \supp(f) \subset \{u, 0 \leq f(u) \leq f^* \}$ and as $\lambda(\overline{\mathcal{Z}_f} \cap \supp(f)) > 0$ and $f > 0$ on $\supp(f)$, $\{u, 0 \leq f(u) \leq f^* \}$ contains a subset of non null $\lambda$-measure on which $f > 0$. This is a contradiction with the fact that $f$ is equal to 0 $\lambda$-almost everywhere on $\{u, 0 \leq f(u) \leq f^* \}$ and thus $\lambda(\overline{\mathcal{Z}_f} \Delta \overline{\supp(f)}) = 0$, which is equivalent to $\lambda(\mathcal{Z}_f \Delta \supp(f)) = 0$.
\end{proof}

\subsection{Proof of Lemma \ref{lem:strong_approx_ineq}}

\begin{proof}[Proof of Lemma \ref{lem:strong_approx_ineq}]
The proof is based on Remark 3.2.4 in \citep{Csorgo1983}. Let $q_n$ denote the quantile process defined for all $\alpha \in [\varepsilon, 1- \varepsilon]$ by
\begin{equation*}
q_n(\alpha) = \sqrt{n}f_s(F_s^{\dag}(\alpha))(\widehat F_s^{\dag}(\alpha) - F_s^{\dag}(\alpha)) \, .
\end{equation*}
Using the mean value theorem we can write for all $\alpha \in [\varepsilon, 1- \varepsilon]$,
\begin{equation*}
\begin{aligned}
q_n(\alpha) &= \sqrt{n}f_s(F_s^{\dag}(\alpha))\bigl(F_s^{\dag}(\widehat U^{\dag}(\alpha)) - F_s^{\dag}(\alpha)\bigr) \\
&= \frac{f_s(F_s^{\dag}(\alpha))}{f_s(F_s^{\dag}(\xi))}\sqrt{n}(\widehat U^{\dag}(\alpha) - \alpha)
\end{aligned}
\end{equation*}
where $\xi$ is between $\widehat U^{\dag}(\alpha)$ and $\alpha$ and where $\widehat U^{\dag}(\alpha)$ is defined as in the proof of Lemma \ref{lem:dkw_taylor_remainder}. Let $u_n$ be the uniform quantile process defined in Lemma \ref{lem:dkw_taylor_remainder}. For all $\alpha \in [\varepsilon, 1-\varepsilon]$, $q_n(\alpha) = u_n(\alpha)f_s(F_s^{\dag}(\alpha))/f_s(F_s^{\dag}(\xi))$.

Now we know from Theorem 1 in \citep{CR78} that for all $z_1$ and for all $n$,
\begin{equation}
\label{eq:strong_approx_ineq_unif}
\mathbb{P}\Bigl(\sup_{\alpha \in [0,1]}\vert u_n(\alpha) - B_n^1(\alpha) \vert > \frac{C_1\log n + z_1}{\sqrt{n}}\Bigr) \leq C_2\exp(-C_3z_1)
\end{equation}
where $C_1, C_2$ and $C_3$ are positive absolute constants and where $\{B_n^1(\alpha), \alpha \in [0,1]\}_{n \geq 1}$ is the same sequence of Brownian Bridges as the one of the proof of assertion $(ii)$ of Theorem \ref{thm:estimation}. For all $\alpha \in [\varepsilon, 1-\varepsilon]$, we can write
\begin{equation}
\label{eq:decom_quantile_brownian}
q_n(\alpha) - B_n^1(\alpha) = \bigl(u_n(\alpha) - B_n^1(\alpha)\bigr) + u_n(\alpha)\left(\frac{f_s(F_s^{\dag}(\alpha))}{f_s(F_s^{\dag}(\xi))} - 1\right) \, .
\end{equation}
Using a Taylor expansion of $\lambda_s$ as in the proof of part $(ii)$ of Theorem \ref{thm:estimation} we have for all $\alpha \in [\varepsilon, 1-\varepsilon]$,
\begin{equation*}
\begin{aligned}
r_n(\alpha) - Z_n(\alpha) &= \frac{\lambda_s'(\alpha_s^{-1}(\alpha))}{f_s(\alpha_s^{-1}(\alpha))}\Bigl(\sqrt{n}f_s(\alpha_s^{-1}(\alpha))(\widehat \alpha_s^{-1}(\alpha) - \alpha_s^{-1}(\alpha)) - B_n^1(1-\alpha) \Bigr) \\
&\quad \quad + \sqrt{n}\frac{\lambda_s''(\xi_1)}{2}(\widehat \alpha_s^{-1}(\alpha) - \alpha_s^{-1}(\alpha))^2
\end{aligned}
\end{equation*}
where $\xi_1$ is between $\widehat \alpha_s^{-1}(\alpha)$ and $\alpha_s^{-1}(\alpha)$. With the change of variable $\alpha' = 1-\alpha$, we obtain for all $\alpha' \in [\varepsilon,1-\varepsilon]$ that $\xi_1$ is between $\widehat F_s^{\dag}(\alpha')$ and $F_s^{\dag}(\alpha')$ and that
\begin{equation*}
r_n(\alpha') - Z_n(\alpha') = \frac{\lambda_s'(F_s^{\dag}(\alpha'))}{f_s(F_s^{\dag}(\alpha'))}(q_n(\alpha') - B_n^1(\alpha')) + \sqrt{n}\frac{\lambda_s''(\xi_1)}{2}(\widehat F_s^{\dag}(\alpha') - F_s^{\dag}(\alpha'))^2 \, .
\end{equation*}
and therefore
\begin{equation}
\label{eq:decom_rn_Zn}
\vert r_n(\alpha') - Z_n(\alpha') \vert \leq \frac{\vert \lambda_s'(F_s^{\dag}(\alpha'))\vert}{f_s(F_s^{\dag}(\alpha'))} \vert q_n(\alpha') - B_n^1(\alpha')\vert + \sqrt{n}\frac{\vert \lambda_s''(\xi_1) \vert }{2}\vert \widehat F_s^{\dag}(\alpha') - F_s^{\dag}(\alpha')\vert^2 \, .
\end{equation}
To bound the right-hand side of the previous inequality we need the following results. From the DKW inequality (see proof of Lemma \ref{lem:dkw_taylor_remainder}) we have for all $z_2 > 0$ and for all $n$,
\begin{equation}
\label{eq:dkw_unif_quantile}
\mathbb{P}(\sup_{\alpha \in [0,1]} \vert u_n(\alpha) \vert \geq z_2) \leq 2\exp(-2z_2^2) \, .
\end{equation}
From Theorem 1.5.1 in \citep{Csorgo1983} we have for all $z_3 > 0$ and for all $n$,
\begin{equation}
\label{eq:ratio_taylor_ineq_1}
\begin{aligned}
\mathbb{P}\left(\sup_{\alpha \in [\varepsilon, 1-\varepsilon]} \left\vert \frac{f_s(F_s^{\dag}(\alpha))}{f_s(F_s^{\dag}(\xi))} - 1 \right\vert > z_3\right) \leq 4 & (\lfloor c \rfloor  + 1)\Bigl(e^{-n\varepsilon h\bigl((1+z_3)^{1/2(\lfloor c \rfloor +1)}\bigr)} \\
&+ e^{-n\varepsilon h\bigl((1+z_3)^{-1/2(\lfloor c \rfloor +1)}\bigr)}\Bigr)
\end{aligned}
\end{equation}
where for all $z > 0$, $h(z) = z + \log(1/z) -1$ and where $x \mapsto \lfloor x \rfloor$ denotes the floor part function. Finally from the proof of Theorem 1.4.3 in \citep{Csorgo1983} we have for all $z_4 \geq 1$ and for all $n$,
\begin{equation}
\label{eq:ratio_taylor_ineq}
\begin{aligned}
\mathbb{P}\left(\sup_{\alpha \in [\varepsilon, 1-\varepsilon]} \left\vert \frac{f_s(F_s^{\dag}(\alpha))}{f_s(F_s^{\dag}(\xi))} \right\vert > z_4\right) \leq 2 & (\lfloor c \rfloor + 1)\Bigl(e^{-n\varepsilon h\bigl(z_4^{1/2(\lfloor c \rfloor+1)}\bigr)} \\
&+ e^{-n\varepsilon h\bigl(z_4^{-1/2(\lfloor c \rfloor+1)}\bigr)}\Bigr) \, .
\end{aligned}
\end{equation}

Let $z_1, z_2$ and $z_3$ be positive and let $z_4 \geq 1$. We consider the four following events:
\begin{equation*}
\begin{aligned}
\mathcal{Z}_1 &= \left\{ \sup_{\alpha \in [0,1]}\vert u_n(\alpha) - B_n^1(\alpha) \vert > (C_1\log n + z_1)/\sqrt{n}\right\}, \\
\mathcal{Z}_2 &= \left\{ \sup_{\alpha \in [0,1]} \vert u_n(\alpha) \vert \geq z_2\right\}, \\
\mathcal{Z}_3 &= \left\{ \sup_{\alpha \in [\varepsilon, 1-\varepsilon]} \left\vert f_s(F_s^{\dag}(\alpha))/f_s(F_s^{\dag}(\xi)) - 1 \right\vert > z_3 \right\} \text{ and } \\
\mathcal{Z}_4 &= \left\{ \sup_{\alpha \in [\varepsilon, 1-\varepsilon]} \left\vert f_s(F_s^{\dag}(\alpha))/f_s(F_s^{\dag}(\xi)) \right\vert > z_4 \right\} \, .
\end{aligned}
\end{equation*}
Let $\omega \in \overline{\mathcal{Z}_1} \cap \overline{\mathcal{Z}_2} \cap \overline{\mathcal{Z}_3} \cap \overline{\mathcal{Z}_4}$ where $\overline{\mathcal{Z}}$ denotes the complement of a set $\mathcal{Z}$.
We first deal with the first term of the right-hand side of \eqref{eq:decom_rn_Zn}. From \eqref{eq:decom_quantile_brownian} we have for all $\alpha' \in [\varepsilon, 1-\varepsilon]$,
\begin{equation*}
\begin{aligned}
\vert q_n(\alpha') - B_n^1(\alpha') \vert &\leq \vert u_n(\alpha') - B_n^1(\alpha') \vert + \vert u_n(\alpha') \vert \left\vert \frac{f_s(F_s^{\dag}(\alpha'))}{f_s(F_s^{\dag}(\xi))} - 1 \right\vert \\
&\leq \frac{C_1\log n + z_1}{\sqrt{n}} + z_2z_3
\end{aligned}
\end{equation*}
where the second inequality holds because $\omega \in \overline{\mathcal{Z}_1} \cap \overline{\mathcal{Z}_2} \cap \overline{\mathcal{Z}_3}$. Now for all $\alpha' \in [\varepsilon, 1-\varepsilon]$, $F_s^{\dag}(\alpha') \in [0, \Vert s \Vert_{\infty}]$. As $\lambda_s'$ is continuous $\vert \lambda_s'(F_s^{\dag}(\alpha'))\vert \leq \sup_{[0, \Vert s \Vert_{\infty}]} \vert \lambda_s' \vert < \infty$. Furthermore, for all $\alpha' \in [\varepsilon, 1-\varepsilon]$, $f_s(F_s^{\dag}(\alpha')) \geq \inf_{\alpha \in [\varepsilon, 1-\varepsilon]} f_s(F_s^{\dag}(\alpha'))$ where the infimum is strictly positive. Therefore for all $\alpha' \in [\varepsilon, 1-\varepsilon]$,
\begin{equation*}
\frac{\vert \lambda_s'(F_s^{\dag}(\alpha'))\vert}{f_s(F_s^{\dag}(\alpha'))} \vert q_n(\alpha') - B_n^1(\alpha')\vert \leq \frac{\sup_{[0, \Vert s \Vert_{\infty}]} \vert \lambda_s' \vert}{\inf_{[\varepsilon, 1-\varepsilon]} f_s \circ F_s^{\dag}}\left(\frac{C_1\log n + z_1}{\sqrt{n}} + z_2z_3 \right)
\end{equation*}

We now deal with the second term of the right-hand side of \eqref{eq:decom_rn_Zn}. For all $\alpha \in [\varepsilon, 1-\varepsilon]$, $\alpha_s^{-1}(\alpha)$ and $\widehat{\alpha}_s^{-1}(\alpha)$ are in $[0, \Vert s\Vert_{\infty}]$. Therefore $\xi_1 \in [0, \Vert s\Vert_{\infty}]$ and as $\lambda_s''$ is continuous, $\vert \lambda_s''(\xi_1) \vert \leq \sup_{[0, \Vert s \Vert_{\infty}]} \vert \lambda_s''\vert$. Now, as in the proof of Lemma \ref{lem:dkw_taylor_remainder}, using the mean value theorem we can write for $\alpha' \in [\varepsilon, 1]$,
\begin{equation*}
\begin{aligned}
\vert \widehat F^{\dag}_s(\alpha') - F^{\dag}_s(\alpha') \vert &= \frac{1}{f_s(F_s^{\dag}(\xi))} \cdot \vert \widehat U^{\dag}(\alpha') - \alpha' \vert \\
&= \frac{1}{f_s(F_s^{\dag}(\alpha'))}\left\vert\frac{f_s(F_s^{\dag}(\alpha'))}{f_s(F_s^{\dag}(\xi))}\right\vert \cdot \vert u_n(\alpha') \vert n^{-1/2}
\end{aligned}
\end{equation*}
where $\xi$ is the one of \eqref{eq:decom_quantile_brownian}. As $\omega \in \overline{\mathcal{Z}_2} \cap \overline{\mathcal{Z}_4}$ and bounding $1/f_s(F_s^{\dag}(\alpha'))$ by $1/\inf_{[\varepsilon, 1-\varepsilon]} f_s \circ F_s^{\dag}$ as previously we obtain
\begin{equation*}
\sqrt{n}\frac{\lambda_s''(\xi_1)}{2}\vert \widehat F^{\dag}_s(\alpha') - F^{\dag}_s(\alpha') \vert^2 \leq \sup_{[0, \Vert s \Vert_{\infty}]} \vert \lambda_s''\vert\frac{z_2^2z_4^2}{2\sqrt{n}(\inf_{[\varepsilon, 1-\varepsilon]} f_s \circ F_s^{\dag})^2} \, .
\end{equation*}
We eventually have
\begin{equation*}
\begin{aligned}
\sup_{\alpha \in [\varepsilon, 1-\varepsilon]} \vert r_n(\alpha) - Z_n(\alpha) \vert &\leq \frac{\sup_{[0, \Vert s \Vert_{\infty}]} \vert \lambda_s' \vert}{\inf_{[\varepsilon, 1-\varepsilon]} f_s \circ F_s^{\dag}}\left(\frac{C_1\log n + z_1}{\sqrt{n}} + z_2z_3 \right) \\
& \quad \quad + \sup_{[0, \Vert s \Vert_{\infty}]} \vert \lambda_s''\vert\frac{z_2^2z_4^2}{2\sqrt{n}(\inf_{[\varepsilon, 1-\varepsilon]} f_s \circ F_s^{\dag})^2} \, .
\end{aligned} 
\end{equation*}
Choosing, $z_1 = c_1\log n$ with $c_1 > 0$ arbitrary, $z_2 = \sqrt{\log n}$, $z_3 = c_3\sqrt{\log n /n}$ with $c_3 > 0$ arbitrary, and $z_4 = 2$ we obtain that,
\begin{equation*}
\begin{aligned}
\sup_{\alpha \in [\varepsilon, 1-\varepsilon]} \vert r_n(\alpha) - Z_n(\alpha) \vert
&\leq \frac{\sup_{[0, \Vert s \Vert_{\infty}]} \vert \lambda_s' \vert}{\inf_{[\varepsilon, 1-\varepsilon]} f_s \circ F_s^{\dag}}(C_1+c_1+c_3)\frac{\log n}{\sqrt{n}}\\
& \quad \quad + 2\sup_{[0, \Vert s \Vert_{\infty}]} \vert \lambda_s''\vert\frac{\log n}{\sqrt{n}(\inf_{[\varepsilon, 1-\varepsilon]} f_s \circ F_s^{\dag})^2} \\
&= c_4v_n
\end{aligned} 
\end{equation*}
where the constant $c_4$ is given by
\begin{equation*}
c_4 = \frac{\sup_{[0, \Vert s \Vert_{\infty}]} \vert \lambda_s' \vert}{\inf_{[\varepsilon, 1-\varepsilon]} f_s \circ F_s^{\dag}}(C_1+c_1+c_3) + \frac{2\sup_{[0, \Vert s \Vert_{\infty}]} \vert \lambda_s''\vert}{(\inf_{[\varepsilon, 1-\varepsilon]} f_s \circ F_s^{\dag})^2} \, .
\end{equation*}
Let $\mathcal{Z}_5$ be the event defined by
\begin{equation*}
\mathcal{Z}_5 = \left\{ \sup_{\alpha \in [\varepsilon, 1-\varepsilon]} \vert r_n(\alpha) - Z_n(\alpha) \vert > c_4v_n \right\} \, .
\end{equation*}

We have shown that $\omega \in \overline{\mathcal{Z}_5}$ therefore
\begin{equation*}
\begin{aligned}
\mathbb{P}(\mathcal{Z}_5) &\leq \mathbb{P}(\mathcal{Z}_1) + \mathbb{P}(\mathcal{Z}_2) + \mathbb{P}(\mathcal{Z}_3) + \mathbb{P}(\mathcal{Z}_4) \\
&\leq \frac{C_2}{n^{C_3c_1}} + \frac{2}{n^2} + 4(\lfloor c \rfloor + 1)\Bigl(e^{-n\varepsilon h\bigl((1+c_3\sqrt{\log n/n})^{1/2(\lfloor c \rfloor+1)}\bigr)} \\
&\quad \quad \quad \quad \quad + e^{-n\varepsilon h\bigl((1+c_3\sqrt{\log n/n})^{-1/2(\lfloor c \rfloor+1)}\bigr)}\Bigr) \\
&\quad \quad + 2 (\lfloor c \rfloor + 1)\Bigl(e^{-n\varepsilon h\bigl(2^{1/2(\lfloor c \rfloor+1)}\bigr)}+ e^{-n\varepsilon h\bigl(2^{-1/2(\lfloor c \rfloor+1)}\bigr)}\Bigr) \, .
\end{aligned}
\end{equation*}
We can choose $c_1$ such that $n^{-C_3c_1} = O(v_n)$. Studying the function $h$ we can also show that $h(2^{1/2(\lfloor c \rfloor+1)})$ and $h(2^{-1/2(\lfloor c \rfloor+1)})$ are strictly positive and if $a = 1/2(\lfloor c \rfloor + 1)$, as $n \rightarrow \infty$,
\begin{equation*}
\exp(-n\varepsilon h((1+ c_3\sqrt{\log n/n})^a)) = \exp\left(-\varepsilon c_3^2a^2\log n/2  + O((\log n)^{3/2}/\sqrt{n})\right) \, .
\end{equation*}
Similarly we can show that
\begin{equation*}
\exp(-n\varepsilon h((1+ c_3\sqrt{\log n/n})^{-a})) = \exp\left(-\varepsilon c_3^2a^2\log n/2  + O((\log n)^{3/2}/\sqrt{n})\right)
\end{equation*}
and we can choose $c_3$ such that the right-hand sides of the two previous equations are of order $O(v_n)$. This finally gives the result of the lemma.
\end{proof}

\subsection{Proof of Lemma \ref{lem:strong_approx_bootstrap_ineq}}

\begin{proof}[Proof of Lemma \ref{lem:strong_approx_bootstrap_ineq}]
To prove this lemma we carefully follow the proof of Lemma \ref{lem:strong_approx_ineq}. From Lemma \ref{lemma:rate_kerneldensity}, Lemma \ref{lemma:rate_kerneldensityderivative} and Lemma \ref{lem:unif_strong_rate_quantile} we know that there exist constants $C_1$ and $C_2$ both strictly positive such that for $\mathbb{P}$-almost all $\mathcal{D}_n$, for $n$ large enough, $\sup_{t\in \mathbb{R}} \vert \widetilde f_s(t) - f_s(t) \vert \leq C_1w_n$, $\sup_{\alpha \in [\varepsilon, 1-\varepsilon]} \vert \widetilde{F_s}^{\dag}(\alpha) - F_s^{\dag}(\alpha) \vert \leq C_2w_n$ and $\sup_{t\in \mathbb{R}} \vert \widetilde{f_s}'(t) - f_s'(t) \vert$ converges to $0$ as $n$ tends to $\infty$. Let $\mathcal{D}_n$ be such that these assertions are fulfilled. We denote by $q_n^*$ the quantile process given for all $\alpha' \in (0, 1)$ by
\begin{equation*}
q_n^*(\alpha') = \widetilde f_s(\widetilde F^{\dag}_s(\alpha')) \sqrt{n}\left((F_s^{Boot})^{\dag}(\alpha') - \widetilde F^{\dag}_s(\alpha')\right)
\end{equation*}
where $(F_s^{Boot})^{\dag}$ is defined on $(0,1)$ as the generalized inverse of the empirical cumulative distribution function $F_s^{Boot}$ which is based on a bootstrap sample $S_1^*, \dots, S_n^*$ which are \textit{i.i.d.} random variables with distribution $\widetilde F_s$.

We have for all $\alpha \in [\varepsilon, 1-\varepsilon]$,
\begin{equation*}
q_n^*(\alpha) = \widetilde f_s(\widetilde F^{\dag}_s(\alpha)) \sqrt{n}\left(\widetilde F_s^{\dag}(\widetilde U^{\dag}(\alpha)) - \widetilde F^{\dag}_s(\alpha)\right)
\end{equation*}
where $\widetilde U^{\dag}(\alpha) = \widetilde F_s((F_s^{Boot})^{\dag}(\alpha))$ for all $\alpha \in (0, 1)$. Let $u^*_n$ be the associated uniform quantile process defined for all $\alpha \in (0,1)$ by $u_n^*(\alpha) = \sqrt{n}\left(\widetilde U^{\dag}(\alpha) - \alpha\right)$.

Let $\varepsilon' \in (0,1)$. We first prove the differentiability of $\widetilde F_s^{\dag}$ on $[\varepsilon', 1-\varepsilon']$ for $n$ large enough. Given assumption \ref{as:kernelderivative_gine} the kernel $K$ is differentiable on $\mathbb{R}$. Therefore $\widetilde f_s$ is differentiable on $\mathbb{R}$. For $n$ large enough and for all $\alpha \in [\varepsilon', 1-\varepsilon']$, $\widetilde{f_s}(\widetilde F_s^{\dag}(\alpha)) = \widetilde{f_s}(\widetilde{\alpha}_s^{-1}(1-\alpha)) > 0$. Indeed, we can write
\begin{equation*}
\inf_{\alpha \in [\varepsilon', 1-\varepsilon']} \widetilde{f_s}(\widetilde{\alpha_s}^{-1}(\alpha)) \geq \inf_{t \in [\widetilde{\alpha_s}^{-1}(1-\varepsilon'),\widetilde{\alpha_s}^{-1}(\varepsilon')]} \widetilde{f_s}(t) \, .
\end{equation*}
As $\sup_{\alpha \in [\varepsilon', 1-\varepsilon']} \vert \widetilde{F_s}^{\dag}(\alpha) - F_s^{\dag}(\alpha) \vert$ converges to $0$, $\widetilde \alpha_s^{-1}(\varepsilon')$ and $\widetilde \alpha_s^{-1}(1-\varepsilon')$ converge respectively towards $\alpha_s^{-1}(\varepsilon')$ and $\alpha_s^{-1}(1-\varepsilon')$. Therefore there exists a constant $c_1 > 0$, depending only on $\alpha_s^{-1}(1-\varepsilon')$ and $\alpha_s^{-1}(\varepsilon')$, such that for $n$ large enough, $\widetilde{\alpha_s}^{-1}(1-\varepsilon') \geq \alpha_s^{-1}(1-\varepsilon') - c_1 > a$ and $\widetilde{\alpha_s}^{-1}(\varepsilon') < \alpha_s^{-1}(\varepsilon') + c_1 < b$, \textit{i.e.}, $[\widetilde{\alpha_s}^{-1}(1-\varepsilon'),\widetilde{\alpha_s}^{-1}(\varepsilon')] \subset [\alpha_s^{-1}(1-\varepsilon') - c_1, \alpha_s^{-1}(\varepsilon') + c_1]$. Therefore for $n$ large enough,
\begin{equation*}
\inf_{t \in [\widetilde{\alpha_s}^{-1}(1-\varepsilon'),\widetilde{\alpha_s}^{-1}(\varepsilon')]} \widetilde{f_s}(t) \geq \inf_{t \in [\alpha_s^{-1}(1-\varepsilon')-c_1,\alpha_s^{-1}(\varepsilon')+c_1]} \widetilde{f_s}(t)
\end{equation*}
and as $\widetilde f_s$ converges uniformly towards $f_s$, there exists a constant $c_2 > 0$ such that for $n$ large enough,
\begin{equation*}
\inf_{t \in [\alpha_s^{-1}(1-\varepsilon')-c_1,\alpha_s^{-1}(\varepsilon')+c_1]} \widetilde{f_s}(t) \geq \inf_{t \in [\alpha_s^{-1}(1-\varepsilon')-c_1,\alpha_s^{-1}(\varepsilon')+c_1]} f_s(t) - c_2 > 0 \, .
\end{equation*}
Note that $c_2$ can be chosen such that the strict positivity holds because the infimum of $f_s$ on $[\alpha_s^{-1}(1-\varepsilon')-c_1,\alpha_s^{-1}(\varepsilon')+c_1]$ is attained and thus strictly positive (assumption \ref{as:density_regularity_positive}).
Eventually, combining the three previous inequalities, for $n$ large enough we have: for all $\alpha \in [\varepsilon', 1-\varepsilon']$, $\widetilde f_s(\widetilde \alpha_s^{-1}(\alpha)) > 0$.

Therefore, $\widetilde F_s$ is differentiable on $[\widetilde \alpha_s^{-1}(1-\varepsilon'), \widetilde \alpha_s^{-1}(\varepsilon')]$ with strictly positive derivative. This gives that $\widetilde F_s^{\dag}: [\varepsilon', 1-\varepsilon'] \rightarrow [\widetilde\alpha_s^{-1}(1-\varepsilon'), \widetilde\alpha_s^{-1}(\varepsilon')]$ is equal to the ordinary inverse of $\widetilde F_s$ on $[\varepsilon', 1-\varepsilon']$ and $\widetilde F_s^{\dag}$ is differentiable on $[\varepsilon', 1-\varepsilon']$. \newline

We now need to invoke the same argument as in the beginning of the proof of Lemma \ref{lem:tail_condition_boot}. This argument says that there exists a constant $c_3$ such that for $n$ large enough, if $\mathcal{Z}_0$ denotes the event $\mathcal{Z}_0 = \{ [\widetilde U^{\dag}(\varepsilon), \widetilde U^{\dag}(1-\varepsilon)] \not\subset [\varepsilon - c_3, 1-\varepsilon + c_3]\}$ then $\mathbb{P}^*(\mathcal{Z}_0) \leq 2/n^2$ (refer to the proof of Lemma \ref{lem:tail_condition_boot} for details). Let $\omega \in \overline{\mathcal{Z}}_0$. Assuming $n$ large enough, for all $\alpha \in [\varepsilon, 1-\varepsilon]$, $\widetilde U^{\dag}(\alpha) \in [\varepsilon - c_3, 1-\varepsilon + c_3]$. Furthermore for $n$ large enough, we proved that $\widetilde F_s^{\dag}$ is differentiable on $[\varepsilon', 1-\varepsilon']$ for all $\varepsilon' \in (0,1)$. Therefore we can apply the mean value theorem to write for all $\alpha \in [\varepsilon, 1-\varepsilon]$,
\begin{equation*}
q_n^*(\alpha) = \frac{\widetilde f_s(\widetilde F_s^{\dag}(\alpha))}{\widetilde f_s(\widetilde F_s^{\dag}(\xi))} \sqrt{n}(\widetilde U^{\dag}(\alpha) - \alpha) = \frac{\widetilde f_s(\widetilde F_s^{\dag}(\alpha))}{\widetilde f_s(\widetilde F_s^{\dag}(\xi))} u^*_n(\alpha)
\end{equation*}
where $\xi$ is between $\widetilde U^{\dag}(\alpha')$ and $\alpha'$.

Applying Theorem 1 in \citep{CR78} conditionally on $\mathcal{D}_n$, there exists a sequence of Brownian Bridges $\{B^{*1}_n(\alpha), \alpha \in [0,1]\}_{n \geq 1}$ such that for all $z_1$ and for all $n$,
\begin{equation}
\label{eq:boot_strong_approx_ineq_unif}
\mathbb{P}^*\Bigl(\sup_{\alpha \in [0,1]}\vert u_n^*(\alpha) - B_n^{*1}(\alpha) \vert > \frac{C_3\log n + z_1}{\sqrt{n}}\Bigr) \leq C_4\exp(-C_5z_1)
\end{equation}
where $C_3, C_4$ and $C_5$ are positive absolute constants. Let $B_n^{*}(\alpha) = B_n^{*1}(1-\alpha)$ for all $\alpha \in [0, 1]$. Reasoning similarly as in the proof of Lemma \ref{lem:strong_approx_ineq}, we have for all $\alpha' \in [\varepsilon, 1-\varepsilon]$,
\begin{equation}
\label{eq:boot_decom_rn_Zn}
\begin{aligned}
\vert r_n^*(\alpha') - Z_n^*(\alpha') \vert &\leq \frac{\vert \lambda_s'(\widetilde F_s^{\dag}(\alpha'))\vert}{\widetilde f_s(\widetilde F_s^{\dag}(\alpha'))} \vert q_n^*(\alpha') - B_n^{*1}(\alpha')\vert \\
& \phantom{\leq}+ \sqrt{n}\frac{\vert \lambda_s''(\xi_1) \vert }{2}\vert (F_s^{Boot})^{\dag}(\alpha') - \widetilde F_s^{\dag}(\alpha')\vert^2
\end{aligned}
\end{equation}
where $\xi_1$ is between $(F_s^{Boot})^{\dag}(\alpha')$ and $\widetilde F_s^{\dag}(\alpha')$.

Still following the proof of Lemma \ref{lem:strong_approx_ineq} we need the following results. Conditionally on $\mathcal{D}_n$, the DKW inequality gives for all $z_2 > 0$ and for all $n$,
\begin{equation}
\label{eq:boot_dkw_unif_quantile}
\mathbb{P}^*(\sup_{\alpha \in [0,1]} \vert u_n^*(\alpha) \vert \geq z_2) \leq 2\exp(-2z_2^2) \, .
\end{equation}
Using the result of Lemma \ref{lem:tail_condition_boot} there exists a constant $c^* > 0$ such that we have, conditionally on $\mathcal{D}_n$, for all $z_4 \geq 1$ and for $n$ large enough,
\begin{equation}
\label{eq:boot_ratio_taylor_ineq}
\begin{aligned}
\mathbb{P}^*&\left(\sup_{\alpha \in [\varepsilon, 1-\varepsilon]} \left\vert \frac{\widetilde f_s(\widetilde F_s^{\dag}(\alpha))}{\widetilde f_s(\widetilde F_s^{\dag}(\xi))} \right\vert > z_4\right) \\
& \quad \quad \leq \frac{2}{n^2} + 2 (\lfloor c^* \rfloor + 1)\Bigl(e^{-n\varepsilon h\bigl(z_4^{1/2(\lfloor c^* \rfloor+1)}\bigr)} + e^{-n\varepsilon h\bigl(z_4^{-1/2(\lfloor c^* \rfloor+1)}\bigr)}\Bigr)
\end{aligned}
\end{equation}
where we recall that for all $z > 0$, $h(z) = z + \log(1/z) -1$.
Following the proof of Lemma \ref{lem:tail_condition_boot} and with calculations similar as the ones used in the proof of Theorem 1.4.3 in \citep{Csorgo1983} it can also be shown that conditionally on $\mathcal{D}_n$ we have for all $z_3 > 0$ and for $n$ large enough,
\begin{equation}
\label{eq:boot_ratio_taylor_ineq_1}
\begin{aligned}
\mathbb{P}^*&\left(\sup_{\alpha \in [\varepsilon, 1-\varepsilon]} \left\vert \frac{\widetilde f_s(\widetilde F_s^{\dag}(\alpha))}{\widetilde f_s(\widetilde F_s^{\dag}(\xi))} - 1 \right\vert > z_3\right) \\
&\leq 4(\lfloor c^* \rfloor + 1)\Bigl(e^{-n\varepsilon h\bigl((1+z_3)^{1/2(\lfloor c^* \rfloor+1)}\bigr)} + e^{-n\varepsilon h\bigl((1+z_3)^{-1/2(\lfloor c^* \rfloor+1)}\bigr)}\Bigr) + \frac{2}{n^2} \, .
\end{aligned}
\end{equation}

Now as for all $\alpha' \in [\varepsilon, 1-\varepsilon]$, $\widetilde F_s^{\dag}(\alpha')$ is in the support of $\widetilde f_s$ which has shown in the proof of Lemma \ref{lem:unif_strong_rate_distribution}, is included in an interval of the form $[-c_4, c_4 + \Vert s \Vert_{\infty}]$ where $c_4$ is a strictly positive constant that depends only on the kernel $K$. As $\lambda_s'$ is continuous, $\vert \lambda_s'(\widetilde F_s^{\dag}(\alpha'))\vert \leq \sup_{[-c_4, c_4 + \Vert s \Vert_{\infty}]} \vert \lambda_s' \vert < \infty$. Furthermore, as proved above, there exist constants $c_5$ and $c_6$ both strictly positive such that $[F_s^{\dag}(\varepsilon)-c_5,F_s^{\dag}(1-\varepsilon)+c_5] \subset (a, b)$ and such that for $n$ large enough,
\begin{equation*}
\inf_{\alpha' \in [\varepsilon, 1-\varepsilon]} \widetilde{f_s}(\widetilde F_s^{\dag}(\alpha')) \geq \inf_{t \in [F_s^{\dag}(\varepsilon)-c_5,F_s^{\dag}(1-\varepsilon)+c_5]} f_s(t) - c_6 > 0 \, .
\end{equation*}
For all $\alpha \in [\varepsilon, 1-\varepsilon]$, $(F_s^{Boot})^{\dag}(\alpha)$ and $\widetilde F_s^{\dag}(\alpha)$ are in the support of $\widetilde f_s$ which is included in $[-c_4, c_4 + \Vert s \Vert_{\infty}]$. Therefore $\xi_1 \in [-c_4, c_4 + \Vert s \Vert_{\infty}]$ and as $\lambda_s''$ is continuous, $\vert \lambda_s''(\xi_1) \vert \leq \sup_{[-c_4, c_4 + \Vert s \Vert_{\infty}]} \vert \lambda_s''\vert < +\infty$. Thus following the steps of the proof of Lemma \ref{lem:strong_approx_ineq} we have for $n$ large enough,
\begin{equation*}
\begin{aligned}
\mathbb{P}^*&\left(\sup_{\alpha \in [\varepsilon, 1-\varepsilon]} \vert r_n^*(\alpha) - Z_n^*(\alpha) \vert > Cv_n\right) \\
&\leq \frac{C_4}{n^{C_5c_7}} + \frac{8}{n^2} + 4(\lfloor c^* \rfloor + 1)\Bigl(\exp\bigl(-n\varepsilon h\bigl((1+c_8\sqrt{\log n/n})^{1/2(\lfloor c^* \rfloor+1)}\bigr)\bigr) \\
&\quad \quad + \exp\bigl(-n\varepsilon h\bigl((1+c_8\sqrt{\log n/n})^{-1/2(\lfloor c^* \rfloor+1)}\bigr)\bigr)\Bigr) \\
&\quad \quad + 2 (\lfloor c^* \rfloor + 1)\Bigl(\exp\bigl(-n\varepsilon h\bigl(2^{1/2(\lfloor c^* \rfloor+1)}\bigr)\bigr) + \exp\bigl(-n\varepsilon h\bigl(2^{-1/2(\lfloor c^* \rfloor+1)}\bigr)\bigr)\Bigr) \, .
\end{aligned}
\end{equation*}
where
\begin{equation*}
\begin{aligned}
C = \frac{\sup_{[-c_4, c_4 + \Vert s \Vert_{\infty}]} \vert \lambda_s' \vert}{\inf_{[F_s^{\dag}(\varepsilon)-c_5,F_s^{\dag}(1-\varepsilon)+c_5]} f_s - c_6}&(C_4+c_7+c_8) \\
&+ \frac{2\sup_{[-c_4, c_4 + \Vert s \Vert_{\infty}]} \vert \lambda_s''\vert}{(\inf_{[F_s^{\dag}(\varepsilon)-c_5,F_s^{\dag}(1-\varepsilon)+c_5]} f_s - c_6)^2} \, .
\end{aligned}
\end{equation*}
Proceeding analogously as in Lemma \ref{lem:strong_approx_ineq}, we can choose $c_7$ and $c_8$ such that the right-hand side of the previous inequality is of order $O(v_n)$.
This finally gives the result of the lemma.
\end{proof}

\subsection{Proof of Lemma \ref{lemma:gine_guillou}}

\begin{proof}[Proof of Lemma \ref{lemma:gine_guillou}]
Notice first that thanks to Lemma \ref{lemma:y_s_positive} $y_s(\alpha) \neq 0$ for all $\alpha \in (0,1)$. Let $\alpha \in [\varepsilon, 1-\varepsilon]$,
\begin{equation*}
\frac{\widetilde y_s(\alpha)}{y_s(\alpha)} -1 = \frac{\widetilde y_s(\alpha) - y_s(\alpha)}{y_s(\alpha)} \, .
\end{equation*}

As the function $\vert 1/y_s(\cdot) \vert$ is continuous and therefore bounded over $[\varepsilon,1-\varepsilon]$, we only need to control $\sup_{\alpha \in [\varepsilon, 1-\varepsilon]}\vert \widetilde y_s(\alpha) - y_s(\alpha) \vert$ which can be bounded by the following decomposition
\begin{equation}
\label{eq:dec_gine_guillou}
\begin{aligned}
\vert y_s(\alpha) - \widetilde y_s(\alpha) \vert
&\leq \left\vert \frac{\lambda_s'(\alpha_s^{-1}(\alpha))}{f_s(\alpha_s^{-1}(\alpha))} - \frac{\lambda_s'(\widetilde \alpha_s^{-1}(\alpha))}{f_s(\widetilde \alpha_s^{-1}(\alpha))} \right\vert + \left\vert \frac{\lambda_s'(\widetilde\alpha_s^{-1}(\alpha))}{f_s(\widetilde\alpha_s^{-1}(\alpha))} - \frac{\lambda_s'(\widetilde \alpha_s^{-1}(\alpha)}{\widetilde f_s(\widetilde \alpha_s^{-1}(\alpha))} \right\vert \, .
\end{aligned}
\end{equation}

As $\lambda_s$ is of class $\mathcal{C}^2$ and $f_s$ is of class $\mathcal{C}^1$ and strictly positive on $(a,b)$, $\lambda_s^{\prime}/f_s$ is of class $\mathcal{C}^1$ on $(a,b)$. Thus using the mean value theorem, we can write, for all $\alpha \in [\varepsilon, 1-\varepsilon]$,
\begin{equation*}
\left\vert \frac{\lambda_s'(\alpha_s^{-1}(\alpha))}{f_s(\alpha_s^{-1}(\alpha))} - \frac{\lambda_s'(\widetilde \alpha_s^{-1}(\alpha))}{f_s(\widetilde \alpha_s^{-1}(\alpha))} \right\vert \leq \left\vert\left(\frac{\lambda'_s}{f_s}\right)'(\xi) \right\vert \cdot \vert \alpha_s^{-1}(\alpha)- \widetilde{\alpha}_s^{-1}(\alpha)\vert
\end{equation*}
where $\xi$ is between $\alpha_s^{-1}(\alpha)$ and $\widetilde \alpha_s^{-1}(\alpha)$. From Lemma \ref{lem:unif_strong_rate_quantile} we know that there exists a constant $c_1 > 0$ such that almost surely, for $n$ large enough, $\widetilde \alpha_s^{-1}(\alpha) \in [\alpha_s^{-1}(1-\varepsilon) - c_1, \alpha_s^{-1}(\varepsilon)+ c_1] \subset (a,b)$. Therefore, as $(\lambda_s'/f_s)'$ is continuous, almost surely and for $n$ large enough,
\begin{equation*}
\left\vert\left(\frac{\lambda'_s}{f_s}\right)'(\xi) \right\vert \leq \sup_{t \in [\alpha_s^{-1}(1-\varepsilon) - c_1, \alpha_s^{-1}(\varepsilon)+ c_1]} \left\vert\left(\frac{\lambda'_s}{f_s}\right)'(t) \right\vert < +\infty \, .
\end{equation*}
Furthermore,
\begin{equation*}
\sup_{\alpha \in [\varepsilon, 1-\varepsilon]}\vert \alpha_s^{-1}(\alpha)- \widetilde{\alpha}_s^{-1}(\alpha)\vert = \sup_{\alpha' \in [\varepsilon, 1-\varepsilon]} \vert F_s^{\dag}(\alpha') - \widetilde F_s^{\dag}(\alpha')\vert \, .
\end{equation*}
Therefore, using Lemma \ref{lem:unif_strong_rate_quantile}, there exists a constant $C_1 > 0$ such that we almost surely have, for $n$ large enough,
\begin{equation}
\label{eq:first_term_variance_lemma}
\left\vert \frac{\lambda_s'(\alpha_s^{-1}(\alpha))}{f_s(\alpha_s^{-1}(\alpha))} - \frac{\lambda_s'(\widetilde \alpha_s^{-1}(\alpha))}{f_s(\widetilde \alpha_s^{-1}(\alpha))} \right\vert \leq C_1w_n \, .
\end{equation}

For the second term of \eqref{eq:dec_gine_guillou} we write
\begin{equation*}
\left\vert \frac{\lambda_s'(\widetilde\alpha_s^{-1}(\alpha))}{f_s(\widetilde\alpha_s^{-1}(\alpha))} - \frac{\lambda_s'(\widetilde \alpha_s^{-1}(\alpha))}{\widetilde f_s(\widetilde \alpha_s^{-1}(\alpha))} \right\vert
= \vert \lambda^{\prime}_s(\widetilde \alpha_s^{-1}(\alpha))\vert \left\vert \frac{1}{f_s(\widetilde \alpha_s^{-1}(\alpha))} - \frac{1}{\widetilde f_s(\widetilde \alpha_s^{-1}(\alpha))}\right\vert
\end{equation*}
and as shown in the proof of Lemma \ref{lem:unif_strong_rate_distribution}, $\widetilde F_s$ has a finite support included in an interval of the form $[-c_2, c_2 + \Vert s \Vert_{\infty}]$. Therefore $\widetilde \alpha_s^{-1}(\alpha) \in [-c_2, c_2 + \Vert s \Vert_{\infty}]$ and as $\lambda_s'$ is continuous, it is bounded on this interval. Therefore there exists $C_2 > 0$ such that for all $\alpha \in [\varepsilon, 1-\varepsilon]$, $\vert \lambda^{\prime}_s(\widetilde \alpha_s^{-1}(\alpha))\vert \leq C_2$
and
\begin{equation}
\left\vert \frac{\lambda_s'(\widetilde\alpha_s^{-1}(\alpha))}{f_s(\widetilde\alpha_s^{-1}(\alpha))} - \frac{\lambda_s'(\widetilde \alpha_s^{-1}(\alpha))}{\widetilde f_s(\widetilde \alpha_s^{-1}(\alpha))} \right\vert
\leq C_2\frac{\sup_{x \in \mathbb{R}}\vert \widetilde f_s(x) - f_s(x) \vert}{\widetilde f_s(\widetilde \alpha_s^{-1}(\alpha))f_s(\widetilde\alpha_s^{-1}(\alpha))} \, .
\end{equation}
We can proceed as in the proof of Lemma \ref{lem:strong_approx_bootstrap_ineq} to show that there exists constants $c_3$ and $c_4$ both strictly positive such that $[\alpha_s^{-1}(1-\varepsilon) - c_3, \alpha_s^{-1}(\varepsilon) + c_3] \subset (a,b)$ and such that almost surely, for $n$ large enough, for all $\alpha \in [\varepsilon, 1-\varepsilon]$,
\begin{equation*}
\widetilde f_s(\widetilde \alpha_s^{-1}(\alpha)) \geq \inf_{t \in [\alpha_s^{-1}(1-\varepsilon) - c_3, \alpha_s^{-1}(\varepsilon) + c_3]} f_s(t) - c_4 > 0 \, .
\end{equation*}
We also almost surely have, for $n$ large enough and for all $\alpha \in [\varepsilon, 1-\varepsilon]$,
\begin{equation*}
\begin{aligned}
f_s(\widetilde \alpha_s^{-1}(\alpha)) &\geq \inf_{\alpha' \in [\varepsilon,1-\varepsilon]} f_s(\widetilde \alpha_s^{-1}(\alpha')) = \inf_{t \in [\widetilde \alpha_s^{-1}(1-\varepsilon),\widetilde \alpha_s^{-1}(\varepsilon)]} f_s(t) \\
&\geq \inf_{t \in [\alpha_s^{-1}(1-\varepsilon) - c_3,\alpha_s^{-1}(\varepsilon) + c_3]} f_s(t) > 0 \, .
\end{aligned}
\end{equation*}
Let $A_{\varepsilon}=[\alpha_s^{-1}(1-\varepsilon) - c_3,\alpha_s^{-1}(\varepsilon) + c_3]$. We almost surely have, for $n$ large enough,
\begin{equation*}
\Bigl\vert \frac{\lambda_s'(\widetilde\alpha_s^{-1}(\alpha))}{f_s(\widetilde\alpha_s^{-1}(\alpha))} - \frac{\lambda_s'(\widetilde \alpha_s^{-1}(\alpha))}{\widetilde f_s(\widetilde \alpha_s^{-1}(\alpha))} \Bigr\vert
\leq C_2\frac{\sup_{t \in \mathbb{R}}\vert \widetilde f_s(t) - f_s(t) \vert}{\inf_{t \in A_{\varepsilon}} f_s(t)\cdot (\inf_{t \in A_{\varepsilon}} f_s(t) - c_4)} \, .
\end{equation*}
Thanks to Lemma \ref{lemma:rate_kerneldensity} and combining the last inequality with \eqref{eq:first_term_variance_lemma}, we obtain the result of the lemma.
\end{proof}

\subsection{Other technical results}

\begin{lemma} \label{lemma:y_s_positive}
For all $\alpha \in (0, 1)$, $y_s(\alpha) \neq 0$.
\end{lemma}

\begin{proof}[Proof of Lemma \ref{lemma:y_s_positive}]
First notice that as $\lambda_s$ is of class $\mathcal{C}^2$ and $F_s'= f_s > 0$ on $(a, b)$, $\mv_s = \lambda_s \circ \alpha_s^{-1}$ is differentiable on $(0,1)$ and for all $\alpha \in (0, 1)$,
\begin{equation}
\label{eq:derivative_mv_s_y_s}
\mv_s'(\alpha) = -\frac{\lambda_s(\alpha_s^{-1}(\alpha))}{f_s(\alpha_s^{-1}(\alpha))} = -y_s(\alpha) \, .
\end{equation}
Now, following the steps of the proof of the formula of the derivative of $\mv^*$ we also have, for all $\alpha \in (0,1)$ and $h > 0$,
\begin{equation*}
\begin{aligned}
\frac{\mv_s(\alpha + h) - \mv_s(\alpha)}{h} &=\frac{1}{h}(\lambda\{s \geq F_s^{-1}(1-(\alpha + h)) \} - \lambda\{s \geq F_s^{-1}(1-\alpha) \}) \\
&= \frac{1}{h}\lambda\{F_s^{-1}(1-(\alpha + h)) \leq s \leq F_s^{-1}(1-\alpha)\} \\
&= \frac{1}{h}\lambda\{1-(\alpha + h) \leq F_s\circ s \leq 1-\alpha\} \\
&= \frac{1}{h}\int_{\X} \mathbb{I}\{x, 1-(\alpha + h) \leq F_s(s(x)) \leq 1-\alpha\} \frac{f(x)}{f(x)}dx \\
&\geq \frac{1}{h\Vert f \Vert_{\infty}} \mathbb{P}(F_s(s(X)) \in [1-(\alpha + h), 1-\alpha]) \\
&= \frac{1}{\Vert f \Vert_{\infty}} \, .
\end{aligned}
\end{equation*}
where we used the fact that the distribution of $F_s(s(X))$ is the uniform distribution on $(0,1)$. Similarly,
\begin{equation*}
\forall \alpha \in (0,1), \quad \frac{\mv_s(\alpha - h) - \mv_s(\alpha)}{-h} \geq \frac{1}{\Vert f \Vert_{\infty}} \, .
\end{equation*}
Therefore, for all $\alpha \in (0,1)$, $\mv_s'(\alpha) \geq 1/\Vert f \Vert_{\infty} > 0$. Using \eqref{eq:derivative_mv_s_y_s}, this implies that $y_s(\alpha) \neq 0$ for all $\alpha \in (0,1)$.
\end{proof}

\begin{lemma} \label{lemma:rate_kerneldensity} Under assumptions \ref{as:bounded_score} and \ref{as:density_bias}-\ref{as:kernel_gine} there exists a constant $C > 0$, depending only on $f_s$ and the kernel $K$ such that we almost surely have, for $n$ large enough,
\begin{equation*}
\sup_{t \in \mathbb{R}} \left\vert \widetilde f_s(t) - f_s(t) \right\vert \leq Cw_n \, .
\end{equation*}
Note that this result does not require $K$ to have a finite support.
\end{lemma}

\begin{proof}[Proof of Lemma \ref{lemma:rate_kerneldensity}]
We first bound $\sup_{t \in \mathbb{R}} \vert \widetilde f_s(t) - f_s(t) \vert$ by the sum of a stochastic term and a deterministic term:
\begin{equation*}
\sup_{t \in \mathbb{R}} \left\vert \widetilde f_s(t) - f_s(t) \right\vert \leq \sup_{t \in \mathbb{R}} \left\vert \widetilde f_s(t) - \mathbb{E}[\widetilde f_s(t)] \right\vert + \sup_{t \in \mathbb{R}} \left\vert \mathbb{E}[\widetilde f_s(t)] - f_s(t) \right\vert \, .
\end{equation*}
Under the stipulated conditions, the first term is controlled thanks to Theorem 2.3 in \citep{Gine}: we almost surely have, for $n$ large enough,
\begin{equation*}
\sup_{t \in \mathbb{R}} \left\lvert \widetilde f_s(t) - \mathbb{E}[\widetilde f_s(t)] \right\rvert \leq M^2\Vert f_s \Vert_{\infty} \Vert K \Vert_2^2 \sqrt{\frac{\log h_n^{-1}}{nh_n}}
\end{equation*}
where $M$ is a constant that depends on the VC characteristics of $K$.

It can then be shown with classical calculations that the second term is bounded by
\begin{equation*}
\sup_{t \in \mathbb{R}}\left\vert \mathbb{E}[\widetilde f_s(t)] - f_s(t) \right\vert \leq \frac{h_n^2}{2}\Vert f_s'' \Vert_{\infty}\int_{-\infty}^{+\infty}y^2K(y) \, .
\end{equation*}
Combining this last inequality with the one on the first term we obtain the result of the lemma.
\end{proof}

\begin{lemma} \label{lemma:rate_kerneldensityderivative} Under assumptions \ref{as:bounded_score}, \ref{as:density_bias}-\ref{as:kernel_bias} and \ref{as:kernelderivative_gine} there exists a constant $C > 0$, depending only on $f_s$ and the kernel $K$ such that we almost surely have, for $n$ large enough,
\begin{equation*}
\sup_{t \in \mathbb{R}} \left\vert \widetilde{f_s}'(t) - f_s'(t) \right\vert \leq C\left(\sqrt{\frac{\log h_n^{-1}}{nh_n^3}} + h_n^2\right) \, .
\end{equation*}
Therefore, as $h_n \rightarrow 0$, if assumption \ref{as:bandwidth_derivative} is also fulfilled then we almost surely have,
\begin{equation*}
\sup_{t \in \mathbb{R}} \left\vert \widetilde{f_s}'(t) - f_s'(t) \right\vert \underset{n \rightarrow \infty}{\longrightarrow} 0 \, .
\end{equation*}
Note that this result does not require $K$ to have a finite support.
\end{lemma}

\begin{proof}[Proof of Lemma \ref{lemma:rate_kerneldensityderivative}]
We proceed as in the proof of Lemma \ref{lemma:rate_kerneldensity}. We first bound $\sup_{t \in \mathbb{R}} \vert \widetilde{f_s}'(t) - f_s'(t) \vert$ by the sum of a stochastic term and a deterministic term:
\begin{equation*}
\sup_{t \in \mathbb{R}} \left\vert \widetilde{f_s}'(t) - f_s'(t) \right\vert \leq \sup_{t \in \mathbb{R}} \left\vert \widetilde{f_s}'(t) - \mathbb{E}[\widetilde{f_s}'(t)] \right\vert + \sup_{t \in \mathbb{R}} \left\vert \mathbb{E}[\widetilde{f_s}'(t)] - f_s'(t) \right\vert \, .
\end{equation*}
We first deal with the first term. Let $t \in \mathbb{R}$,
\begin{equation*}
\widetilde{f_s}'(t) - \mathbb{E}[\widetilde{f_s'}(t)] =\frac{1}{h_n}\left(\frac{1}{nh_n}\sum_{i=1}^{n}K'\left(\frac{t-s(X_i)}{h_n} \right) - \mathbb{E}\left[\frac{1}{nh_n}\sum_{i=1}^{n}K'\left(\frac{t-s(X_i)}{h_n} \right)\right]\right) \, .
\end{equation*}
Therefore $K'$ satisfying the conditions required for the result of \citep{Gine} we almost surely have, for $n$ large enough,
\begin{equation*}
\sup_{t \in \mathbb{R}}\left\vert \widetilde{f_s}'(t) - \mathbb{E}[\widetilde{f_s'}(t)] \right\vert \leq M'^2\Vert f_s \Vert_{\infty} \Vert K' \Vert_2^2 \sqrt{\frac{\log h_n^{-1}}{nh^3_n}}
\end{equation*}
where $M'$ is a constant that depends on the VC characteristics of $K'$.

We now deal with the second term as follows. Let $t \in \mathbb{R}$,
\begin{equation*}
\begin{aligned}
\mathbb{E}[\widetilde{f_s}'(t)] &= \frac{1}{nh_n^2}\sum_{i=1}^n\mathbb{E}\left[K'\left(\frac{t-s(X_i)}{h_n}\right)\right] = \frac{1}{h_n^2}\mathbb{E}\left[K'\left(\frac{t-s(X_1)}{h_n}\right)\right] \\
&= \frac{1}{h_n}\int_{-\infty}^{+\infty}\frac{1}{h_n}K'\left(\frac{t-u}{h_n}\right)f_s(u)du \, .
\end{aligned}
\end{equation*}
With an integration by parts we obtain,
\begin{equation*}
\mathbb{E}[\widetilde{f_s}'(t)] - f_s'(t) = \frac{1}{h_n}\int_{-\infty}^{+\infty}K\left(\frac{t-u}{h_n}\right)f'_s(u)du - f_s'(t)
\end{equation*}
and we can deal with this term as we dealt with $\mathbb{E}[\widetilde{f_s}(t)] - f_s(t)$ in the proof of Lemma \ref{lemma:rate_kerneldensity}, to obtain
\begin{equation*}
\sup_{t \in \mathbb{R}}\left\vert \mathbb{E}[\widetilde{f_s}'(t)] - f_s'(t) \right\vert \leq \frac{h_n^2}{2}\Vert f_s''' \Vert_{\infty}\int_{-\infty}^{+\infty}y^2K(y) \, .
\end{equation*}
Combining this last inequality with the one on the first term we obtain the result of the lemma.

\end{proof}

\begin{lemma} \label{lem:unif_strong_rate_distribution}
Under the conditions of Lemma \ref{lemma:rate_kerneldensity} there exists a constant $C > 0$ such that we almost surely have for $n$ large enough,
\begin{equation*}
\sup_{t \in \mathbb{R}} \left\vert \widetilde F_s(t) - F_s(t) \right\vert \leq Cw_n \, .
\end{equation*}
\end{lemma}

\begin{proof}[Proof of Lemma \ref{lem:unif_strong_rate_distribution}]
Let $t \in \mathbb{R}$,
\begin{equation*}
\left\vert F_s(t) - \widetilde F_s(t) \right\vert = \left\vert \int_{-\infty}^{t} f_s(u)du - \int_{-\infty}^{t} \widetilde f_s(u)du \right\vert \leq \int_{-\infty}^{t} \vert f_s(u) - \widetilde f_s(u) \vert du \, .
\end{equation*}
As $f_s$ has a support included in $[0, \Vert s \Vert_{\infty}]$, as $K$ has a finite support and as the sequence $(h_n)_{n\geq 1}$ decreases towards $0$, $\widetilde f_s$ also has a finite support included in an interval of the form $[-c_1, \Vert s \Vert_{\infty} + c_1]$ where $c_1 >0$. Thus,
\begin{equation*}
\int_{-\infty}^{t} \vert f_s(u) - \widetilde f_s(u) \vert du \leq (2c_1 + \Vert s \Vert_{\infty}) \sup_{t \in \mathbb{R}} \left\vert f_s(t) - \widetilde f_s(t) \right\vert\, .
\end{equation*}
Thanks to Lemma \ref{lemma:rate_kerneldensity}, we finally obtain the result.
\end{proof}

\begin{lemma} \label{lem:unif_strong_rate_quantile}
Let $\varepsilon \in (0,1)$. Under the conditions of Lemma \ref{lemma:rate_kerneldensity}, there exists a constant $C > 0$ such that we almost surely have for $n$ large enough,
\begin{equation*}
\sup_{y \in [\varepsilon,1-\varepsilon]} \left\vert \widetilde F_s^{\dag}(y) - F_s^{\dag}(y) \right\vert \leq Cw_n \, .
\end{equation*}
This lemma only requires $\widetilde F_s$ to be continuous.
\end{lemma}

\begin{proof}[Proof of Lemma \ref{lem:unif_strong_rate_quantile}]
First we recall that $\widetilde F_s$ being continuous, for all $y \in (0,1)$, $F_s(F_s^{\dag}(y)) = y$. Let $y \in [\varepsilon,1-\varepsilon]$,
\begin{equation*}
\begin{aligned}
\vert \widetilde F_s(\widetilde F_s^{\dag}(y)) - F_s(\widetilde F_s^{\dag}(y)) \vert &= \vert y - F_s(\widetilde F_s^{\dag}(y)) \vert = \vert F_s(F_s^{\dag}(y)) - F_s(\widetilde F_s^{\dag}(y)) \vert \\
&= \vert f_s(\xi) \vert \cdot \vert F_s^{\dag}(y) - \widetilde F_s^{\dag}(y) \vert
\end{aligned}
\end{equation*}
with $\xi$ between $F_s^{\dag}(y)$ and $\widetilde F_s^{\dag}(y)$. As $F_s^{\dag}$ and $\widetilde F_s^{\dag}$ are increasing, $F_s^{\dag}(\varepsilon) \leq F_s^{\dag}(y) \leq F_s^{\dag}(1-\varepsilon)$ and $\widetilde F_s^{\dag}(\varepsilon) \leq \widetilde F_s^{\dag}(y) \leq \widetilde F_s^{\dag}(1-\varepsilon)$. Now
\begin{equation*}
\vert \varepsilon - F_s(\widetilde F_s^{\dag}(\varepsilon))\vert = \vert \widetilde F_s(\widetilde F_s^{\dag}(\varepsilon)) - F_s(\widetilde F_s^{\dag}(\varepsilon)) \vert \leq \sup_{t \in \mathbb{R}} \vert \widetilde F_s(t) - F_s(t) \vert \, .
\end{equation*}
Therefore from Lemma \ref{lem:unif_strong_rate_distribution} as $w_n \rightarrow 0$ when $n$ tends to $+\infty$, $F_s(\widetilde F_s^{\dag}(\varepsilon))$ converges almost surely to $\varepsilon$. And as $F_s^{\dag}$ is continuous, we almost surely have, $F_s^{\dag}(F_s(\widetilde F_s^{\dag}(\varepsilon))) \underset{n \rightarrow \infty}{\longrightarrow} F_s^{\dag}(\varepsilon)$, i.e., $\widetilde F_s^{\dag}(\varepsilon) \underset{n \rightarrow \infty}{\longrightarrow} F_s^{\dag}(\varepsilon)$. Similarly, we almost surely have $\widetilde F_s^{\dag}(1-\varepsilon) \underset{n \rightarrow \infty}{\longrightarrow} F_s^{\dag}(1-\varepsilon)$. Therefore there exists a constant $c_1 > 0$ such that we almost surely have, for $n$ large enough, $\widetilde F_s^{\dag}(y) \in [F_s^{\dag}(\varepsilon) - c_1, F_s^{\dag}(1-\varepsilon) + c_1] \subset (a,b)$. This implies that $\xi$ is between $F_s^{\dag}(\varepsilon) - c_1 > a$ and $F_s^{\dag}(1-\varepsilon) + c_1 < b$ and thus,
\begin{equation*}
f_s(\xi) \geq \inf_{t \in [F_s^{\dag}(\varepsilon) - c_1, F_s^{\dag}(1-\varepsilon) + c_1]} f_s(t) \, .
\end{equation*}
We thus almost surely have, for $n$ large enough,
\begin{equation*}
\vert \widetilde F_s(\widetilde F_s^{\dag}(y)) - F_s(\widetilde F_s^{\dag}(y)) \vert \geq \vert F_s^{\dag}(y) - \widetilde F_s^{\dag}(y) \vert \inf_{t \in [F_s^{\dag}(\varepsilon) - c_1, F_s^{\dag}(1-\varepsilon) + c_1]} f_s(t)
\end{equation*}
Now as $f_s$ is strictly positive on $(a, b)$ and continuous, the infimum of the right-hand side is strictly positive. We can thus write almost surely and for $n$ large enough,
\begin{equation*}
\begin{aligned}
\vert F_s^{\dag}(y) - \widetilde F_s^{\dag}(y) \vert &\leq \frac{1}{\inf_{t \in [F_s^{\dag}(\varepsilon) - c_1, F_s^{\dag}(1-\varepsilon) + c_1]} f_s(t)} \vert \widetilde F_s(\widetilde F_s^{\dag}(y)) - F_s(\widetilde F_s^{\dag}(y)) \vert \\
&\leq \frac{C}{\inf_{t \in [F_s^{\dag}(\varepsilon) - c_1, F_s^{\dag}(1-\varepsilon) + c_1]} f_s(t)}w_n
\end{aligned}
\end{equation*}
where we used Lemma \ref{lem:unif_strong_rate_distribution} for the last inequality. This concludes the proof.
\end{proof}

\begin{lemma} \label{lem:unif_cv_ratio}
Let $a', b' \in \mathbb{R}$ with $a < a' < b' < b$. Under assumptions \ref{as:bounded_score}-\ref{as:density_regularity_positive} and \ref{as:density_bias}-\ref{as:bandwidth_derivative},
\begin{equation*}
\sup_{t \in [a', b']} \left\vert \widetilde F_s(t)(1- \widetilde F_s(t))\frac{\vert \widetilde{f_s}'(t)\vert}{\widetilde{f_s}^2(t)} - F_s(t)(1- F_s(t))\frac{\vert f_s'(t)\vert}{f_s^2(t)}\right\vert \underset{n \rightarrow +\infty}{\longrightarrow} 0
\end{equation*}
\end{lemma}

\begin{proof}[Proof of Lemma \ref{lem:unif_cv_ratio}]
We have for all $t \in [a', b']$,
\begin{equation*}
\begin{aligned}
\Bigl\vert \widetilde F_s(t)&(1- \widetilde F_s(t))\frac{\vert \widetilde{f_s}'(t)\vert}{\widetilde{f_s}^2(t)} - F_s(t)(1- F_s(t))\frac{\vert f_s'(t)\vert}{f_s^2(t)}\Bigr\vert \\
&\leq \vert \widetilde F_s(t)(1- \widetilde F_s(t))\vert \left\vert \frac{\vert\widetilde{f_s}'(t)\vert}{\widetilde{f_s}^2(t)} - \frac{\vert f_s'(t)\vert}{f_s^2(t)}\right\vert \\
&\phantom{\leq} + \left\vert\frac{\widetilde{f_s}'(t)}{\widetilde{f_s}^2(t)} \right\vert \vert \widetilde F_s(t)(1- \widetilde F_s(t)) - F_s(t)(1- F_s(t)) \vert \\
&\leq \left\vert \frac{\vert\widetilde{f_s}'(t)\vert}{\widetilde{f_s}^2(t)} - \frac{\vert f_s'(t)\vert}{f_s^2(t)}\right\vert + \left\vert\frac{\widetilde{f_s}'(t)}{\widetilde{f_s}^2(t)} \right\vert \vert \widetilde F_s(t)(1- \widetilde F_s(t)) - F_s(t)(1- F_s(t)) \vert
\end{aligned}
\end{equation*}
where the third inequality holds because for all $t$, $\widetilde F_s(t) \in [0,1]$. The result can then be derived thanks to Lemma \ref{lemma:rate_kerneldensity}, Lemma \ref{lemma:rate_kerneldensityderivative} and Lemma \ref{lem:unif_strong_rate_distribution}.
\end{proof}

\begin{lemma} \label{lem:tail_condition_boot}
Let $\mathcal{D}_n$ be such that $\sup_{t\in \mathbb{R}} \vert \widetilde f_s(t) - f_s(t) \vert \leq C_1w_n$, $\sup_{\alpha \in [\varepsilon, 1-\varepsilon]} \vert \widetilde{F_s}^{\dag}(\alpha) - F_s^{\dag}(\alpha) \vert \leq C_2w_n$ and $\sup_{t\in \mathbb{R}} \vert \widetilde{f_s}'(t) - f_s'(t) \vert$ converges to $0$ as $n$ tends to $\infty$. There exists a constant $c^* > 0$ such that at $\mathcal{D}_n$ fixed, for all $z_4 \geq 1$ and for $n$ large enough,
\begin{equation*}
\begin{aligned}
\mathbb{P}^*\left(\sup_{\alpha \in [\varepsilon, 1-\varepsilon]} \left\vert \frac{\widetilde f_s(\widetilde F_s^{\dag}(\alpha))}{\widetilde f_s(\widetilde F_s^{\dag}(\xi))} \right\vert > z_4\right) \leq \frac{2}{n^2} + 2 & (\lfloor c^* \rfloor + 1)\Bigl(\exp\bigl(-n\varepsilon h\bigl(z_4^{1/2(\lfloor c^* \rfloor+1)}\bigr)\bigr) \\
&+ \exp\bigl(-n\varepsilon h\bigl(z_4^{-1/2(\lfloor c^* \rfloor+1)}\bigr)\bigr)\Bigr) \, .
\end{aligned}
\end{equation*}
where $\xi$ is between $\widetilde U^{\dag}(\alpha)$ and $\alpha$ with $\alpha \in [\varepsilon, 1-\varepsilon]$.
\end{lemma}

\begin{proof}[Proof of Lemma \ref{lem:tail_condition_boot}]
From the DKW inequality (see proof of Lemma \ref{lem:dkw_taylor_remainder}), we have for all $n$,
\begin{equation*}
\mathbb{P}^*\left(\sup_{\alpha \in [0,1]} \vert \widetilde U^{\dag}(\alpha) - \alpha \vert \geq \sqrt{\frac{\log n}{n}}\right) \leq \frac{2}{n^2} \, .
\end{equation*}
Therefore, if $\omega' \in \{\sup_{\alpha \in [0,1]} \vert \widetilde U^{\dag}(\alpha) - \alpha \vert < \sqrt{\log n/n} \}$, we have,
\begin{equation*}
\vert \widetilde U^{\dag}(\varepsilon) - \varepsilon \vert \leq \sup_{\alpha \in [0,1]} \vert \widetilde U^{\dag}(\alpha) - \alpha \vert < \sqrt{\log n/n}
\end{equation*}
and similarly,
\begin{equation*}
\vert \widetilde U^{\dag}(1-\varepsilon) - (1-\varepsilon) \vert \leq \sup_{\alpha \in [0,1]} \vert \widetilde U^{\dag}(\alpha) - \alpha \vert < \sqrt{\log n/n} \, .
\end{equation*}
Thus there exists a constant $c_1$ such that for $n$ large enough (independent of $\omega'$), $\widetilde U^{\dag}(\varepsilon) \geq \varepsilon - c_1 > 0$ and $\widetilde U^{\dag}(1-\varepsilon) \leq 1-\varepsilon + c_1 < 1$. Now as $\alpha \in [\varepsilon,1-\varepsilon]$, this implies that for $n$ large enough, $\xi \in [\varepsilon - c_1, 1-\varepsilon + c_1]$. Eventually, for $n$ large enough, $\omega' \in \{ \xi \in [\varepsilon - c_1, 1-\varepsilon + c_1]\}$. Let $\mathcal{Z}_1 = \{ \xi \notin [\varepsilon - c_1, 1-\varepsilon + c_1]\}$. We thus have for $n$ large enough, $\mathbb{P}^*(\mathcal{Z}_1) \leq \mathbb{P}^*(\sup_{\alpha \in [0,1]} \vert u_n^*(\alpha) \vert \geq \sqrt{\log n}) \leq 2/n^2$.

Now, from Lemma \ref{lem:unif_strong_rate_quantile} and Lemma \ref{lem:unif_cv_ratio} there exist positive constants $c_2$ and $C_3$ such that for $n$ high enough,
\begin{equation*}
\begin{aligned}
\sup_{\alpha \in [\varepsilon - c_1, 1-\varepsilon + c_1]}& \alpha(1-\alpha)\frac{\vert \widetilde{f_s}'(\widetilde F_s^{\dag}(\alpha))}{\widetilde f_s^2(\widetilde F_s^{\dag}(\alpha))} \\
&\leq \sup_{t \in [F_s^{\dag}(\varepsilon - c_1) - c_2, F_s^{\dag}(1-\varepsilon + c_1) + c_2]} F_s(t)(1-F_s(t))\frac{\vert f_s'(t)\vert}{f_s^2(t)} + C_3\\
&\leq c + C_3
\end{aligned}
\end{equation*}
where $c_2$ is such that
\begin{equation*}
[\widetilde F_s^{\dag}(\varepsilon - c_1), \widetilde F_s^{\dag}(1-\varepsilon + c_1)] \subset [F_s^{\dag}(\varepsilon - c_1) - c_2, F_s^{\dag}(1-\varepsilon + c_1) + c_2] \subset (0, \Vert s \Vert_{\infty}) \, .
\end{equation*}

Let $c^* = c + C_3$ and let $\alpha_1, \alpha_2 \in [\varepsilon - c_1, 1-\varepsilon + c_1]$. Following the proof of Lemma 1.4.1 in \citep{Csorgo1983}, we have for $n$ high enough,
\begin{equation*}
\frac{\widetilde f_s(\widetilde F_s^{\dag}(\alpha_1))}{\widetilde f_s(\widetilde F_s^{\dag}(\alpha_2))} \leq \left(\frac{\max(\alpha_1, \alpha_2)}{\min(\alpha_1, \alpha_2)}\frac{1-\min(\alpha_1, \alpha_2)}{1-\max(\alpha_1, \alpha_2)}\right)^{c^*} \, .
\end{equation*}

Let $z_4 \geq 1$ et let $\mathcal{Z}_2$ be the following event
\begin{equation*}
\mathcal{Z}_2 = \left\{ \sup_{\alpha \in [\varepsilon,1-\varepsilon]}\left(\frac{\max(\widetilde U^{\dag}(\alpha), \alpha)}{\min(\widetilde U^{\dag}(\alpha), \alpha)}\frac{1-\min(\widetilde U^{\dag}(\alpha), \alpha)}{1-\max(\widetilde U^{\dag}(\alpha), \alpha)}\right)^{c^*} \geq z_4\right\} \, .
\end{equation*}
Let $\omega' \in \overline{\mathcal{Z}_1} \cap \overline{\mathcal{Z}_2}$. We therefore have
\begin{equation*}
\sup_{\alpha \in [\varepsilon, 1-\varepsilon]} \frac{\widetilde f_s(\widetilde F_s^{\dag}(\alpha))}{\widetilde f_s(\widetilde F_s^{\dag}(\xi))} \leq \sup_{\alpha \in [\varepsilon,1-\varepsilon]}\left(\frac{\max(\widetilde U^{\dag}(\alpha), \alpha)}{\min(\widetilde U^{\dag}(\alpha), \alpha)}\frac{1-\min(\widetilde U^{\dag}(\alpha), \alpha)}{1-\max(\widetilde U^{\dag}(\alpha), \alpha)}\right)^{c^*} \leq z_4 \, .
\end{equation*}
Therefore $\omega' \in \overline{\mathcal{Z}_3}$ defined as the complement of
\begin{equation*}
\mathcal{Z}_3 = \left\{ \sup_{\alpha \in [\varepsilon, 1-\varepsilon]} \frac{\widetilde f_s(\widetilde F_s^{\dag}(\alpha))}{\widetilde f_s(\widetilde F_s^{\dag}(\xi))} \geq z_4 \right\} \, .
\end{equation*}
Thus $\mathbb{P}^*(\overline{\mathcal{Z}_1} \cap \overline{\mathcal{Z}_2}) \leq \mathbb{P}^*(\overline{\mathcal{Z}_3})$ which implies that for $n$ large enough,
\begin{equation*}
\mathbb{P}^*\left(\sup_{\alpha \in [\varepsilon, 1-\varepsilon]} \frac{\widetilde f_s(\widetilde F_s^{\dag}(\alpha))}{\widetilde f_s(\widetilde F_s^{\dag}(\xi))} \geq z_4\right) = \mathbb{P}^*(\mathcal{Z}_3) \leq \mathbb{P}^*(\mathcal{Z}_1) + \mathbb{P}^*(\mathcal{Z}_2) \leq \frac{2}{n^2} + \mathbb{P}^*(\mathcal{Z}_2) \, .
\end{equation*}
Now from the proof of Theorem 1.4.3 in \citep{Csorgo1983} we can bound $\mathbb{P}^*(\mathcal{S}_2)$ and obtain the result of the lemma.
\end{proof}

%% file: mv_curves.bbl
\begin{thebibliography}{43}
\providecommand{\natexlab}[1]{#1}
\providecommand{\url}[1]{\texttt{#1}}
\expandafter\ifx\csname urlstyle\endcsname\relax
  \providecommand{\doi}[1]{doi: #1}\else
  \providecommand{\doi}{doi: \begingroup \urlstyle{rm}\Url}\fi

\bibitem[Cadre(2006)]{Cadre2006}
B.~Cadre.
\newblock Kernel estimation of density level sets.
\newblock \emph{Journal of Multivariate Analysis}, 97\penalty0 (4):\penalty0
  999 -- 1023, 2006.

\bibitem[Cadre et~al.(2013)Cadre, Pelletier, and Pudlo]{Cadre2013}
B.~Cadre, B.~Pelletier, and P.~Pudlo.
\newblock Estimation of density level sets with a given probability content.
\newblock \emph{Journal of Nonparametric Statistics}, 25\penalty0 (1):\penalty0
  261--272, 2013.

\bibitem[Cavalier(1997)]{Cavalier1997}
L.~Cavalier.
\newblock Nonparametric estimation of regression level sets.
\newblock \emph{Statistics}, 29\penalty0 (2):\penalty0 131--160, 1997.

\bibitem[Cl\'emen\c{c}on and Jakubowicz(2013)]{AISTATS13}
S.~Cl\'emen\c{c}on and J.~Jakubowicz.
\newblock Scoring anomalies: a {M}-estimation formulation.
\newblock In \emph{Proceedings of the 16-th International Conference on
  Artificial Intelligence and Statistics, Scottsdale, USA}, 2013.

\bibitem[Cl\'emen\c{c}on and Robbiano(2014)]{ClemRob14}
S.~Cl\'emen\c{c}on and S.~Robbiano.
\newblock Anomaly ranking as supervised bipartite ranking.
\newblock In \emph{Proceedings of the 31th International Conference on Machine
  Learning, ICML 2014, Beijing, China}, pages 343--351, 2014.

\bibitem[Cl\'emen\c{c}on and Vayatis(2009)]{ClemVay09}
S.~Cl\'emen\c{c}on and N.~Vayatis.
\newblock Adaptive estimation of the optimal {ROC} curve and a bipartite
  ranking algorithm.
\newblock In \emph{Algorithmic Learning Theory}, volume 5809 of \emph{Lecture
  Notes in Computer Science}, pages 216--231. Springer Berlin Heidelberg, 2009.

\bibitem[Cs\"{o}rg\H{o}(1983)]{Csorgo1983}
M.~Cs\"{o}rg\H{o}.
\newblock \emph{Quantile Processes with Statistical Applications}.
\newblock Society for Industrial and Applied Mathematics, 1983.

\bibitem[Cs\"{o}rg\H{o} and R\'ev\'esz(1978)]{CR78}
M.~Cs\"{o}rg\H{o} and P.~R\'ev\'esz.
\newblock Strong approximations of the quantile process.
\newblock \emph{The Annals of Statistics}, 6\penalty0 (4):\penalty0 882--894,
  1978.

\bibitem[Cs\"{o}rg\H{o} and R\'ev\'esz(1981)]{CR81}
M.~Cs\"{o}rg\H{o} and P.~R\'ev\'esz.
\newblock \emph{Strong Approximations in Probability and Statistics}.
\newblock Academic Press, 1981.

\bibitem[DeVore(1987)]{DeVore87}
R.~DeVore.
\newblock A note on adaptive approximation.
\newblock \emph{{A}pprox. {T}heory {A}ppl.}, 3:\penalty0 74--78, 1987.

\bibitem[DeVore(1998)]{DeVore98}
R.~A. DeVore.
\newblock Nonlinear approximation.
\newblock \emph{Acta Numerica}, 7:\penalty0 51--150, 1998.

\bibitem[Donoho and Gasko(1992)]{DonohoGasko}
D.~Donoho and M.~Gasko.
\newblock Breakdown properties of location estimates based on half space depth
  and projected outlyingness.
\newblock \emph{The Annals of Statistics}, 20:\penalty0 1803--1827, 1992.

\bibitem[Efron(1979)]{Efron}
B.~Efron.
\newblock Bootstrap methods: another look at the jacknife.
\newblock \emph{Annals of Statistics}, 7:\penalty0 1--26, 1979.

\bibitem[Egan(1975)]{Ega75}
J.P. Egan.
\newblock \emph{Signal Detection Theory and ROC Analysis}.
\newblock Academic Press, 1975.

\bibitem[Einmahl and Mason(1992)]{EinmahlMason92}
J.H.J. Einmahl and D.M. Mason.
\newblock Generalized quantile processes.
\newblock \emph{The Annals of Statistics}, 20:\penalty0 1062--1078, 1992.

\bibitem[Embrechts and Hofert(2013)]{Embrechts2013}
P.~Embrechts and M.~Hofert.
\newblock A note on generalized inverses.
\newblock \emph{Mathematical Methods of Operations Research}, 77\penalty0
  (3):\penalty0 423--432, 2013.

\bibitem[Falk and Reiss(1989)]{Reiss89}
M.~Falk and R.~Reiss.
\newblock Weak convergence of smoothed and nonsmoothed bootstrap quantile
  estimates.
\newblock \emph{Annals of Probability}, 17:\penalty0 362--371, 1989.

\bibitem[Gin\'e and Guillou(2002)]{Gine}
E.~Gin\'e and A.~Guillou.
\newblock Rates of strong uniform consistency for multivariate kernel density
  estimators.
\newblock \emph{Ann. Inst. Poincar\'e (B), Probabilit\'es et Statistiques},
  38:\penalty0 907--921, 2002.

\bibitem[Hall(1986)]{Hall86}
P.~Hall.
\newblock On the number of bootstrap simulations required to construct a
  confidence interval.
\newblock \emph{Annals of Statistics}, 14:\penalty0 1453--1462, 1986.

\bibitem[Koltchinskii(1997)]{Kolt97}
V.~Koltchinskii.
\newblock M-estimation, convexity and quantiles.
\newblock \emph{The Annals of Statistics}, 25\penalty0 (2):\penalty0 435--477,
  1997.

\bibitem[Koltchinskii(2006)]{Kolt06}
V.~Koltchinskii.
\newblock Local {R}ademacher complexities and oracle inequalities in risk
  minimization (with discussion).
\newblock \emph{The Annals of Statistics}, 34:\penalty0 2593--2706, 2006.

\bibitem[Lifshits(1987)]{Lifshits}
M.~A. Lifshits.
\newblock {On the distribution of the maximum of a Gaussian process}.
\newblock \emph{{Theory of Probability and its Applications}}, 31:\penalty0
  125--132, 1987.

\bibitem[Liu et~al.(1999)Liu, Parelius, and Singh]{LPS99}
R.~Y. Liu, J.~M. Parelius, and K.~Singh.
\newblock Multivariate analysis by data depth: descriptive statistics, graphics
  and inference.
\newblock \emph{Ann. Statist.}, 27\penalty0 (3):\penalty0 783--858, 1999.

\bibitem[Lov\'{a}sz and Vempala(2006)]{Lovasz2006}
L.~Lov\'{a}sz and S.~Vempala.
\newblock {Simulated annealing in convex bodies and an $O(n^4)$ volume
  algorithm}.
\newblock \emph{Journal of Computer and System Sciences}, 72\penalty0
  (2):\penalty0 392 -- 417, 2006.

\bibitem[Mallat(1990)]{Mallat90}
S.~Mallat.
\newblock \emph{A Wavelet Tour of Signal Processing}.
\newblock Academic Press, 1990.

\bibitem[Massart(1990)]{Massart1990}
P.~Massart.
\newblock The tight constant in the dvoretzky-kiefer-wolfowitz inequality.
\newblock \emph{Ann. Probab.}, 18\penalty0 (3):\penalty0 1269--1283, 1990.

\bibitem[Muller and Sawitzki(1991)]{MS91}
D.~W. Muller and G.~Sawitzki.
\newblock Excess mass estimates and tests for multimodality.
\newblock \emph{Journal of the American Statistical Association}, 86\penalty0
  (415):\penalty0 738--746, 1991.

\bibitem[Pitt and Tran(1979)]{Pitt1979}
L.~D. Pitt and L.~T. Tran.
\newblock {Local Sample Path Properties of Gaussian Fields}.
\newblock \emph{Ann. Probab.}, 7\penalty0 (3):\penalty0 477--493, 1979.

\bibitem[Polonik(1995)]{Polonik95}
W.~Polonik.
\newblock Measuring mass concentrations and estimating density contour clusters
  - an excess mass approach.
\newblock \emph{The Annals of Statistics}, 23:\penalty0 855--881, 1995.

\bibitem[Polonik(1997)]{Polonik97}
W.~Polonik.
\newblock Minimum volume sets and generalized quantile processes.
\newblock \emph{Stochastic Processes and their Applications}, 69\penalty0
  (1):\penalty0 1--24, 1997.

\bibitem[Polonik(1999)]{Polonik99}
W.~Polonik.
\newblock Concentration and goodness-of-fit in higher dimensions:
  (asymptotically) distribution-free methods.
\newblock \emph{The Annals of Statistics}, 27\penalty0 (4):\penalty0
  1210--1229, 1999.

\bibitem[Rigollet and Vert(2009)]{Rigollet09}
P.~Rigollet and R.~Vert.
\newblock Fast rates for plug-in estimators of density level sets.
\newblock \emph{Bernoulli}, 14\penalty0 (4):\penalty0 1154--1178, 2009.

\bibitem[Sargan and Mikhail(1971)]{Sargan1971}
J.~D. Sargan and W.~M. Mikhail.
\newblock A general approximation to the distribution of instrumental variables
  estimates.
\newblock \emph{Econometrica}, 39\penalty0 (1):\penalty0 131--169, 1971.

\bibitem[Scott and Nowak(2006)]{ScottNowak06}
C.~Scott and R.~Nowak.
\newblock Learning \uppercase{M}inimum \uppercase{V}olume \uppercase{S}ets.
\newblock \emph{Journal of Machine Learning Research}, 7:\penalty0 665--704,
  2006.

\bibitem[Silverman and Young(1987)]{SY87}
B.~Silverman and G.~Young.
\newblock The bootstrap: to smooth or not to smooth.
\newblock \emph{Biometrika}, 7\penalty0 (4):\penalty0 469--479, 1987.

\bibitem[Steinwart et~al.(2005)Steinwart, Hush, and Scovel]{SHS05}
I.~Steinwart, D.~Hush, and C.~Scovel.
\newblock A classification framework for anomaly detection.
\newblock \emph{J. Machine Learning Research}, 6:\penalty0 211--232, 2005.

\bibitem[Stute(1982)]{Stute1982}
W.~Stute.
\newblock A law of the logarithm for kernel density estimators.
\newblock \emph{The Annals of Probability}, 10\penalty0 (2):\penalty0 414--422,
  05 1982.

\bibitem[{Tsirel'son}(1976)]{Tsirelson1976}
V.~S. {Tsirel'son}.
\newblock {The Density of the Distribution of the Maximum of a Gaussian
  Process}.
\newblock \emph{{Theory of Probability \& Its Applications}}, 20\penalty0
  (4):\penalty0 847--856, 1976.

\bibitem[Tsybakov(1997)]{Tsybakov97}
A.~Tsybakov.
\newblock On nonparametric estimation of density level sets.
\newblock \emph{Annals of Statistics}, 25:\penalty0 948--969, 1997.

\bibitem[Tukey(1975)]{Tukey75}
J.~Tukey.
\newblock Mathematics and picturing data.
\newblock pages 523--531. Canadian Math. Congress, 1975.

\bibitem[Viswanathan et~al.(2012)Viswanathan, Choudur, Talwar, Wang, Macdonald,
  and Satterfield]{VCTWMS}
K.~Viswanathan, L.~Choudur, V.~Talwar, C.~Wang, G.~Macdonald, and
  W.~Satterfield.
\newblock Ranking anomalies in data centers.
\newblock In R.D.James, editor, \emph{Network Operations and System
  Management}, pages 79--87. IEEE, 2012.

\bibitem[Wand and Jones(1994)]{WandJones1994}
M.P. Wand and M.C. Jones.
\newblock \emph{Kernel Smoothing}.
\newblock Chapman \& Hall/CRC Monographs on Statistics \& Applied Probability.
  Taylor \& Francis, 1994.

\bibitem[Zuo and Serfling(2000)]{ZuoSerfling00}
B.Y. Zuo and R.~Serfling.
\newblock General notions of statistical depth function.
\newblock \emph{The Annals of Statistics}, 28\penalty0 (2):\penalty0 461--482,
  2000.

\end{thebibliography}
